\documentclass{article}

\PassOptionsToPackage{numbers, compress}{natbib}

\usepackage[preprint]{neurips_2022}

\usepackage[utf8]{inputenc} %
\usepackage[T1]{fontenc}    %
\usepackage{hyperref}       %
\usepackage{url}            %
\usepackage{booktabs}       %
\usepackage{amsfonts}       %
\usepackage{nicefrac}       %
\usepackage{microtype}      %
\usepackage{xcolor}         %
\usepackage{ulem} 		%
\usepackage{bbm}

\usepackage{array}  

\usepackage{graphicx,epsfig}
\usepackage{epstopdf}
\usepackage{caption}
\usepackage{subcaption}

\usepackage{amsmath, amssymb, amsthm}

\usepackage{multibib}
\newcites{supp}{Supplementary References}

\usepackage[para]{footmisc} %
\usepackage{tablefootnote}

\usepackage{xspace}
\newcommand{\ours}{GAR\xspace}
\newcommand{\ourss}{CIGAR\xspace}

\usepackage{color}

\usepackage{amsmath,amsfonts,bm}

\def\Figref#1{Fig.~\ref{#1}}

\def\Secref#1{Section~\ref{#1}}

\def\Eqref#1{Eq.~\eqref{#1}}

\def\1{\bm{1}}

\DeclareMathAlphabet{\mathsfit}{\encodingdefault}{\sfdefault}{m}{sl}
\SetMathAlphabet{\mathsfit}{bold}{\encodingdefault}{\sfdefault}{bx}{n}
\newcommand{\tens}[1]{\bm{\mathsfit{#1}}}
\def\tA{{\tens{A}}}

\def\tF{{\tens{F}}}

\def\tM{{\tens{M}}}

\def\tT{{\tens{T}}}

\def\tW{{\tens{W}}}

\def\tY{{\tens{Y}}}
\def\tZ{{\tens{Z}}}
\def\tzero{{\tens{0}}}

\newcommand{\etens}[1]{\mathsfit{#1}}

\def\etZ{{\etens{Z}}}

\newcommand{\laplace}{\mathrm{Laplace}} %

\newcommand{\E}{{\bf E}}

\newtheorem{theorem}{Theorem}[section]
\newtheorem{corollary}{Corollary}[theorem]

\renewcommand{\b}{{\bf b}}
\renewcommand{\c}{{\bf c}}

\newcommand{\f}{{\bf f}}

\renewcommand{\k}{{\bf k}}

\renewcommand{\r}{{\bf r}}

\renewcommand{\u}{{\bf u}}
\renewcommand{\v}{{\bf v}}

\newcommand{\x}{{\bf x}}
\newcommand{\y}{{\bf y}}

\newcommand{\A}{{\bf A}}
\newcommand{\B}{{\bf B}}
\newcommand{\C}{{\bf C}}

\newcommand{\I}{{\bf I}}

\newcommand{\K}{{\bf K}}
\renewcommand{\L}{{\bf L}}
\newcommand{\M}{{\bf M}}
\newcommand{\N}{\mathcal{N}}  %

\newcommand{\Ocal}{\mathcal{O}}

\newcommand{\Lcal}{\mathcal{L}}

\renewcommand{\P}{{\bf P}}
\newcommand{\Q}{{\bf Q}}
\renewcommand{\S}{{\bf S}}
\newcommand{\T}{{\bf T}}
\newcommand{\U}{{\bf U}}
\newcommand{\V}{{\bf V}}
\newcommand{\W}{{\bf W}}
\newcommand{\X}{{\bf X}}
\newcommand{\Y}{{\bf Y}}

\newcommand{\Ncal}{{\mathcal{N}}}

\newcommand{\bphi}{\boldsymbol{\phi}}

\newcommand{\btheta}{\boldsymbol{\theta}}

\newcommand{\bSigma}{\boldsymbol{\Sigma}}

\newcommand{\bGamma}{\mathbf{\Gamma}}

\newcommand{\ben}{\begin{enumerate}}
\newcommand{\een}{\end{enumerate}}

\newcommand{\ie}{{i.e.,}\xspace}
\newcommand{\eg}{{e.g.,}\xspace}

\newcommand{\cmt}[1]{}

\newcommand{\bk}{{\bf k}}

\newcommand{\bPsi}{\boldsymbol{\Psi}}

\newtheorem{lemma}{Lemma}

\newtheorem{Property}{Property}

\def\vecrm {\mathrm{vec}}
\def\vec#1{\mathrm{vec}\left({#1}\right)}
\def\cov {\mathrm{cov}}
\newcommand{\kron}{\otimes}

\newcommand{\0}{{\bf 0}}

\newcommand{\RR}{\mathbb{R}}

\newcommand{\TGP}{\mathcal{TGP}}

\title{GAR: A Revisit to the Classic Linear Autoregressive Model And Its Application To Multi-Fidelity Problems}
\title{GAR: Generalized Linear Autoregressive Model}
\title{GAR: Generalized Autoregression \\ for Multi-Fidelity Fusion}
\title{\LARGE{GAR: Generalized Autoregression for Multi-Fidelity Fusion}}
\title{\Large{GAR: Generalized Autoregression for Multi-Fidelity Fusion}}
\title{{GAR: Generalized Autoregression \\for Multi-Fidelity Fusion}}

\author{%
  Yuxin Wang \thanks{The authors contribute equally to this paper.}
    \\
  School of Mathematical Science\\
   Beihang University \\
  Beijing, China, 100191.  \\
  \texttt{WYXtt\_2011@163.com} \\
  \And
  Zheng Xing \footnote[1] \\
  \\
  Graphics\&Computing Department\\
  Rockchip Electronics Co., Ltd\\
  Fuzhou, China, 350003\\
  \texttt{zheng.xing@rock-chips.com}
  \AND
  Wei W. Xing \thanks{Corresponding author.} \\
  School of Integrated Circuit Science and Engineering, Beihang University, Beijing, China, 100191. \\
  \texttt{wayne.xingle@gmail.com} \\
}

\begin{document}
\maketitle

\begin{abstract}
In many scientific research and engineering applications where repeated simulations of complex systems are conducted, a surrogate is commonly adopted to quickly estimate the whole system.
To reduce the expensive cost of generating training examples, it has become a promising approach to combine the results of low-fidelity (fast but inaccurate) and high-fidelity (slow but accurate) simulations.
Despite the fast developments of multi-fidelity fusion techniques, most existing methods require particular data structures and do not scale well to high-dimensional output.
To resolve these issues, we generalize the classic autoregression (AR), which is wildly used due to its simplicity, robustness, accuracy, and tractability, and propose generalized autoregression (GAR) using tensor formulation and latent features.
\ours can deal with arbitrary dimensional outputs and arbitrary multi-fidelity data structure to satisfy the demand of multi-fidelity fusion for complex problems;
 it admits a fully tractable likelihood and posterior requiring no approximate inference and scales well to high-dimensional problems.
Furthermore, we prove the autokrigeability theorem based on \ours in the multi-fidelity case and develop CIGAR, a simplified GAR with the exact predictive mean accuracy with computation reduction by a factor of $d^3$, where $d$ is the dimensionality of the output.
The empirical assessment includes many canonical PDEs and real scientific examples and demonstrates that
the proposed method consistently outperforms the SOTA methods with a large margin (up to 6x improvement in RMSE) with only a couple high-fidelity training samples.

\end{abstract}

\section{Introduction}
\vspace*{-1em}
The design, optimization, and control of many systems in science and engineering can rely heavily on the numerical analysis of differential equations, which is generally computationally intense.
In this case, a data-driven surrogate model is used to approximate the system based on the input-output data of the numerical solver and to help improve convergence efficiency where repeated simulations are required, \eg in Bayesian optimization (BO) \citep{shahriari2016taking} and uncertainty analysis (UQ) \citep{psaros2022uncertainty}.

With the surrogate model in place, the remaining challenge is that executing high-fidelity numerical simulations to generate training data can still be very expensive.
To further reduce the computational burden, it is possible to combine low-fidelity results to make high-fidelity predictions \cite{kennedy2000predicting}.
More specifically, low-fidelity solvers are normally based on simplified PDEs (\eg reducing the levels of physical detail) or simple solver setup (\eg using a coarse mesh, a large time step, a lower order of an approximating basis, and a higher error tolerance). They provide fast but inaccurate solutions to the original problems whereas the high-fidelity solvers are accurate yet expensive.
The multi-fidelity fusion technique works similar to transfer learning to utilize many low-fidelity observations to improve the accuracy when using only a few high-fidelity samples.
In general, it involves constructions of surrogates for different fidelities and a cross-fidelity transition process.
Due to its efficiency, multi-fidelity method has attracted increasing attention in BO \citep{poloczek2017multi,song2019general}, UQ \citep{parussini2017multi,xing2020greedy}, and surrogate modeling \citep{wang2021multi,xing2021residual}. We refer to \citep{fernandezgodino2017review,peherstorfer2018survey} for a great review.

Despite the success of many state-of-the-art (SOTA) approaches, they normally assume that 
(1) the output dimension is the same and aligned across all fidelities, which generally does not hold for multi-fidelity simulation where the output are quantities at nodes that are naturally not aligned;
(2) the high-fidelity samples' corresponding inputs form a subset of the low-fidelity inputs;
and (3) the output dimension is small, which is not practical for scientific computing where dimension can be 1-million (for a $100\times100\times100$ spatial-temporal field).
These assumptions seriously hinder their applications for practical problems, \eg MRI imaging and solving PDEs in scientific computation.

To resolve these challenges, previous work either uses interpolation to align the dimension \citep{xing2021deep,xing2021residual} or relies on approximate inference with brutal simplification \citep{wang2021multi,li2020deep}, leading to inferior performance. 
We notice that the classic autoregression (AR), which is widely used due to its simplicity, robustness, accuracy, and tractability, consistently shows robust and top-tier performance for different datasets in the literature, despite its incapability for high-dimensional problems.
Thus, instead of proposing another ad-hoc model with pre-processing and simplification (leading to models that are difficult to tune and generalize poorly), 
we generalize AR with tensor algebra and latent features and propose generalized autoregression (GAR), which can deal with arbitrary high-dimensional problems without the subset multi-fidelity data structure.
\ours is a fully tractable Bayesian model with scalability to extremely high-dimensional outputs, without requiring any approximate inference.
The novelty of this work is as follows,
\begin{enumerate}
  \item Generalization of AR for arbitrary non-structured and high-dimensional outputs. With tensor algebra and latent features, \ours allows effective knowledge transfer in closed-form and is scalable to extreme high-dimensional problems.
  \item Generalization to non-subset multi-fidelity data for AR. To the best of our knowledge, we are the first to generalize the closed form solution of subset data to non-subset cases for AR and the proposed \ours.
  \item For the first time, we reveal the autokrigeability \citep{alvarez2011kernels} for the multi-fidelity fusion within an AR structure, based on which we derive
  conditional independent GAR (CIGAR), an efficient implementation of \ours with the exact accuracy in posterior mean predictions.
\end{enumerate}

\section{Backgronud}
\vspace*{-1em}
\subsection{Statement of the problem}\label{sec:Statement}

Given multi-fidelity data $D^i=\{\x^i_n,\y^i_n\}_{n=1}^{N^i}$, for $i=1,\dots,\tau$, where $\x \in \mathbb{R}^l$ denotes the system inputs (a vector of parameters that appear in the system of equations and/or in the initial-boundary conditions for a simulation); $\y^i \in \mathbb{R}^{d^i}$ denotes the corresponding outputs, where $d^i$ is the dimension for $i$ fidelity; $\tau$ is the total number of fidelities.
Generally speaking, higher-fidelity data are closer to the ground truth and are more expensive to obtain. Thus, we have fewer samples for the high fidelity, \ie $N^1 \geqslant  N^2 \geqslant \dots \geqslant N^\tau$. The dimensionality is not necessary the same or aligned across different fidelities.
In most work \eg \citep{legratiet2013bayesian,perdikaris2017nonlinear,peherstorfer2018survey}, the system inputs of higher-fidelity are chosen to be the subset of the lower-fidelity, \ie $\X^\tau \subset \dots \subset \X^2 \subset \X^1$. We call this the subset structure for a multi-fidelity dataset, as opposed to arbitrary data structures, which we will resolve in \Secref{sec gar non subset} with a closed-form solution and extend it to the classic AR.
Our goal is to estimate the function $\f^\tau:\RR^l \rightarrow \RR^{d^\tau}$ given the observation data across different fidelities $\{D^i\}_{i=1}^\tau$.

\subsection{Autoregression}
For the sake of clarity, we consider a two-fidelity case with superscript $h$ indicating high-fidelity and $l$ low-fidelity. 
Nevertheless, the formulation can be generalized to problems with multiple fidelities recursively.
Considering a simple scalar output for all fidelities,
AR \citep{kennedy2000predicting} assumes
\begin{equation}
  \label{eq:ar}
  f^{h}(\x) = \rho f^l(\x) + f^r(\x),
\end{equation} 
where $\rho$ is a factor transferring knowledge from the low fidelity in a linear fashion, whereas $f^r(\x)$ tries to capture the residual information.
If we assume a zero mean Gaussian process (GP) prior \citep{rasmussen2006gaussian} (see Appendix \ref{appe gp} for a brief description) for $f^l(\x)$ and $f^r(\x)$, \ie $f^l(\x)\sim\Ncal(0,k^l(\x,\x'))$ and $f^r(\x)\sim\Ncal(0,k^r(\x,\x'))$, the high-fidelity function also follows a GP. 
This gives an elegant joint GP for the joint observations $\y=[{\y^l};{\y^h}]^T$, 
\begin{equation}
  \label{eq:lar joint prob}
\left(
\begin{array}{c}
\y^l\\
\y^h\\
\end{array}
\right) \sim
\Ncal \left(
  \0,
\begin{array}{cc}
  {\K^l(\X^l,\X^l)} & {\rho\K^l(\X^l,\X^h)} \\
  \rho\K^l(\X^h,\X^l) & \rho^2\K^l(\X^h,\X^h)+\K^r(\X^h,\X^h)
\end{array}
\right)
\end{equation}

where $\y^{l} \in \mathbb{R}^{N_l \times 1} $ is the low-fidelity observations corresponding to input $\X^{l} \in \mathbb{R}^{N_l \times L} $ and $\y^h \in \mathbb{R}^{N_h \times 1} $ is the high-fidelity observations; 
$[\K^l(\X^l,\X^l)]_{ij}=k^l(\x_i,\x_j)$ is the covariance matrix of the low-fidelity inputs $ \x_i,\x_j \in \X^l$; 
$[\K^r(\X^l,\X^l)]_{ij}=k^r(\x_i,\x_j)$ is for the high-fidelity inputs $ \x_i,\x_j \in \X^h$;
$[\K^l(\X^l,\X^h)]_{ij}=k^l(\x_i,\x_j)$ is the cross-fidelity covariance matrix of the low-fidelity inputs $ \x_i \in \X^l$ and high-fidelity inputs $\x_j \in \X^h$, and $\K^l(\X^h,\X^l)=(\K^l(\X^l,\X^h))^T$.
One immediate advantage of AR is that the joint Gaussian form allows not only joint training for all low- and high-fidelity data but also predictions for any given new inputs by a conditional Gaussian (as the posterior is derived in a standard GP \citep{rasmussen2006gaussian}). Furthermore, \citet{legratiet2013bayesian} derive Lemma \ref{lemma1} to reduce the complexity from $O((N^l+N^h)^3)$ to $O((N^l)^3+(N^h)^3)$ with a subset data structure.
\begin{lemma}\citep{legratiet2013bayesian}
\label{lemma1}
If $\X^h \subset \X^l$, the joint likelihood and predictive posterior of AR can be decomposed into two independent parts corresponding to the low- and high-fidelity data.
\end{lemma}

\section{Generalized Autoregression}
\vspace*{-1em}
Let's now consider the more general high-dimensional case.
A naive approach is to simply convert the multi-dimensional output into a scalar value by attaching a dimension index to the input. However, AR will end up with a joint GP with a covariance matrix of the size of $(N^l d^l+N^h d^h)^2$, making it infeasible for modestly high-dimensional problems.

\subsection{Tensor Factorized Generalization with Latent Features}
To resolve the scalability issue, we rearrange all the output into a multidimensional space (i.e., a tensor space) and introduce latent coordinate features to index the outputs to capture their correlations as in HOGP~\citep{zhe2019scalable}.
More specifically, 
we organize the low-fidelity output as a $M$-mode tensor, $\tZ^l \in \RR^{d^l_1 \times \dots \times d^l_{M}}$, where the output dimension $d^l=\prod_{m=1}^M d^l_m$.
The element $\etZ^l$ is indexed based on its coordinates $\c=(c_1,\dots,c_M)$($1 \leqslant c_k \leqslant d_k$ for $k=1,\dots,M$).
If the original data indeed admits a multi-array structure, we can use their original index with actual meaning to index the coordinates.
For instance, a 2D spatial dataset can use its original spatial coordinate to index a single location (pixel).
To improve our model flexibility, we do not have to limit ourselves from using the original index, particularly for the cases where the original data does not admit a multi-array structure or the multi-array structure is of too large size.
In such case, we can use arbitrary tensorization and a latent feature vector
$\v^l_{c_m}$ (whose values are inferred in model training) for each coordinate $c_m$ in mode $m$. This way, the element $\etZ^l$ is indexed by the vector $(\v^l_{c_1},\dots,\v^l_{c_M})$.
Following the linear transformation of \Eqref{eq:ar}, we first introduce the tensor-matrix product \citep*{kolda2006multilinear},
\begin{equation}
  \label{eq:tensor LAR}
  \tF^h(\x) =  \tF^l(\x) \times_1 \W_1, \dots, \times_M \W_M  + \tF^r(\x),
\end{equation}
where $\tF^h(\x)$ denotes target function $\f^h(\x)$ with its output organized into a multi-array $\tZ^h$, and the same concept applies to $\tF^l(\x)$ and $\tF^r(\x)$;
$\times_m$ denotes the tensor-matrix product at mode $m$.
To give a concrete example, considering an arbitrary tensor $\tZ^l \in \RR^{d^l_1\times \dots \times d_M^l}$ and a matrix $\W_m \in \RR^{s \times d_m}$,
the $\times_m$ product is calculated as 
$[\tZ^l \times_m \W_m]_{i_1,\dots,i_{m-1},j,i_{m+1}\dots,i_M}= \sum_{k=1}^{c_m} w_{jk} \etZ_{i_1,\dots,i_{m-1},k,i_{m+1}\dots,i_M}$, which becomes an $d_1^l \times \dots \times d_{m-1}^l \times s \times d_{m+1}^l \times \dots \times d^l_M$ tensor. 
We can further denote the group of $M$ linear transformation matrixes as a Tucker tensor $\tW = [\W_1,\dots,\W_M]$ and represent \Eqref{eq:tensor LAR} compactly using a Tucker operator~\citep{kolda2006multilinear}, $\tF^l(\x) \times \tW$,
which has an important property:
\begin{equation}
  \vec{\tF^h(\x) - \tF^r(\x)} = \left(\W_1 \kron \dots \kron \W_M \right) \vec{\tF^l(\x)}.
\end{equation}

Inspired by AR of \Eqref{eq:ar}, we place a tensor-variate GP (TGP) prior \citep{xu2011infinite} for the low-fidelity tensor function $\tF^l(\x)$ and the residual tensor function $\tF^r(\x)$:
\begin{equation}
  \label{eq: tgp l}
  \tZ^l(\x,\x') \sim \TGP \left(\tzero, k^l(\x,\x'),\S_1^l,\dots,\S_M^l  \right), \tZ^r(\x,\x') \sim \TGP \left(\tzero, k^r(\x,\x'), \S_1^r,\dots,\S_M^r \right),
\end{equation}
where $\S_m^i \in \RR^{d_m \times d_m}$ are the output correlation matrix with $[\S_m^i]_{jk}=\tilde{k}^i_m(\v^i_{c_i},\v^i_{c_k})$ and $\tilde{k}^i_m(\cdot,\cdot)$ being the kernel function (with unknown hyperparameters).
A TGP is a generalization of a multivariate GP that essentially represents a joint GP prior $\vecrm({\tY^l}) \sim \Ncal \left(0, \K^l(\X^l,\X^l) \bigotimes_{m=1}^M \S_m  \right)$.
Similar to the joint probability of \eqref{eq:lar joint prob}, we can derive the joint probability for $\y=[\vecrm({\tY^l})^T,\vecrm(\tY^h)^T]^T$ based on Tucker transformation of \eqref{eq:tensor LAR}; we preserve the proof in the Appendix for clarity. 

\begin{lemma}
  \label{lamma gar joint}
  Given the tensor GP priors for $\tY^l(\x,\x')$ and $\tY^r(\x,\x')$ and the Tucker transformation of \eqref{eq:tensor LAR}, the joint probability for $\y=[\vecrm({\tY^l})^T,\vecrm(\tY^h)^T]^T$ is $\y \sim \Ncal(\0, \bSigma)$, where $\bSigma = $
  \[
    {\small
\label{eq:joint likelihood GAR} 
\left(\
\begin{array}{cc}
{\K^l}(\X^l, \X^l)\kron \left(\bigotimes_{m=1}^M\S_m\right) &  {\K^l(\X^l,\X^h)}\kron \left(\bigotimes_{m=1}^M\S_m\W_m^T\right) \\
\K^l(\X^h,\X^l)\kron \left(\bigotimes_{m=1}^M\W_m\S_m\right) & \K^l(\X^h,\X^h)\kron\left(\bigotimes_{m=1}^M \W_m \S_m \W_m^T\right) +  \K^r(\X^h,\X^h )\kron\left(\bigotimes_{m=1}^M \S^r_m\right)
\end{array}\right).
    }
\]

\end{lemma}

Lemma \ref{lamma gar joint} admits any arbitrary outputs (living in different spaces, having different dimension and/or mode, and being unaligned) at different fidelity. Also, it does not require a subset dataset to hold.

\begin{corollary}
  Lemma \ref{lamma gar joint} can be applied to data with a different number of mode at each fidelity, \ie $M^l \neq M^h$, if we add a redundancy index such that all outputs have the same $M$ number of modes.
\end{corollary}

Lemma \ref*{lamma gar joint} defines our \ours model, a generalized AR with special tensor structures.
The covariance for low-fidelity is $\cov(\etZ^l_\c(\x),\etZ^l_{\c'}(\x'))=k^l(\x,\x')\prod_{m=1}^M \tilde{k}^l_m(\v^m_{c_m},\v^m_{c'_m})$,
cross-fidelity $\cov(\etZ^l_\c(\x),\etZ^h_{\c'}(\x'))=k^l(\x,\x')\prod_{m=1}^M \tilde{k}^l_m(\v^l_{c_m},\v^l_{c'_m}) w^m_{c,c'} $ (where $w^m_{c,c'}$ is the ${c,c'}$-th element of $\W^m$),
and high-fidelity $\cov(\etZ^h_\c(\x),\etZ^h_{\c'}(\x'))=k^l(\x,\x')\prod_{m=1}^M \tilde{k}^l_m(\v^l_{c_m},\v^l_{c'_m}) (w^m_{c,c'})^2 + k^h(\x,\x')\prod_{m=1}^M \tilde{k}^r_m(\u^m_{c_m},\u^m_{c'_m}) (w^m_{c,c'})^2 $.
The complex between-fidelity output correlations are captured using latent features $\{\V^m,\U^m\}_{m=1}^M$ with arbitrary kernel function $\tilde{k}_m^i$, whereas the cross-fidelity output correlations are captured in a composite manner. This combination overcomes the simple linear correlations assumed in previous work that simply decomposes the output as a dimension reduction preprocess \citep{xing2021deep}. 
When the dimensionality aligns for $\tZ^l$ and $\tZ^h$ and thus $d^l_m=d^h_m$, we can share the same latent features across the two fidelities by letting $\v^m_j=\u^m_j$ while keeping the kernel functions different. This way, the latent features are more resistant to overfitting. 
For non-aligned data with explicit indexing, we can use kernel interpolation \citep{wilson2015kernel} for the same purpose.
To further encourage sparsity in the latent feature, we impose a Laplace prior, \ie $\v^m_j \sim \laplace(\lambda) \propto \exp(-\lambda ||\v^m_j||_1)$.

\subsection{Efficient Model Inference for Subset Data Structure}
With the model fully defined, we can now train the model to obtain all unknown model parameters.
For compactness, we use the following compact notation 
$\S^l=\bigotimes_{m=1}^M \S^l_m$,
$\S^h=\bigotimes_{m=1}^M \S^h_m$,
$\W=\bigotimes_{m=1}^M \W_m$,
$\K^l=\K^l(\X^l, \X^l)$,
$\K^{lh}=\K^l(\X^l, \X^h)$,
$\K^{hl}=\K^l(\X^h, \X^l)$,
$\K^{lr}=\K^l(\X^h, \X^h)$,
and,
$\K^r=\K^r(\X^h, \X^h)$ (with a slight abuse of notation).

\begin{lemma}\label{lemma3: GAR}
Tensor generalization of Lemma \ref{lemma1}.
If $\X^h \subset \X^l$, the joint likelihood $\Lcal$ for $\y=[\vecrm({\tY^l})^T,\vecrm(\tY^h)^T]^T$ admits two independent separable likelihoods $\Lcal = \Lcal^l + \Lcal^r$, where
\[
  \Lcal^l = -\frac{1}{2}\vec{\tY^l}^T (\K^l\otimes\S^l )^{-1} \vec{\tY^l} - \frac{1}{2} \log|\K^l\otimes\S^l | - \frac{N^l D^l}{2} \log(2\pi),
  \]
\[ 
  \Lcal^r = -\frac{1}{2} \vec{\tY^h- \tY^l \times \hat{\tW}}^T (\K^r\otimes\S^r)^{-1} \vec{\tY^h- \tY^l \times \hat{\tW}} - \frac{1}{2} \log|\K^r\otimes\S^r| - \frac{N^h D^h}{2}  \log(2\pi),
  \]
  where $\hat{\tW} = [\E,\W_1,\dots,\W_M]$ is a Tucker tensor with selection matrix $\E^T \X^l=\X^h$. 
\end{lemma}

We preserve the proof in Appendix for clarity. Note that $\Lcal^l$ and $\Lcal^r$ are HOGP likelihoods for $\tY^l$ and the residual $\tY^h- \tY^l \times \hat{\tW}$, respectively.
Since the computational of $\Lcal^l$ and $\Lcal^r$ are independent, the model training can be conducted efficiently in parallel.

{\bf{Predictive posterior.}} Similarly, we can derive the concrete predictive posterior for the high-fidelity outputs by integrating out the latent functions after some tedious linear algebra (see Appendix), which is also Gaussian,
$\vecrm(\tZ^h_*) \sim \mathcal{N}(\vecrm(\bar{\tZ}^h_*), {\S^h_*})$, where
\begin{equation}
  \label{eq gar post}
	\begin{aligned}
		\vecrm({\bar{\tZ}^h_*})
		&=\left( \k^l_* \left(\K^l\right)^{-1} \kron \W \right)\vecrm({\tY^l}) + \left( \k^r_* \left(\K^r\right)^{-1} \kron \I \right)\vecrm({\tY^r}),\\
	    \S_*^h 
	    &=  \left(k^l_{**} - (\k^l_*)^T \left(\K^l\right)^{-1} \k^l_* 	\right) \kron \W \S^l \W^T + { \left(k^r_{**} -\left(\k^r_*\right)^T (\K^r)^{-1} \k^r_* \right) \kron \S^r },
	\end{aligned}
\end{equation}
$\k^l_*=\k^l(\x_*,\X^l)$ is the vector of covariance between the give input $\x_*$ and low-fidelity observation inputs $\X^l$; similarly, we have $\k^l_{**}=\k^l(\x_*,\x_*)$, $\k^r_*=\k^r(\x_*,\X^h)$, $\k^r_{**}=\k^r(\x_*,\x_*)$.

\subsection{Generalization for Non-subset Data: Efficient Model Inference and Prediction}
\label{sec gar non subset}
In practice, it is sometimes difficult to ask the multi-fidelity data to preserve a subset structure, particularly in the case of multi-fidelity Bayesian optimization \cite{perrone2018scalable,li2020multi}.
This presents the challenge for most SOTA multi-fidelity models \eg NAR \cite{perdikaris2017nonlinear}, ResGP~\cite{xing2021residual}, stochastic collocation~\cite{narayan2014stochastic}. 
In contrast, the advantage of AR is that even if the multi-fidelity data does not admit a subset data structure, the model can still be trained using all available data based on the joint likelihood of \eqref{eq:joint likelihood GAR}. 
However, this method lacks scalability due to the inversion of the large joint covariance matrix $\bSigma$. The situation gets worse if we are dealing with multi-fidelity data with more than two fidelities.
To resolve this issue, we propose a fast inference method based on 
imaginary subset. More specifically, considering the missing low-fidelity data as latent variables $\hat{\tY}^l$, the joint likelihood function is
\begin{equation}
  \begin{aligned}
    \log  p(\tY^l,\tY^h) 
    &= \log \int p(\tY^l,\tY^h, \hat{\tY}^l) d \hat{\tY}^l
     = \log \int \left( p(\tY^h|\hat{\tY}^l,\tY^l)  p(\hat{\tY}^l|\tY^l) p(\tY^l) \right) d \hat{\tY}^l \\
    & = \log \int p(\tY^h |\hat{\tY}^l,\tY^l) p(\hat{\tY}^l|\tY^l) d \hat{\tY}^l + \log p(\tY^l),
  \end{aligned}
\end{equation}
where $p(\tY^h |\hat{\tY}^l,\tY^l)$ is the likelihood in Lemma \ref*{lemma3: GAR} given the complementary imaginary subset,
and $p(\hat{\tY}^l|\tY^l)\sim \Ncal(\bar{\tY}^l,\hat{\S}^l\kron\S^l)$ is the imaginary posterior with the given low-fidelity observations $\tY^l$.
The integral can be calculated using Gaussian quadrature or other sampling methods as in \citep{wang2021multi,cutajar2019deep}, which is slow and inaccurate.

\begin{lemma}
  \label{lamma gar likelihood}
   The joint likelihood of \ours for non-subset (and also unaligned) data can be decomposed into two independent GPs' likelihood
  {\small
  \begin{equation}
    \label{eq: gar likelihood}
    \begin{aligned}
      &\log p(\tY^l, \tY^h) 
      =\mathcal{L}^l-\frac{N^hd^h}{2}\log(2\pi)-\frac{1}{2}\log
      \left|{\K}^r\kron\S^r+\hat{\E}\hat{\S}^l\hat{\E}^T\kron\W^T\S^l{\W}\right|\\
      &-\frac{1}{2}\left[\left(\begin{array}{c}
        \vecrm(\check{\tY}^h)\\
        \vecrm(\hat{\tY}^h)
      \end{array}\right)^T-\left(\begin{array}{c}\vecrm(\check{\tY}^l) \\ \vecrm(\bar{\tY}^l)
      \end{array}\right)^T
  		\tilde{\W}^T\right]
  	  \left({\K}^r\kron\S^r+\hat{\E}\hat{\S}^l\hat{\E}^T\kron\W^T\S^l{\W}\right)^{-1}
  	  \left[\left(\begin{array}{c}
      \vecrm(\check{\tY}^h)\\
      \vecrm(\hat{\tY}^h)	\end{array}\right)-\tilde{\W}\left(\begin{array}{c}\vecrm(\check{\tY}^l) \\ \vecrm(\bar{\tY}^l)
      \end{array}\right)\right],
    \end{aligned}
  \end{equation}
  }
where $\mathcal{L}^l$ is the likelihood for low-fidelity data $\tY^l$, $ \tilde{\W} = \I_{N^h}\kron\W $
$\hat{\tY}^h$ is the collection of high-fidelity observations corresponding to the imaginary low-fidelity outputs $\hat{\tY}^l$; $\check{\tY}^h$ is the complement (with selection matrix $\check{\X}^h = \E^T \X^l$) corresponding to low-fidelity outputs $\check{\tY}^l$, \ie $\tY^h=[\check{\tY}^h,\hat{\tY}^h]$; and $\hat{\X}^h = \hat{\E}^T \X^h$ are the selection matrix for $\hat{\tY}^l$.

\end{lemma}

We preserve the proof in the appendix.
Notice that 
$ \hat{\E}\hat{\S}^l\hat{\E}^T\kron\W^T\S^l{\W} = \left(\begin{array}{cc}
\0 & \0 \\ \0 & \hat{\S}^l\end{array}\right)\kron\W^T\S^l{\W}$ is the low-right block of the predictive variance for the missing low-fidelity observations $\vecrm(\hat{\tY})$;
We can easily understand the last part of the likelihood as a GP with accumulated uncertainty(variance) added to the corresponding missing points. Lemma \ref{lamma gar likelihood} naturally applies to AR when the outputs is a scaler, where $\W=\rho$, $\S^l=1$, and $\S^r=1$.

{\bf{Predictive posterior.}} Surprisingly, the posterior also turns out to be a Gaussian distribution,
\begin{equation}
  \label{eq: gar posterior}
	\begin{aligned}
		& p\left(\tZ^h_*|\tY^l,\tY^h,\x_*\right)
		=2\pi^{-\frac{d^h}{2}}\times \left|\S_*^h+\bGamma\left(\hat{\S}^l\kron\S^l\right)\bGamma^T\right|^{-\frac{1}{2}}\\
		& \quad \times \exp\left[-\frac{1}{2}\left(\vecrm(\tZ^h_*)-\vecrm(\bar{\tZ})\right)^T
		\left(\S_*^h+\bGamma\left(\hat{\S}^l\kron\S^l\right)\bGamma^T\right)^{-1}\left(\vecrm(\tZ^h_*)-\vecrm(\bar{\tZ})\right)\right],
	\end{aligned}
\end{equation}
where $\bGamma$  and the mean of the predictive posterior $\bar{\tZ}_*$ are given as follows,
\begin{equation}
	\begin{aligned}
		\bGamma &= \left([\k_{*}^r (\K^r)^{-1}\E_n^T-\k_{*}^l (\hat{\K}^l)^{-1}]\kron\W\right)\E_m\kron\I^l,\\
		\vecrm(\bar{\tZ}_*) &= \left(\k^l_{*} (
		\hat{\K}^l)^{-1}\kron\W\right)\left(\begin{array}{c}
			\vecrm(\tY^l) \\ \vecrm(\bar{\tY}^l)
		\end{array}\right)+\left(\k_{*}^r (\K^r)^{-1} \kron\I\right)
		\left(\vecrm(\tY^h)-\hat{\W}
		\left(\begin{array}{c}
			\vecrm(\tY^l) \\ \vecrm(\bar{\tY}^l)
		\end{array}\right)
		\right),
	\end{aligned}
\end{equation}
where $\E_m$ and $ \E_n $ are the selection matrices such that $ \hat{\X}^h = \E_m^T[\X^l, \hat{\X}^h] $, $ \X^h = \E_n^T[\X^l, \hat{\X}^h] $, $ \hat{\W} = \E_n^T\kron\W $, and $ \hat{\K}^l $ is the covariance matrix that $ \hat{\K}^l = \K^l([\X^l, \hat{\X}^h], [\X^l, \hat{\X}^h]) $.

\subsection{Autokrigeability, Complexity, and Further Acceleration}\label{cigar}
For subset structured data, the computational complexity of \ours is decomposed into two independent TGPs for likelihood and predictive posterior.
Thanks to the tensor algebra (mainly $(\K \kron \S)^{-1}=\K^{-1} \kron \S^{-1}$), the complexity of the $i$-fidelity kernel matrix inversion is reduced to $\Ocal(\sum_{m=1}^M (d^i_m)^3+(N^i)^3)$ instead of $\Ocal((N^i d^i)^3)$. For the non-subset case, the computational complexity in \Eqref{eq: gar likelihood} is unfortunately $\Ocal((N_m^i d^i)^3)$ where $N_m$ is the number of the imaginary low-fidelity points. Nevertheless, due to the tensor structure, we can still use conjugate gradient \citep{wilson2013gaussian} to solve the linear system efficiently.

Notice that the mean prediction $\bar{\tZ}^h_*$ in \Eqref{eq: gar posterior} does not depend on any output covariance matrixes $\{\S^h_m, \S^l_m\}_{m=1}^M$,
which reassemble the autokrigeability (no knowledge transfer in noiseless cases for mean predictions)~\citep{alvarez2011kernels,xing2021residual} based on the GAR framework.
For applications where the predictive variation is not of interest, we can 
introduce a conditional independent output-correlation, \ie $\S^h_m=\I, \S^l_m=\I$ and orthogonal weight matrixes, \ie $\W_m^T \W_m = \I$, to reduce the computationally complexity further down to $\Ocal((N^i)^3)$ (see Appendix for detailed proof). We call this CIGAR as an abbreviation for conditional independent GAR. In our empirical assessment, CIGAR is slightly worse than GAR due to the difficulty of ensuring $\W_m^T \W_m = \I$ and numerical noise.

\section{Related Work}
\vspace*{-1em}
GP for high-dimensional outputs is an important model in many applications such as spatial data modeling and uncertainty quantification. For an excellent review, the readers are referred to \cite{alvarez2012kernels}.
Linear model of coregionalization (LMC) \citep{goulard1992linear,goovaerts1997geostatistics} might be the most general framework for high-dimensional GP developed in the geostatistic community. LMC assumes that the full covariance matrix as a sum of constant matrixes timing input-dependent kernels.
To reduce model complexity, semiparametric latent factor models (SLFM) \citep{teh2005semiparametric} simplify LMC by assuming that the matrixes are rank-1 matrixes.
\citet{higdon2008computer} further simplifies SLFM using singular value decomposition (SVD) on the output collection to fix the bases for the rank-1 matrixes. 
To overcome the linear assumptions of LMC, the (implicit) bases can be constructed in a nonlinear fashion using manifold learning, \eg KPCA~\citep{xing2016manifold} and IsoMap~\citep{xing2015reduced} or process convolution \citep{alvarez2019non,boyle2004dependent,higdon2002space}.
Other approaches include multi-task GP, which considers the outputs as dependent tasks \citep{bonilla2007kernel,rakitsch2013it,li2018hierarchical} in a framework similar to LMC and
GP regression network (GPRN) \citep{wilson2011gaussian,nguyen2013efficient}, which proposes products of GPs to model nonlinear outputs, leading to nontractable models.
Despite their success, the complicity of the above approaches are at best $\Ocal(N^3+d^3)$ whereas some are $\Ocal(N^3d^3)$, which cannot scale well to high-dimensional outputs for scientific data where $d$ can be, says, 1 million. 
This problem can be well resolved by introducing tensor algebra \citep{kolda2009tensor} to form HOGP~\citep{zhe2019scalable} or scalable model inference, \eg in GPRN \citep{li2020scalable}.

Multi-fidelity has become a promising approach to further reduce the data demands in building a surrogate model \citep{li2020deep} and Bayesian optimization.
The seminal autoregressive (AR) model of \citet{kennedy2000predicting} introduces a linear transformation to univariate high-fidelity outputs.
This model was enhanced by \citet{legratiet2013bayesian} by adopting a deterministic parametric form of linear transformation for the efficient training scheme as introduced previously. However, it is unclear how AR can deal with high-dimensional outputs.
To overcome the linearity of AR, \citet{perdikaris2017nonlinear} proposes nonlinear AR (NAR). It ignores the output distributions and directly uses the low-fidelity solution as an input for the high-fidelity GP model, which is essentially a concatenating GP structure known as \textit{deep GP}~\citep{damianou2015deep}.
To propagate uncertainty through the multi-fidelity model, \citet{cutajar2019deep} uses expensive approximation inference, which makes it prone to overfitting and incapable of dealing with very large dimensional problems.
For multi-fidelity Bayesian optimization (MFBO), \citet{poloczek2017multi} and \citet{kandasamy2016gaussian} approximate each fidelity with a GP independently; 
\citet{zhang2017information} use convolution kernel, similar to the process convolution \citep{alvarez2019non,higdon2002space} to learn the fidelity correlations;
\citet{song2019general} combine all fidelity data into one single GP to reduce uncertainty.
However, most MFBO surrogates do not scale to high-dimensional problems because they are designed for one target (or at most a couple).

To deal with large dimensional outputs, \eg spatial-temporal fields, 
\citet{xing2021residual} extend AR by assuming a simple additive structure and replacing the simple GPs with scalable multi-output GPs at the cost of losing the power for capturing the output correlations, leading to inferior performance and inaccurate uncertainty estimates;
\citet{xing2021deep} propose Deep coregionalization to extend NAR by learning the latent process \citep{teh2005semiparametric,goovaerts1997geostatistics} extracted from embedding the high-dimensional outputs onto a residual latent space using a proposed residual PCA;
\citet{wang2021multi} further introduce bases propagation along with latent process propagation in a deep GP to increase model flexibility at the cost of significant growth in the number of model parameters and a few simplifications in the approximated inference.
\citet{parussini2017multi} generalize NAR to high-dimensional problems.
However, these methods lack a systematic way for joint model training, leading to instability and poor fitting for small datasets.
{\citet{wu2022multi} extend GP using neural process to model  high-dimensional and non-subset problem effectively.}
In scientific computing, multi-fidelity fusion has been implemented using stochastic collocation (SC) method~\citep{narayan2014stochastic} for high-dimensional problems, which provides closed-form solutions and efficient design of experiments for the multi-fidelity problem. \citet{xing2020greedy} showed that SC is essentially a special case of AR and proposed active learning to select the best subset for the high-fidelity experiments.

To take the advances of deep learning neural network (NN) and being compatible with arbitrary multi-fidelity data (\ie non-subset structure), \citet{perrone2018scalable} propose an NN-based multi-task method that can naturally extend to MFBO.
\citet{li2020multi} further extend it as a Bayesian neural network (BNN) to MFBO. 
\citet{meng2020composite} add a physics regularization layer, which requires an explicit form of the problem PDEs, to improve prediction accuracy.
To scale for high-dimensional problems with arbitrary dimensions in each fidelity, \citet{li2020deep} propose a Bayesian network approach to multi-fidelity fusion with active learning techniques for efficiency improvement.

Except for multi-fidelity fusion, AR can be used for model multi-variate problem \citep{requeima2019gaussian,xia2020gaussian}, where \ours can also find its applications.
\ours is a general framework for GP-based multi-fidelity fusion of high-dimensional outputs. 
Specifically, AR is a special case when setting $\W = \rho \I$ and using a separable kernel; ResGP is a special case of \ours by setting $\W = \I$ and $\S=\I$; NAR is a special case of integrating out $\W$ with a normal prior and using a separable kernel; DC is a special case of \ours if it only uses one latent process, integrating out $\W$ as in NAR with a separable kernel; MF-BNN is a finite case of \ours if only one hidden layer is used. See Appendix \ref{appe: sota} for the comparison between SOTA methods.

\vspace*{-1em}
\section{Experimental Results}
\label{exp result}
\vspace*{-1em}

To assess \ours and \ourss, we compare with the SOTA multi-fidelity fusion methods for high-dimensional outputs including:
(1) AR \citep{kennedy2000predicting},
(2) NAR \citep{perdikaris2017nonlinear},
(3) ResGP \citep{xing2021residual},
(4) DC\footnote{https://github.com/wayXing/DC} \citep{xing2021deep},
and
(5) MF-BNN\footnote{https://github.com/shib0li/DNN-MFBO} \citep{li2020deep}.
All GPs use an RBF kernel for a fair comparison. Because the ARD kernel is separable, the AR and NAR are accelerated using the Kronecker product structure as in \ours for a feasible computation. The original DC with residual PCA cannot deal with unaligned outputs, but it can do so by using an independent PCA, which we called DC-I. Both DCs preserve 99\% energy for dimension reductions. MF-BNN is conducted using its default setting.
\ours, \ourss, AR, NAR, and ResGP are implemented using Pytorch\footnote{https://pytorch.org/}. All experiments are run on a workstation with an AMD 5950x CPU and 32 GB RAM.

\vspace*{-1em}
\begin{figure}[h]
	\centering
	\begin{subfigure}[b]{0.32\linewidth}
		\includegraphics[width=1\textwidth]{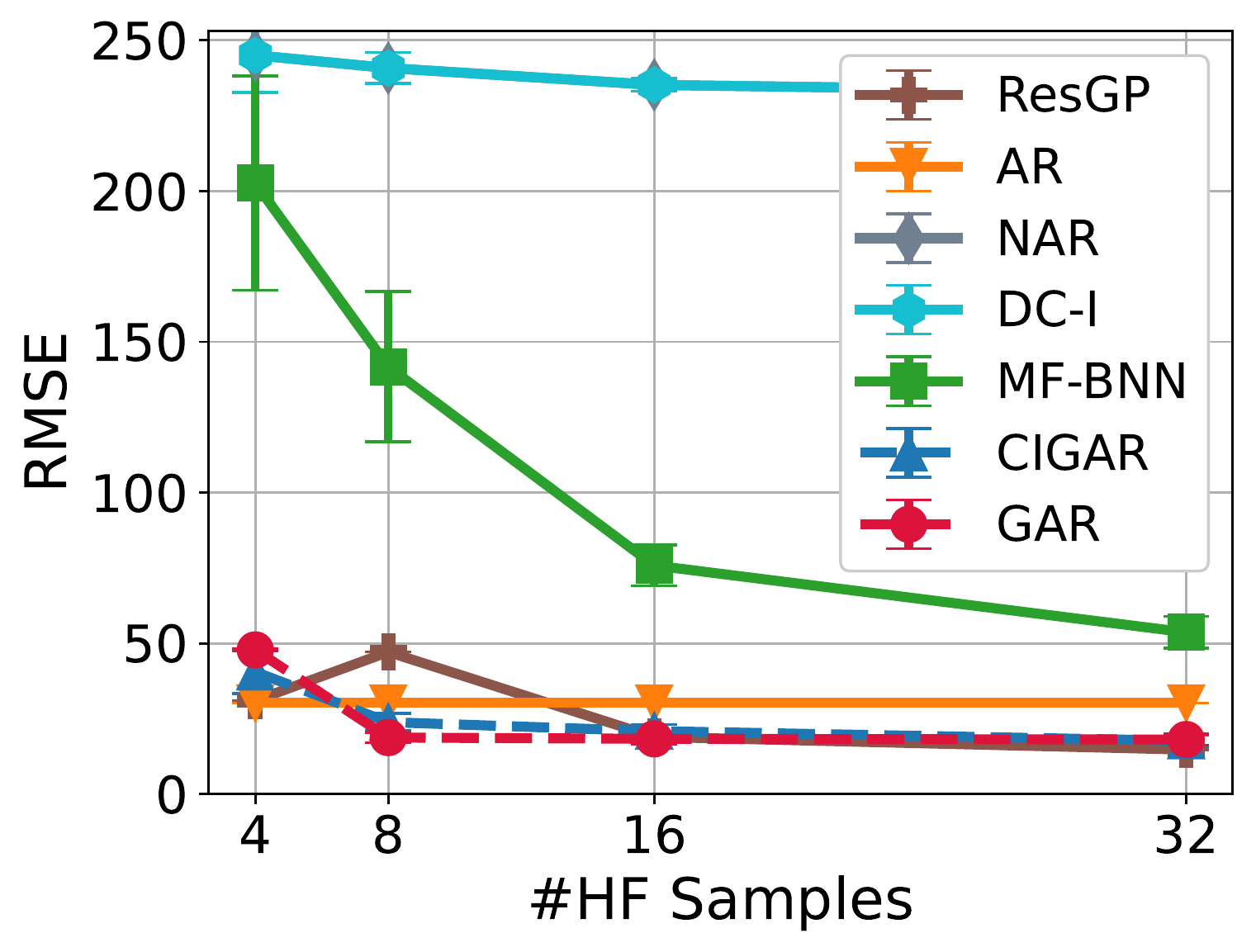}
	\end{subfigure}
	\begin{subfigure}[b]{0.32\linewidth}
		\includegraphics[width=1\textwidth]{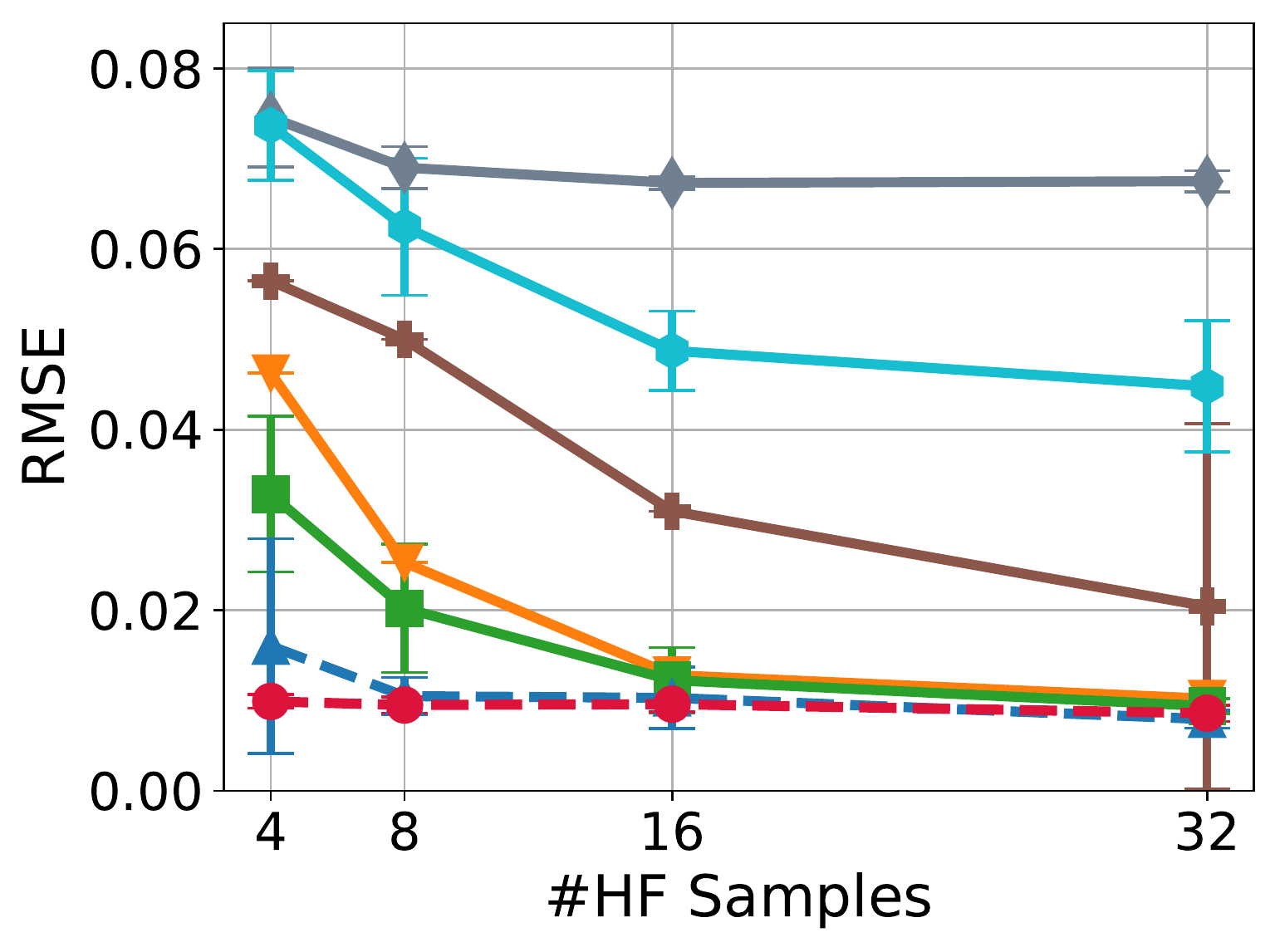}
	\end{subfigure}
    \begin{subfigure}[b]{0.32\linewidth}
		\centering
		\includegraphics[width=1\textwidth]{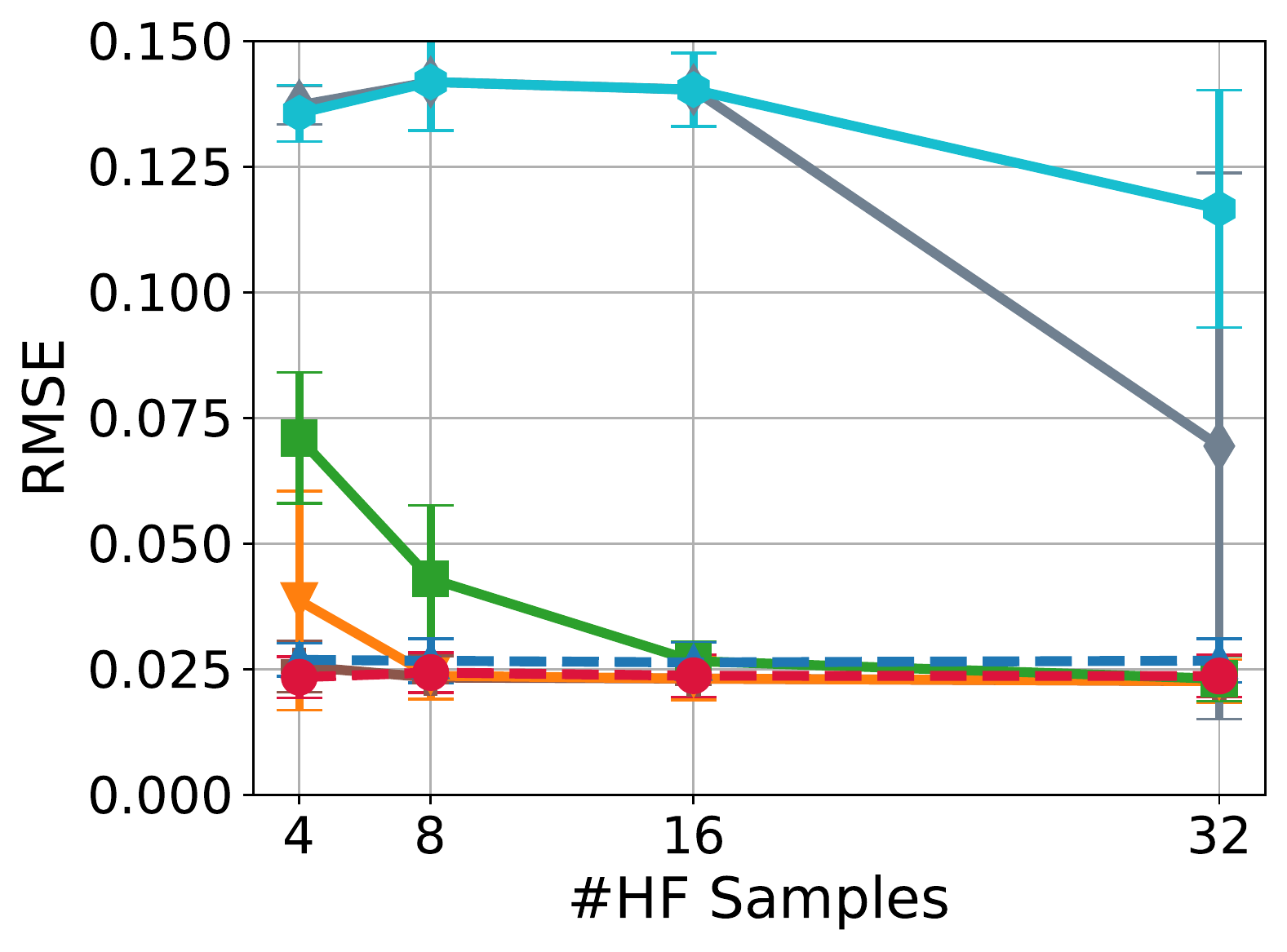}
	\end{subfigure}\\
	\begin{subfigure}[b]{0.32\linewidth}
		\includegraphics[width=1\textwidth]{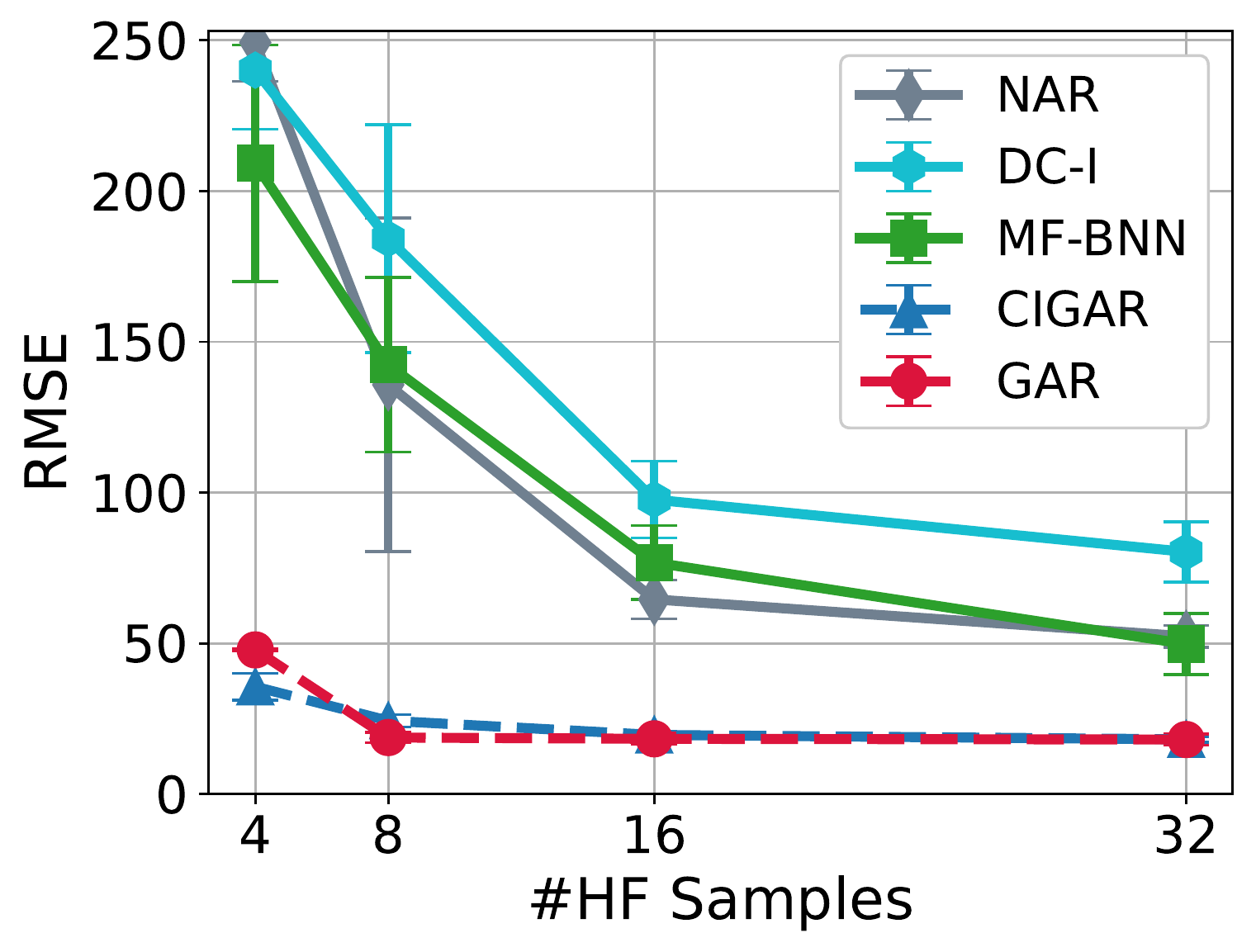}
	\end{subfigure}
	\begin{subfigure}[b]{0.32\linewidth}
		\includegraphics[width=1\textwidth]{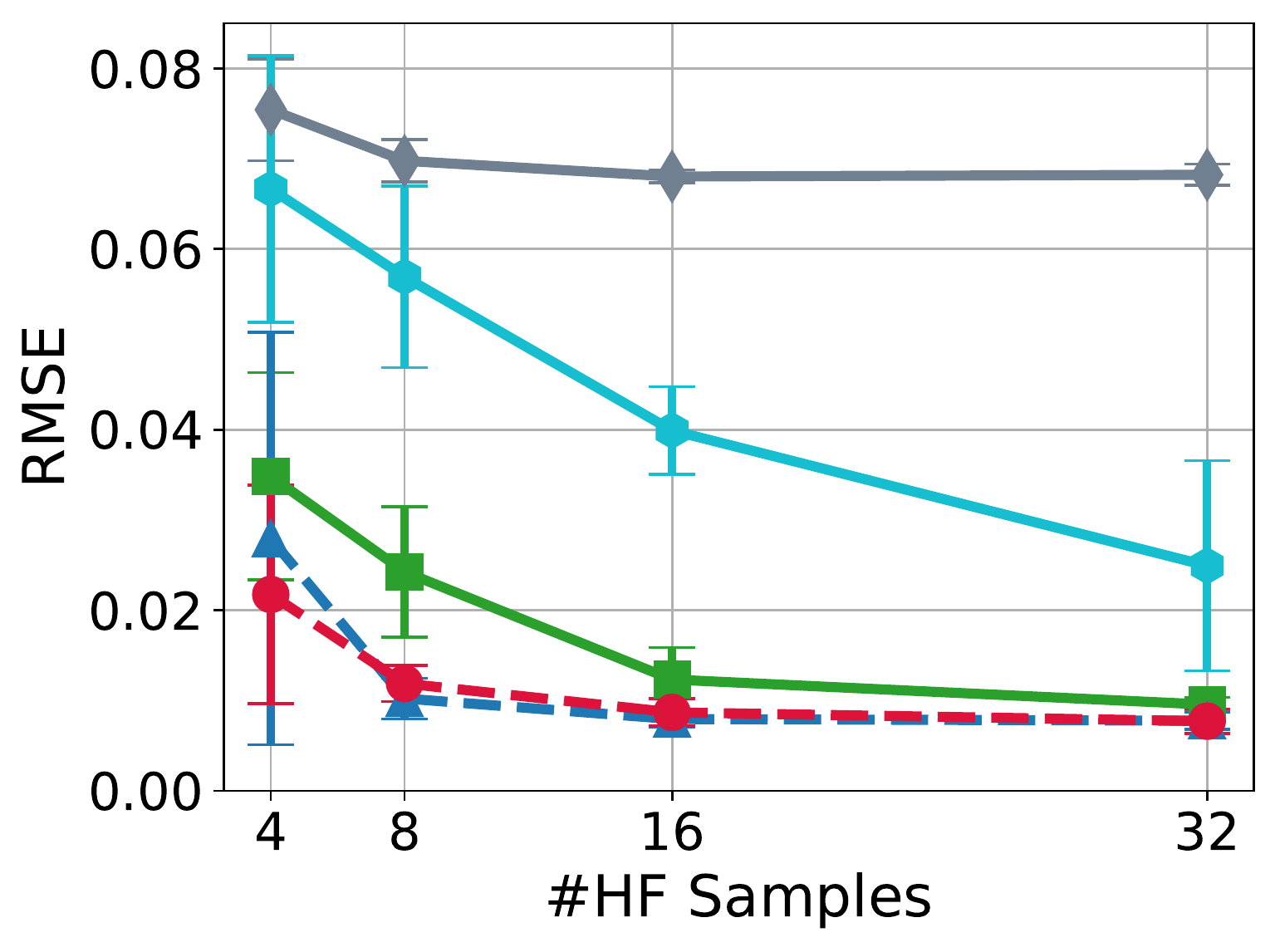}
	\end{subfigure}
	\begin{subfigure}[b]{0.32\linewidth}
		\includegraphics[width=1\textwidth]{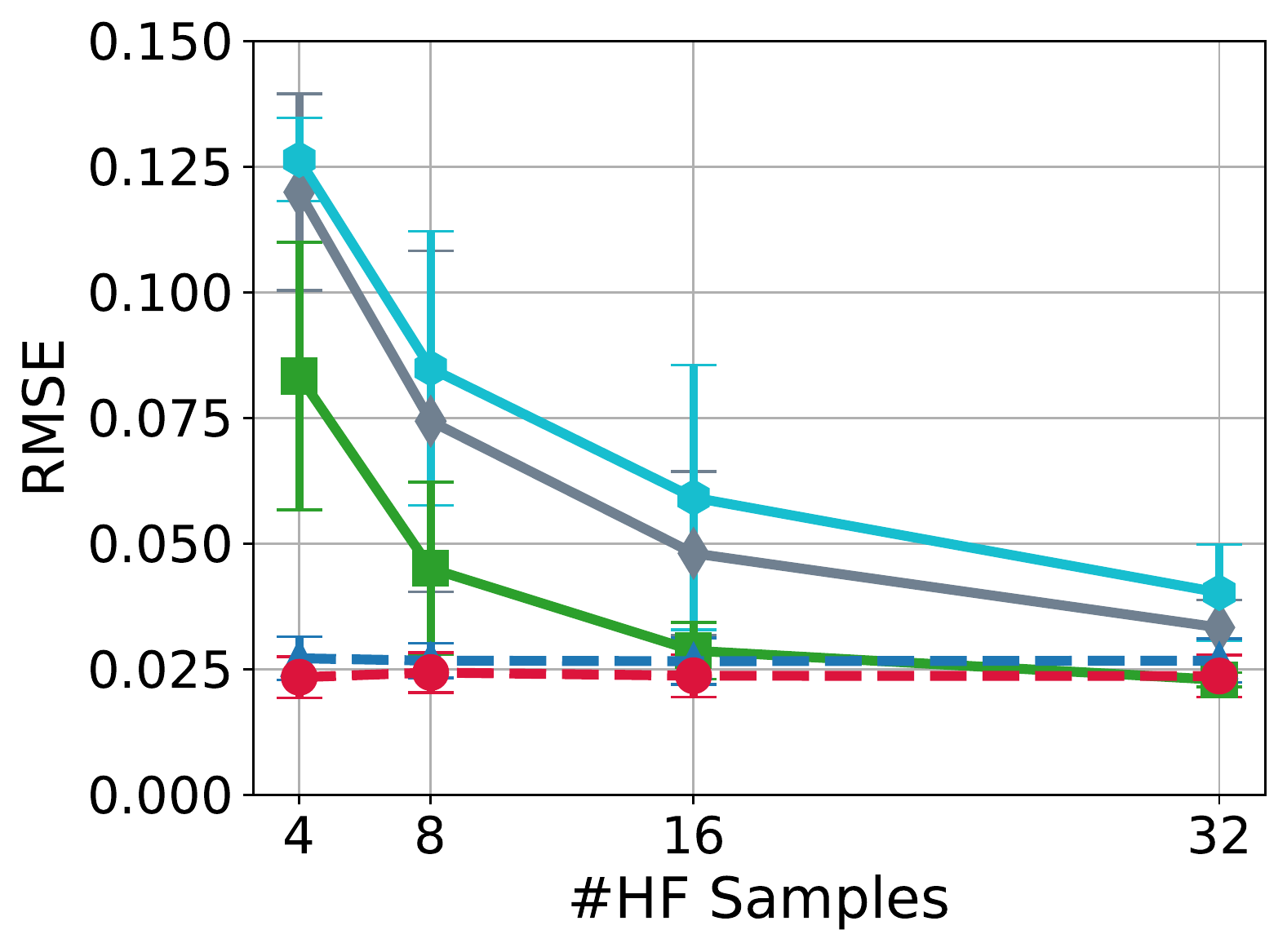}
	\end{subfigure}
	\begin{subfigure}[b]{0.32\linewidth}
		\includegraphics[width=1\textwidth]{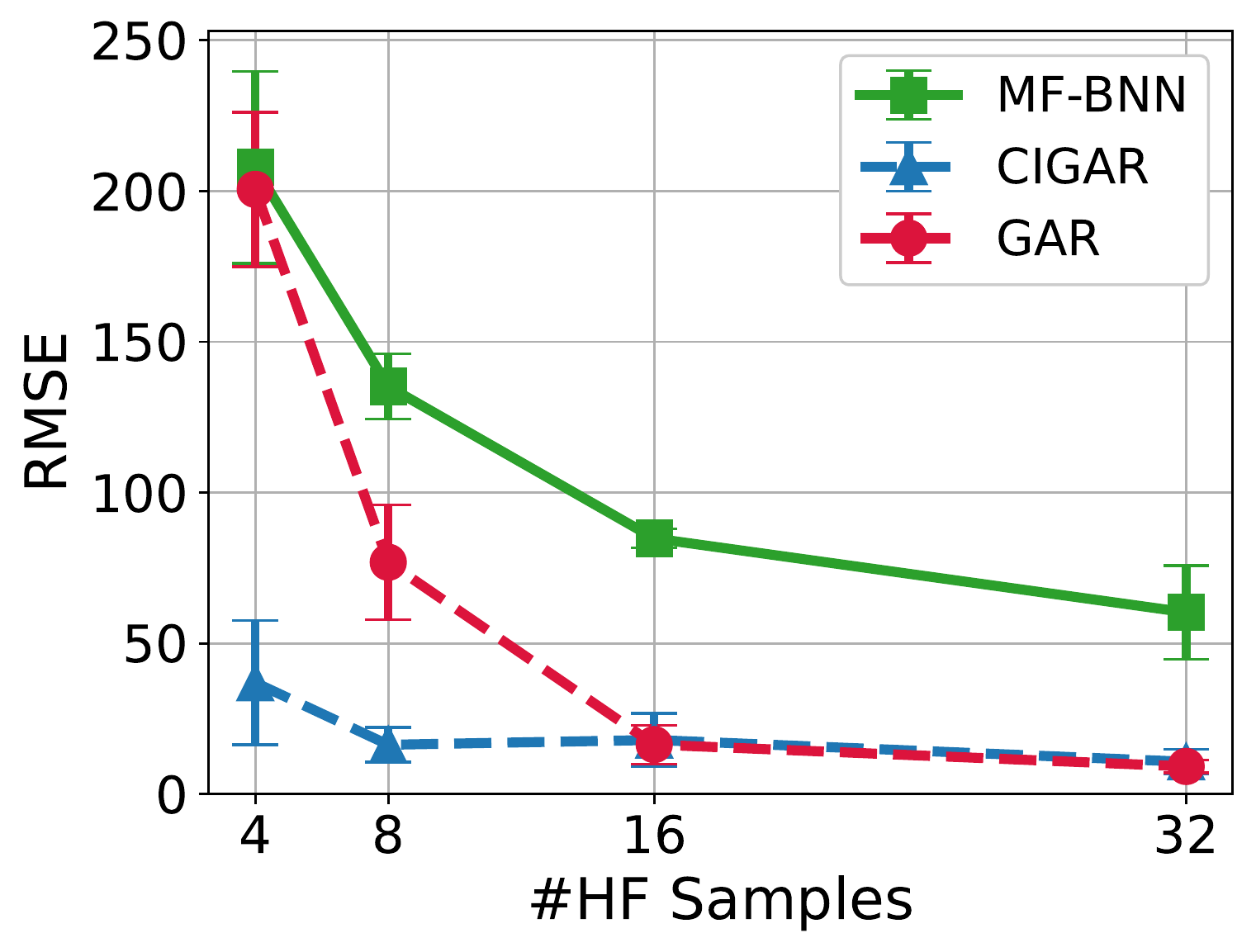}
		\caption{Poisson's equation}
	\end{subfigure}
	\begin{subfigure}[b]{0.32\linewidth}
		\includegraphics[width=1\textwidth]{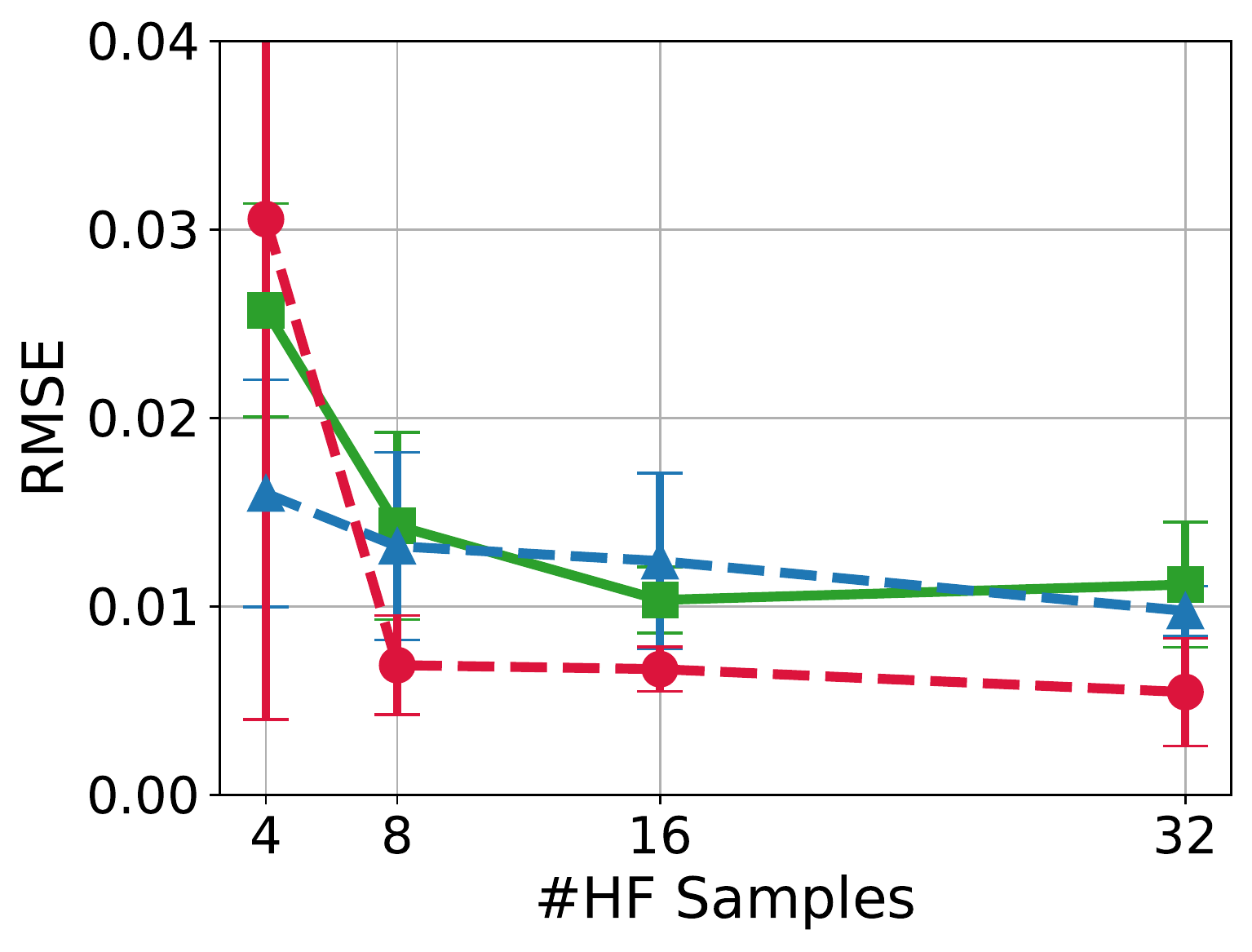}
		\caption{Burger's}
	\end{subfigure}
	\begin{subfigure}[b]{0.32\linewidth}
	\includegraphics[width=1\textwidth]{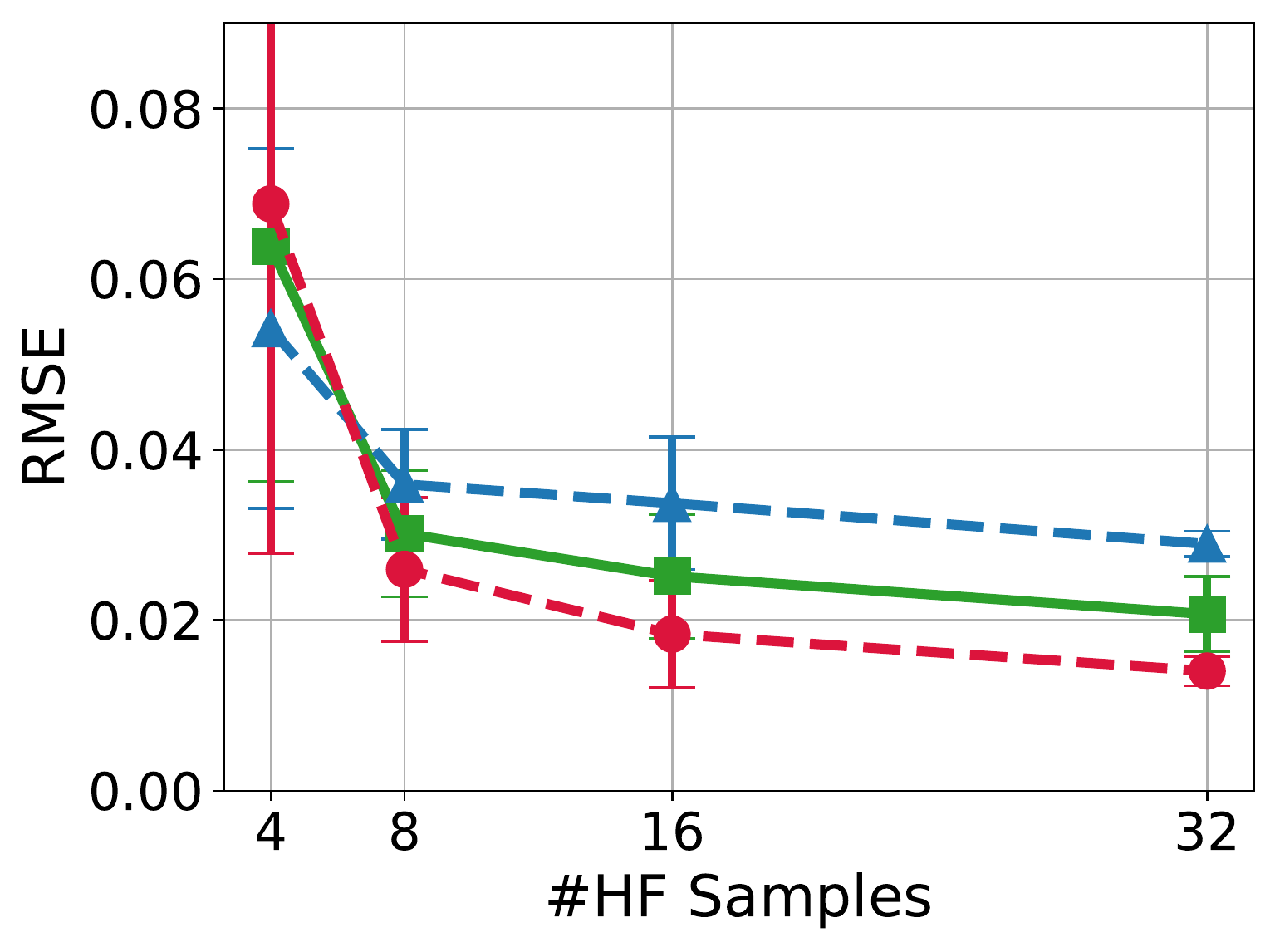}
	\caption{Heat}
	\end{subfigure}
    \caption{RMSE against increasing number of high-fidelity training samples for Poisson's, Burger's, and heat equations with aligned outputs (top row), non-aligned outputs (middle row), and non-subset data (bottom row).}
	\label{fig: rmse canonical}
	\vspace*{-1em}
\end{figure}

\subsection{Multi-Fidelity Fusion for Canonical PDEs}
\label{pdes}
We first assess \ours in canonical PDE simulation benchmarks, which produce high-dimensional spatial/spatial-temporal fields as model outputs.
Specifically, we test on Burger's, Poisson's and the heat equations commonly used in the literature \citep{xing2021deep,tuo2014surrogate,efe2003proper,raissi2017machine}.
The high fidelity results are obtained by solving these equations using finite difference on a $32\times32$ mesh, whereas the low fidelity by an $8\times8$ coarse mesh.
The solutions on these grid points are recorded and vectorized as outputs. Because the mesh differs, the dimensionality varies.
To compare with the standard multi-fidelity method that can only deal with aligned outputs, we use interpolation to upscale the low-fidelity and record them at the high-fidelity grid nodes.
The corresponding inputs are PDE parameters and parameterized initial or boundary conditions. Detailed experimental setups can be found in Appendix \ref{appe pdes}

We uniformly generate $128$ samples for testing and 32 for training. 
We increase the high-fidelity training samples to the number of low-fidelity training samples 32. The comparisons are conducted five times with shuffled samples. The statistical results (mean and std) of the RMSE are reported in \Figref{fig: rmse canonical}.
\ours and \ourss outperform the competitors with a significant margin, up to 6x reduction in RMSE and also reaching the optimal performance with a maximum of 8 high-fidelity samples, indicating a successful fusion of low- and high-fidelity.
\ourss is slightly worse than \ours possibility due to the lack of hard constraints on the orthogonality of its weight matrixes during implementation.
As we have discovered in the literature, AR consistently performs well.
With a flexible linear transformation, \ours outperforms AR while inheriting its robustness, leading to the best performance. 
For the unaligned output, MF-BNN showed slightly worse performance than in the aligned cases, highlighting the challenges of the unaligned outputs.
In contrast, \ours and CIGAR show almost identical performances for both cases. Nevertheless, MF-BNN also shows good performance compared to the rest of the other methods, which is consistent with the finding in \citep{li2020deep}.
It is interesting to see that for the non-subset data, the capable methods show better performances than in the subset cases. \ours and \ourss still outperform the competitors with a clear margin.

To approximately assess the performance under an active learning process. We instead generate training samples in a Sobol sequence \citep{sobol1967distribution}. 
The results are shown in Appendix \ref{appe: PDEs} where \ours and \ourss also outperform the other methods by a large margin.

\subsection{Multi-Fidelity Fusion for Real-World Applications}
{\bf{Optimal topology structure}} is the optimized layout of materials, \eg alloy and concrete, given some design specifications, \eg external force and angle.
This topology optimization is a key technique in mechanical designs, but it is also known for its high computational cost, which renders the need for multi-fidelity fusion.
We consider the topology optimization of a cantilever beam with the location of the point load, the angle of the point load, and the filter radius~\cite{BRUNS20013443} as system inputs.
The low-fidelity use a $16\times16$ regular mesh for the finite element solver, whereas the high-fidelity $64\times64$. Please see Appendix \ref{appe: sec exp topop} for a detailed setup.

As in the previous experiment, the RMSE statistics against an increasing number of high-fidelity training samples are shown in \Figref{fig: topop rmse}.
It is clear that \ours outperforms the competitors with a large margin consistently.
\ourss can be as good as \ours when the number of training samples is large.

\begin{figure}[h]
	\centering
    \begin{subfigure}[b]{0.32\linewidth}
        \includegraphics[width=1\textwidth]{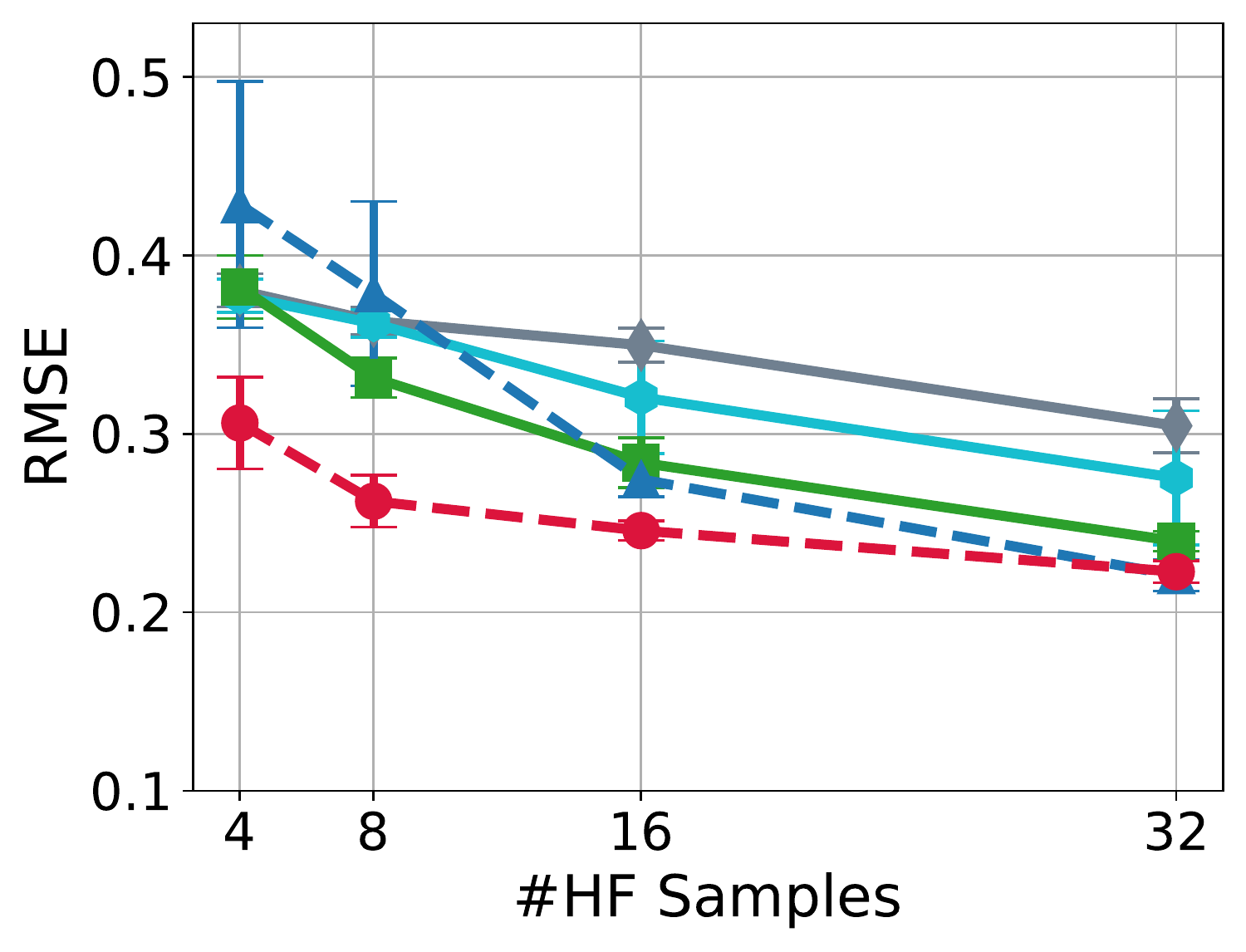}
    \end{subfigure}
	\begin{subfigure}[b]{0.32\linewidth}
		\includegraphics[width=1\textwidth]{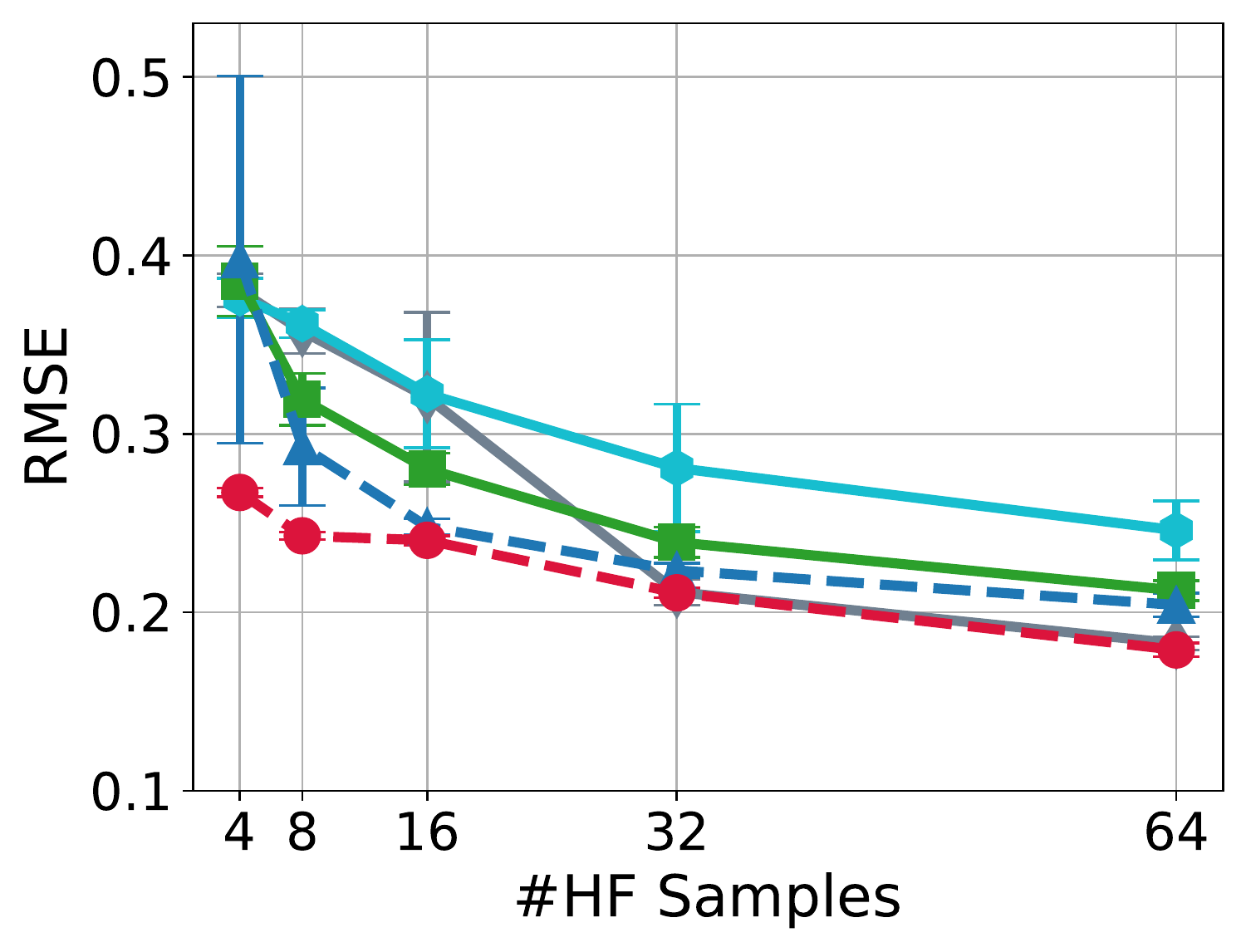}
	\end{subfigure}
    \begin{subfigure}[b]{0.32\linewidth}
	 	\includegraphics[width=1\textwidth]{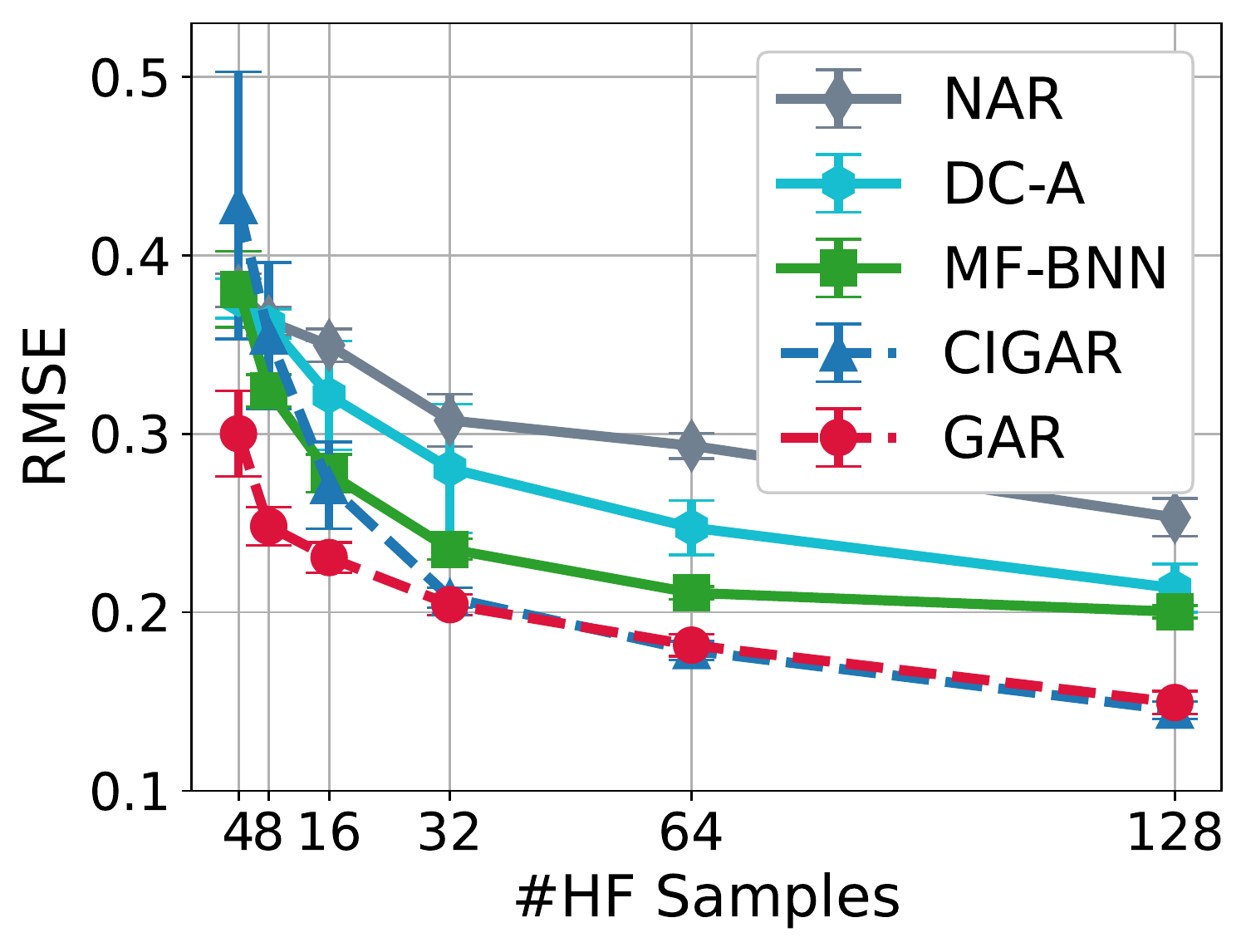}
	\end{subfigure}
	\caption{RMSE  with low-fidelity training sample number fixed to \{32,64,128\}.}
	\label{fig: topop rmse}
\end{figure}

{\bf Steady-state 3D solid oxide fuel cell},
which solves complex coupled PDEs including Ohm’s law, Navier-Stokes equations, Brinkman equation, Maxwell-Stefan diffusion, and convection simultaneously, is a key model for modern fuel cell optimization.
The model was solved using finite elements in COMSOL.
The inputs were taken to be the electrode porosities, cell voltage, temperature, and pressure in the channels.
The low-fidelity experiment is conducted using 3164 elements and relative tolerance of 0.1, whereas the high-fidelity uses 37064 elements and relative tolerance of 0.001.
The outputs are the coupled fields of electrolyte current density (ECD) and ionic potential (IP) in the $x-z$ plane located at the center of the channel.

The RMSE statistics are shown in \Figref{fig: sofc rmse a}, which, again, highlights the superiority of the proposed method with only four high-fidelity data training samples.
To further assess the model capacity for non-structured outputs, we keep only the ECD (\Figref{fig: sofc rmse b}) and IP (\Figref{fig: sofc rmse c}) in the low-fidelity training data to rise the challenges of predicting high fidelity ECD+IP fields. We can see that removing some low-fidelity information indeed increases the difficulties, especially when removing ECD, where MF-BNN outperforms \ours and \ourss with a small number of training data. As soon as the training number increases, \ours and \ourss become superior again.

{\bf Plasmonic nanoparticle arrays} is a complex physical simulation that 
calculates the extinction and scattering efficiencies $Q_{ext}$ and $Q_{sc}$ for plasmonic systems with varying numbers of scatterers using coupled dipole approximation (CDA), which is a method for mimicking the optical response of an array of similar, non-magnetic metallic nanoparticles with dimensions far smaller than the wavelength of light (here 25 nm).
$Q_{ext}$ and $Q_{sc}$ are defined as the QoIs in this work.
Please see Appendix \ref{appe: palas} for detailed experiment setup.
We conducted the experiments 5 times with shuffled samples, and we fixed the number of low-fidelity training samples to 32, 64, and 128 and gradually increase the high-fidelity training data from 4 to 32, 64, and 128. We can see in Fig.\ref{fig: real-data}, \ours outperforms others by a clear margin, especially when the high-fidelity data contains only 4 samples. When there is a large training sample dataset, \ourss can be as excellent as \ours.
\begin{figure}[h]
		\begin{subfigure}[b]{0.32\linewidth}
			\includegraphics[width=1\textwidth]{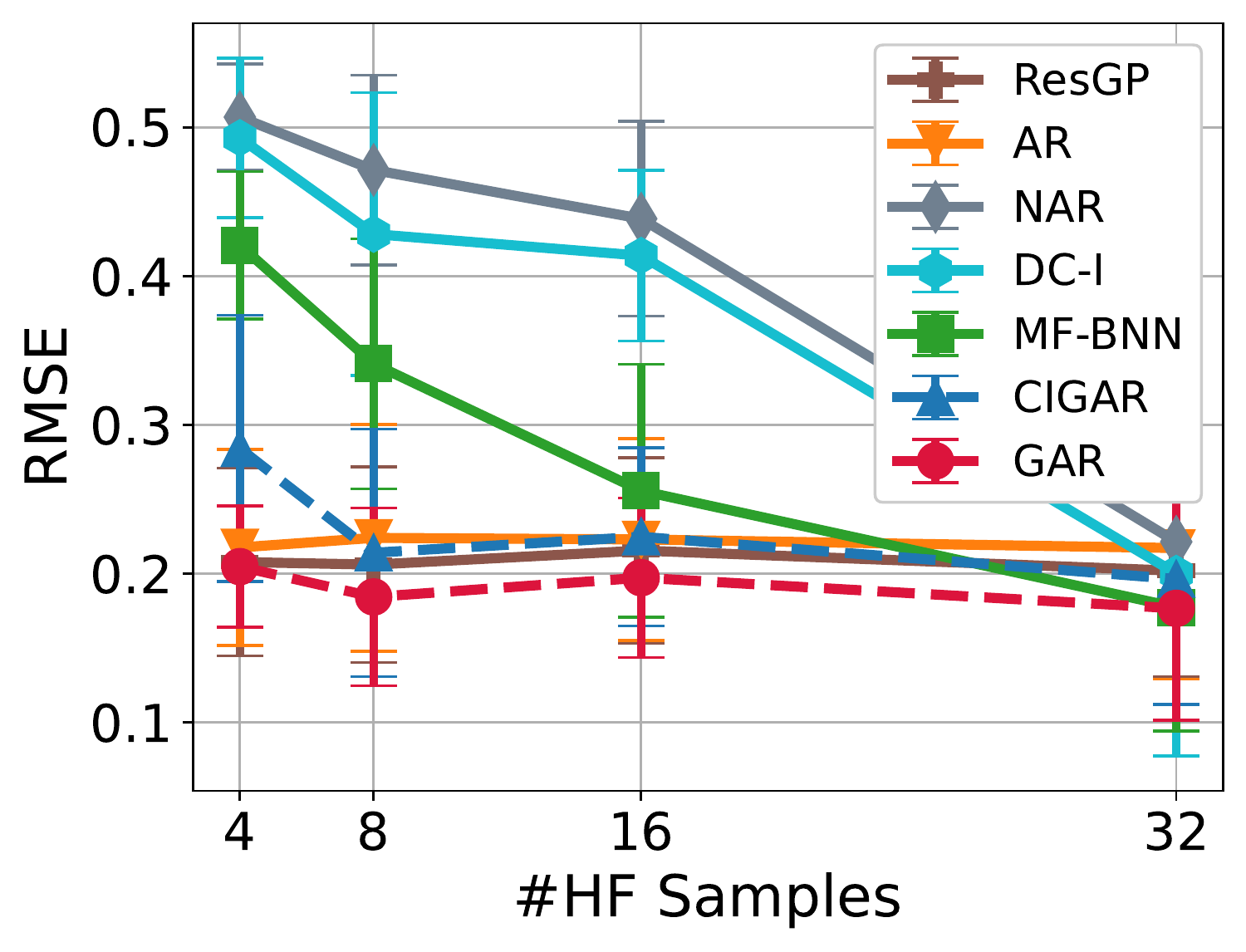}
		\end{subfigure}
		\begin{subfigure}[b]{0.32\linewidth}
		\includegraphics[width=1\textwidth]{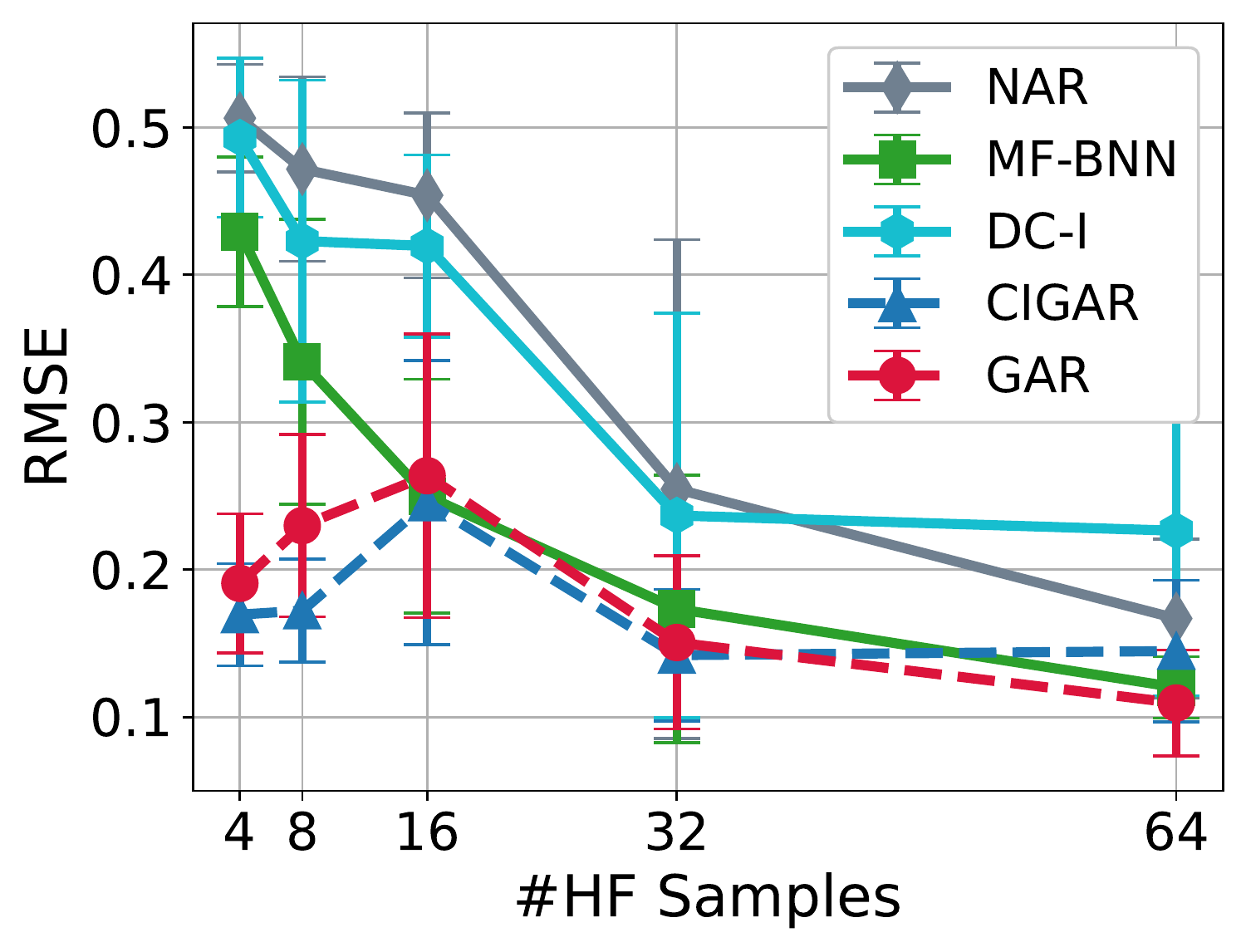}
		\end{subfigure}
		\begin{subfigure}[b]{0.32\linewidth}
		\includegraphics[width=1\textwidth]{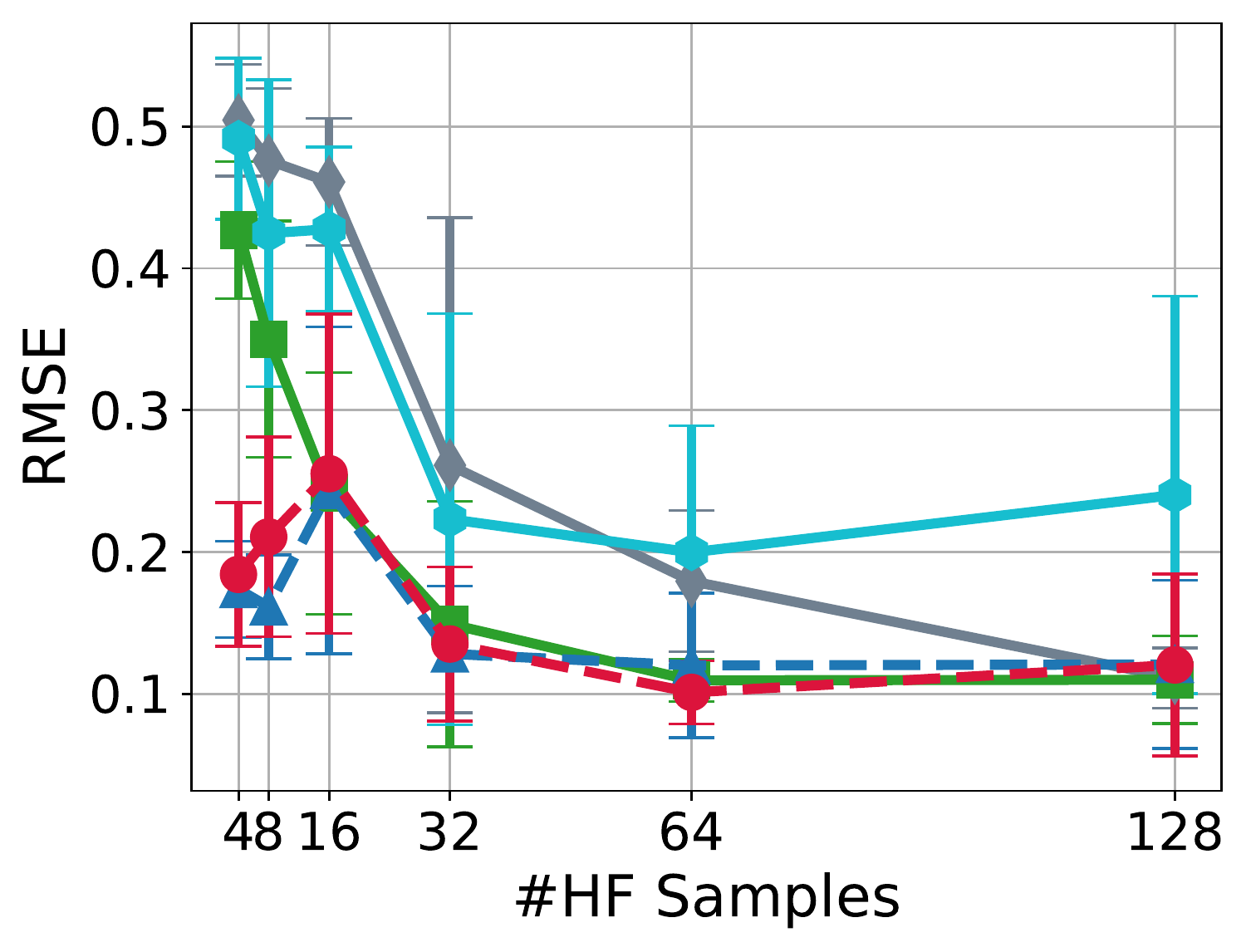}
		\end{subfigure}
		\caption{RMSE against an increasing number of high-fidelity training samples with low-fidelity training sample number fixed to \{32,64,128\} for Plasmonic nanoparticle arrays simulations.}
		\label{fig: real-data}
\end{figure}

{
\subsection{Stability Test}
As a non-parametric model, we expect \ours and \ourss to have stable performance against overfitting compared to the NN-based methods.
In this section, we show the testing RMSE against the training epoch for \ours, \ourss, and MF-BNN for the previous Poisson's equation, SOFC, and topology optimization. 
The experiments are repeated five times to ensure fairness.
The results are shown in \Figref{fig: epoch rmse}.
We can see clearly that \ours and \ourss are more stable than MF-BNN in almost all cases. 
The most notables are the converge rate of \ours and \ourss, which is more than 10x faster if we look at the topology optimization and SOFC cases. 
For Poisson's equation, the MF-BNN is not likely to match the performance of \ours and \ourss regardless of the large number of training epochs being used.
For the SOFC, MF-BNN, and topology optimization, MF-BNN might be able to match \ours given a very large epoch number, making it a bad choice that consumes expensive computational resources.
\begin{figure}[]
	\centering
	\begin{subfigure}[b]{0.32\linewidth}
		\includegraphics[width=1\textwidth]{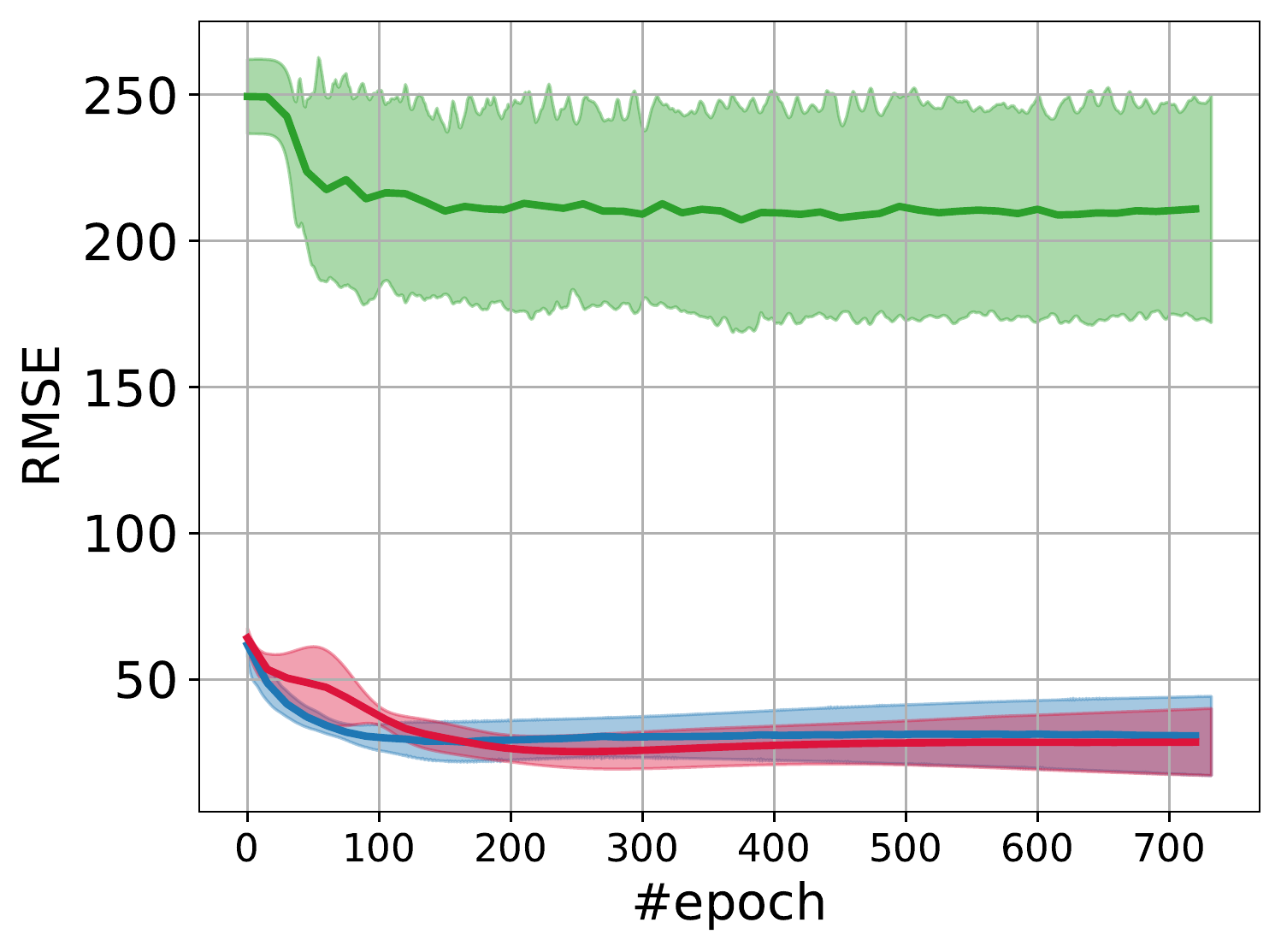}
		\caption{Poisson's}
	\end{subfigure}
	\begin{subfigure}[b]{0.32\linewidth}
		\includegraphics[width=1\textwidth]{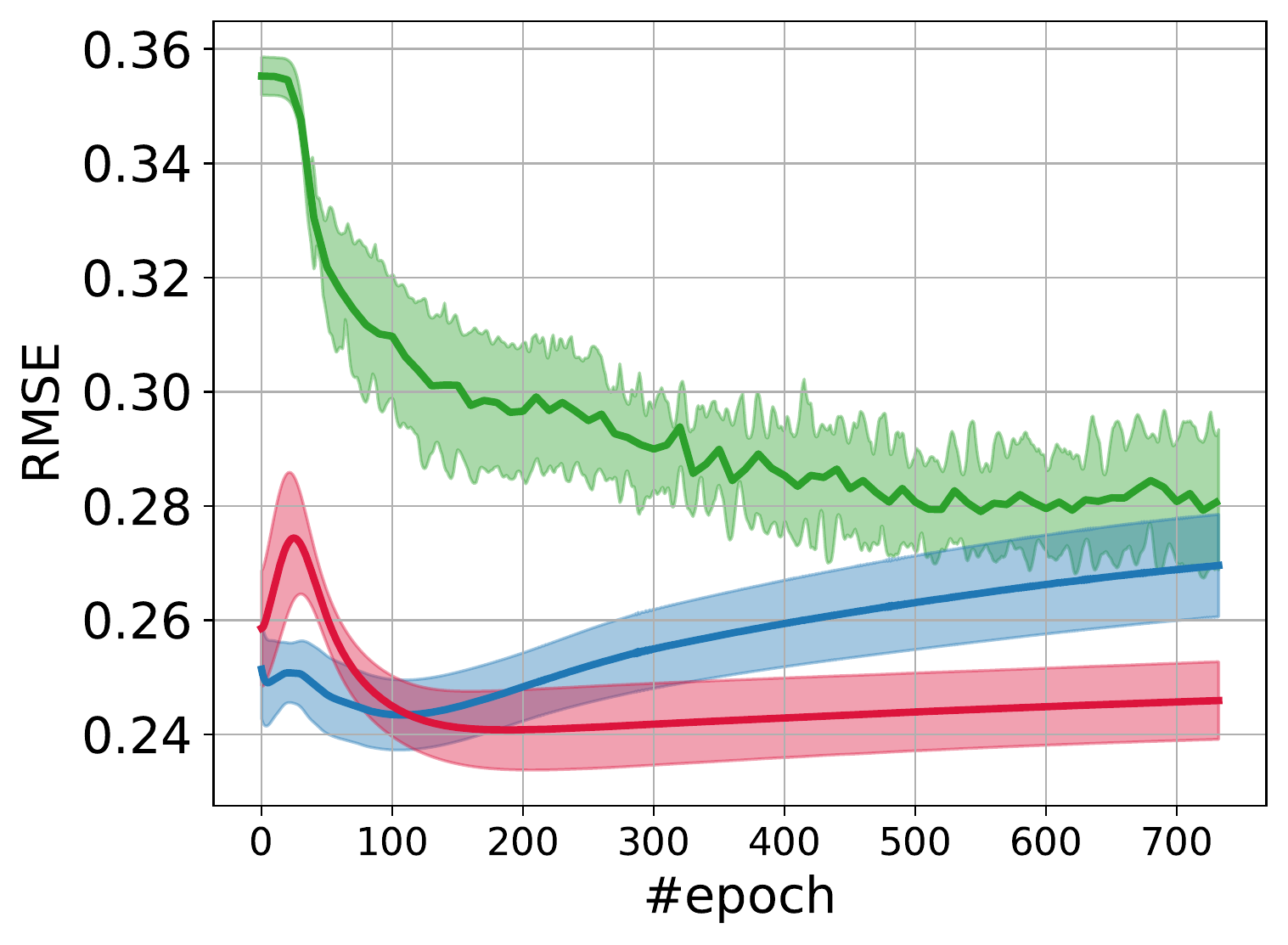}
		\caption{Topology optimization}
	\end{subfigure}
	\begin{subfigure}[b]{0.32\linewidth}
		\includegraphics[width=1\textwidth]{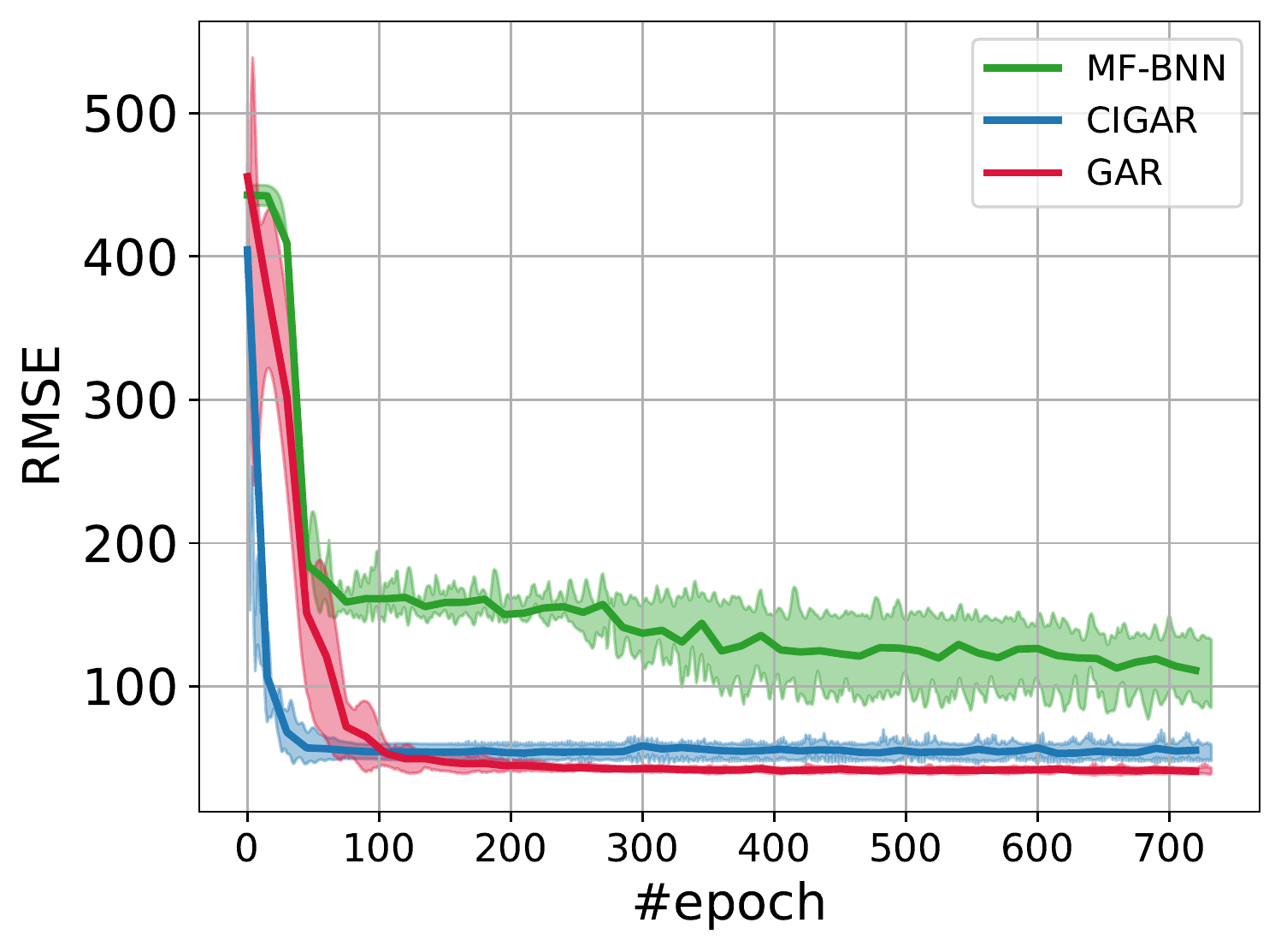}
		\caption{SOFC}
	\end{subfigure}
	\caption{Testing RMSE against an increasing number of training epochs for SOFC, topology optimization, and Poisson datasets.}
	\label{fig: epoch rmse}
\end{figure}
}

\begin{figure}[h]
	\centering
	\begin{subfigure}[b]{0.32\linewidth}
		\includegraphics[width=1\textwidth]{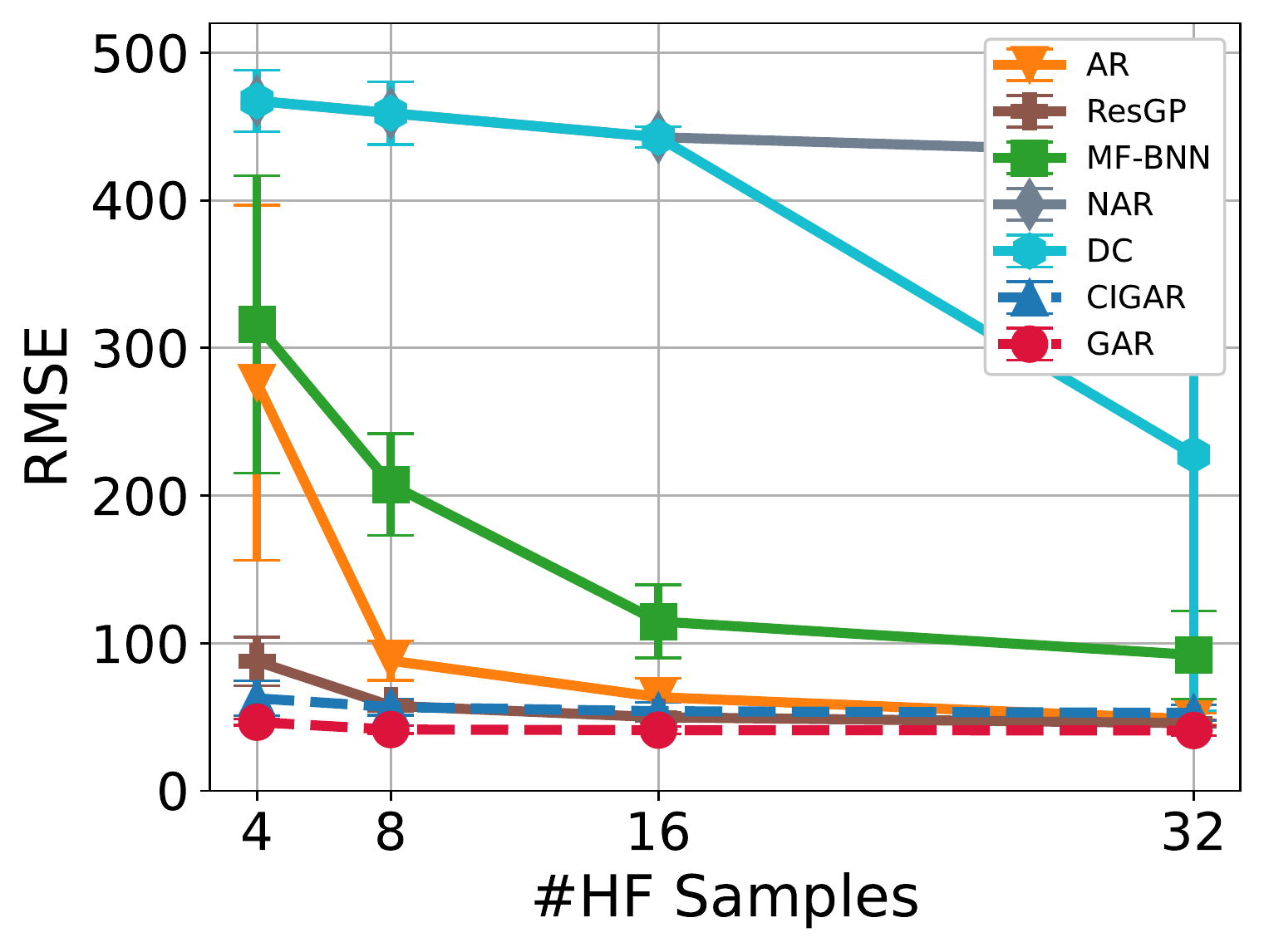}
		\caption{ECD+IP}
        \label{fig: sofc rmse a}
	\end{subfigure}
	\begin{subfigure}[b]{0.32\linewidth}
		\includegraphics[width=1\textwidth]{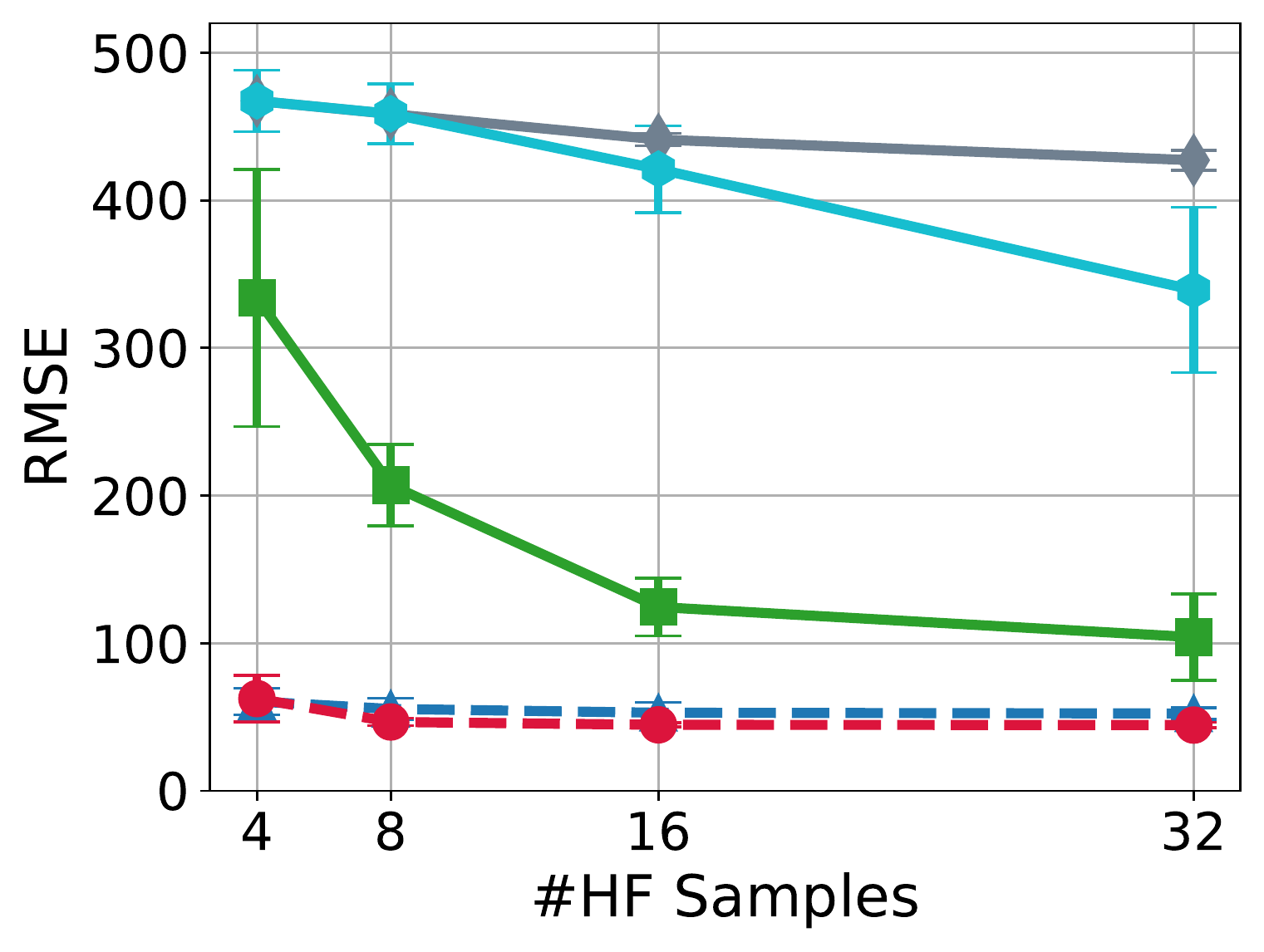}
        \caption{ECD}
        \label{fig: sofc rmse b}
	\end{subfigure}
	\begin{subfigure}[b]{0.32\linewidth}
		\includegraphics[width=1\textwidth]{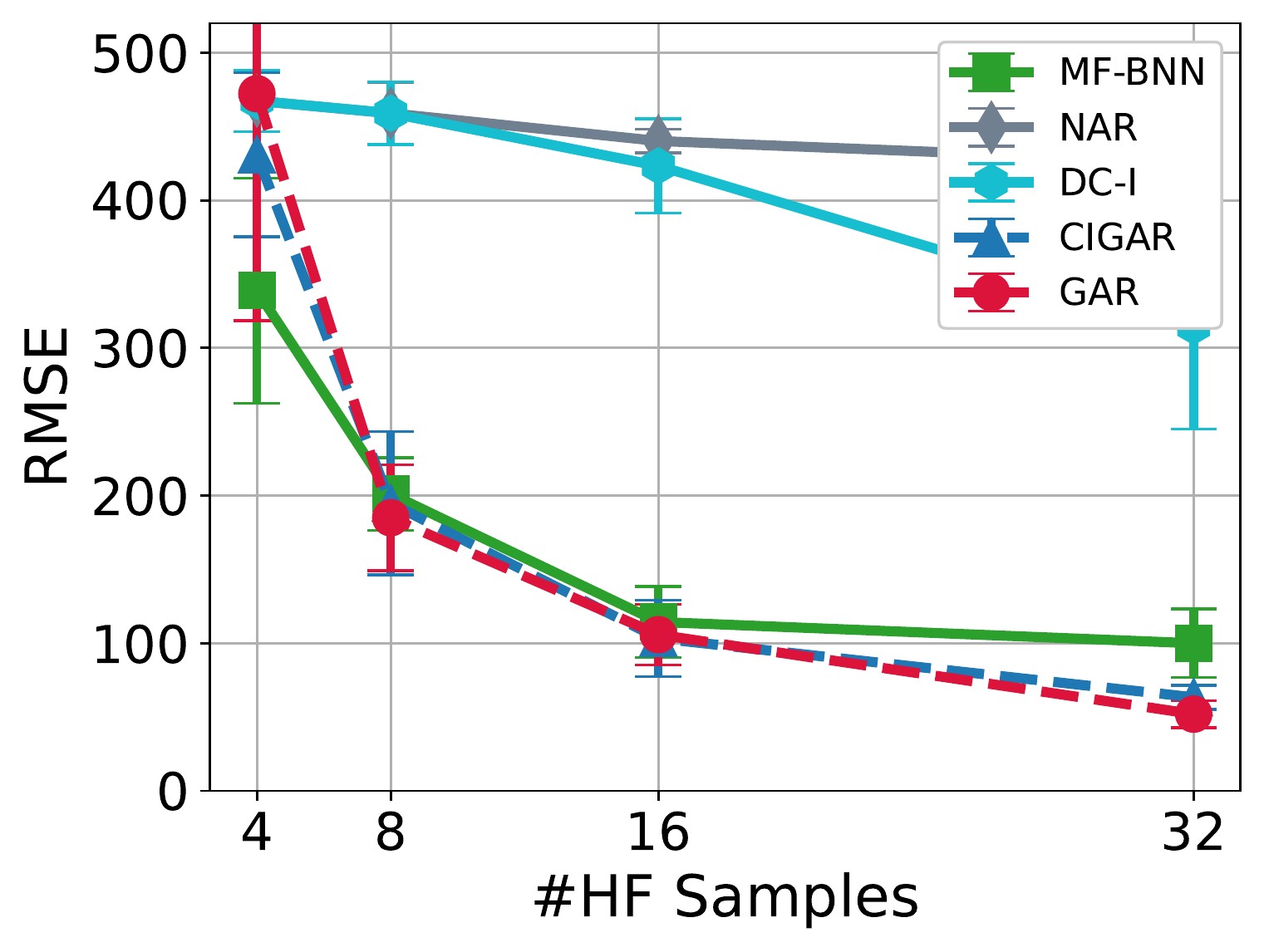}
        \caption{IP}
        \label{fig: sofc rmse c}
	\end{subfigure}
	\caption{RMSE for SOFC with low-fidelity training sample number fixed to 32.}
	\label{fig: sofc rmse}
\end{figure}

\vspace*{-1em}
\section{Conclusion}
\vspace*{-1em}
We propose \ours, the first AR generalization to arbitrary outputs and non-subset multi-fidelity data with a closed-form solution, and \ourss, an efficient implementation by revealing the autokrigeability in AR. Limitation of this work is scalability w.r.t to samples, lack of active learning \citep{li2020deep}, and the applications to broader problems of time series and transfer learning\citep{requeima2019gaussian,xia2020gaussian} using AR-based methods.

\small
\bibliographystyle{unsrtnat}

\medskip

\section*{Checklist}

The checklist follows the references.  Please
read the checklist guidelines carefully for information on how to answer these
questions.  For each question, change the default \answerTODO{} to \answerYes{},
\answerNo{}, or \answerNA{}.  You are strongly encouraged to include a {\bf
justification to your answer}, either by referencing the appropriate section of
your paper or providing a brief inline description.  For example:
\begin{itemize}
  \item Did you include the license to the code and datasets? \answerYes{See Section~\ref{gen_inst}.}
  \item Did you include the license to the code and datasets? \answerNo{The code and the data are proprietary.}
  \item Did you include the license to the code and datasets? \answerNA{}
\end{itemize}
Please do not modify the questions and only use the provided macros for your
answers.  Note that the Checklist section does not count towards the page
limit.  In your paper, please delete this instructions block and only keep the
Checklist section heading above along with the questions/answers below.

\begin{enumerate}

\item For all authors...
\begin{enumerate}
  \item Do the main claims made in the abstract and introduction accurately reflect the paper's contributions and scope?
    \answerYes{See contributions, abstract, and introduction}
  \item Did you describe the limitations of your work?
    \answerYes{See the conclusion and complexity analysis section}
  \item Did you discuss any potential negative societal impacts of your work?
    \answerNA{We do not see an obvisou negative societal impacts as it is fundamental and quite theoretical. }
  \item Have you read the ethics review guidelines and ensured that your paper conforms to them?
    \answerYes{}
\end{enumerate}

\item If you are including theoretical results...
\begin{enumerate}
  \item Did you state the full set of assumptions of all theoretical results?
    \answerYes{}
        \item Did you include complete proofs of all theoretical results?
    \answerYes{Please see Appendix}
\end{enumerate}

\item If you ran experiments...
\begin{enumerate}
  \item Did you include the code, data, and instructions needed to reproduce the main experimental results (either in the supplemental material or as a URL)?
    \answerYes{Please see supplementary materials}
  \item Did you specify all the training details (e.g., data splits, hyperparameters, how they were chosen)?
    \answerYes{Please see the Appendix}
        \item Did you report error bars (e.g., with respect to the random seed after running experiments multiple times)?
    \answerYes{Please see the experimental section}
        \item Did you include the total amount of compute and the type of resources used (e.g., type of GPUs, internal cluster, or cloud provider)?
    \answerYes{Please see the experimental section}
\end{enumerate}

\item If you are using existing assets (e.g., code, data, models) or curating/releasing new assets...
\begin{enumerate}
  \item If your work uses existing assets, did you cite the creators?
    \answerYes{Please see the experimental section}
  \item Did you mention the license of the assets?
    \answerYes{}
  \item Did you include any new assets either in the supplemental material or as a URL?
    \answerYes{}
  \item Did you discuss whether and how consent was obtained from people whose data you're using/curating?
    \answerYes{see experiment senction and the Appendix}
  \item Did you discuss whether the data you are using/curating contains personally identifiable information or offensive content?
    \answerNA{We do not use such data}
\end{enumerate}

\item If you used crowdsourcing or conducted research with human subjects...
\begin{enumerate}
  \item Did you include the full text of instructions given to participants and screenshots, if applicable?
    \answerNA{}
  \item Did you describe any potential participant risks, with links to Institutional Review Board (IRB) approvals, if applicable?
    \answerNA{}
  \item Did you include the estimated hourly wage paid to participants and the total amount spent on participant compensation?
    \answerNA{}
\end{enumerate}

\end{enumerate}

\appendix
\setcounter{lemma}{0}
\setcounter{equation}{0}
\renewcommand{\theequation}{A.\arabic{equation}}

\clearpage
\rule{\linewidth}{4pt}
\begin{center}
\section*{\hfil \LARGE Appendix \hfil}
\end{center}
\rule{\linewidth}{1pt}

\setcounter{equation}{0}

\section{Gaussian process}
\label{appe gp}
Gaussian process (GP) is a typical choice for the surrogate model because of its model capacity for complicated black-box functions and uncertainty quantification.
Consider, for the time being, a simplified scenario in which we have noise-contaminated observations $\{y_i=g({\x}_i) + \epsilon_i\}_{i=1}^N$.
In a GP model, a prior distribution is placed over $f({\x})$, indexed by ${\x}$:

\begin{equation}\label{scalarem}
\eta({\x})| {\btheta}\sim \mathcal{GP} \left(m({\x}),k({\x}, {\x}'|{\btheta})\right),
\end{equation}
with mean and covariance functions:
\begin{equation}
\begin{array}{ll}
m_0({\x}) &=\mathbb{E}[f({\x})],\\
k({\x},{\x}'|\pmb{\theta}) &=\mathbb{E}[(f({\x})-m_0({\x}))(f({\x}')-m_0({\x}'))],
\end{array}
\end{equation}

where $\mathbb{E}[\cdot]$ is the expectation and $\pmb{\theta}$ are the hyperparameters that control the kernel function. By centering the data, the mean function may be assumed to be an equal constant, $m_0({\x})\equiv m_0$. Alternative options are feasible, such as a linear function of ${\x}$, but they are rarely used until previous knowledge of the shape of the function is provided. The covariance function can take several forms, with the automated relevance determinant (ARD) kernel being the most popular.
\begin{equation}\label{covfunc}
k({\x}, {\x}'|\pmb{\theta})=\theta_0\exp\left(-({\x}-{\x}')^T\mbox{diag}(\theta_1^{-2},\hdots,\theta_l^{-2})({\x}-{\x}')\right).
\end{equation}
From this point on, we eliminate the explicit notation of $k(x, x')$'s reliance on $\pmb{\theta}$.
In this instance, the hyperparameters $\theta_1,\hdots,\theta_l$ are referred to as length-scales. For constant parameter ${\x}$, $f({\x})$ is its random variable. In contrast, a collection of values, $f({\x}_i)$, $i=1,\hdots,N$, is a partial realization of the GP. GP's realizations are functions of $x$ that are deterministic. The primary characteristic of GPs is that the joint distribution of $eta(x i), i=1,hdots,N$ is multivariate Gaussian.

Assuming the model deficiency $\varepsilon \sim \mathcal{N}(0,\sigma^2)$ is likewise Gaussian, we can derive the model likelihood using the prior \eqref{scalarem} and available data.
\begin{equation} \label{MLE}
    \begin{aligned}
        \Lcal & \triangleq  p(\y|\x,\btheta) = \int (f(\x) + \varepsilon) d f
        = \mathcal{N}(\y|m_0 \1, \textbf{K} +\sigma^2\I) \\
        &= -\frac{1}{2} \left( {\y} - {m_0} {\bf 1} \right)^T  (\textbf{K} +\sigma^2\I) ^{-1} \left( {\y}-{m_0} {\bf 1} \right) \\
        & \quad -\frac{1}{2}\ln|\textbf{K} + \sigma^2\I|  - \frac{N}{2} \log(2\pi),
    \end{aligned}
\end{equation}
where $\K=[K_{ij}]$ is the covariance matrix, in which $K_{ij}=k({\x}_i, {\x}_j)$, $i,j=1,\hdots,N$.
The hyperparameters $\pmb{\theta}$ is often derived from point estimations using the maximum likelihood (MLE) of \Eqref{MLE} w.r.t. $\btheta$.
The joint distribution of $\y$ and $f(\x)$ is also a joint Gaussian distribution with mean value $m_0 \1$ and covariance matrix.
\begin{equation}
    \begin{array}{c}\label{cm3}
    \displaystyle \textbf{K}'=
    \left[ 
        \begin{array}{c|c} \textbf{K} + \sigma^2\I & {\k}({\x})  \\ 
        \hline
        {\k}^T({\x})& k({\x}, {\x})  + \sigma^2
    \end{array}
    \right],
    \end{array}
\end{equation}
where ${\k}({\x})=(k({\x}_1, {\x}),\hdots, k({\x}_N, {\x}))^T$.
Conditioning on $\y$, the conditional predictive distribution at $\x$ is obtained.
\begin{equation}\label{postpredA}
\begin{array}{c}
\hat{f}({\x})|{\y} \sim \mathcal{N}\left(\mu ({\x} ), v ({\x},{\x}' )\right),
\vspace{2mm}\\
\mu ({\x} )= m_0 {\bf 1} + {\k}({\x})^T \left( \textbf{K} + \sigma^2 \I \right)^{-1} \left( \y - {m_0} {\bf 1} \right),
\vspace{2mm}\\
 v ({\x} )=  \sigma^2 + k({\x}, {\x} ) -{\k}^T({\x}) \left( \textbf{K} + \sigma^2 \I \right)^{-1} {\k}({\x}).
\end{array}
\end{equation}
The expected value $\mathbb{E}[f({\x})]$ is given by $\mu({\x} )$ and the predictive variance by $v({\x})$. 
From \Eqref{cm3} to \Eqref{postpredA} is crucial since the prediction posterior of this wake is based on a comparable block covariance matrix.

\section{Proof of Theorem}
\begin{lemma}\label{lemma1: LAR}\citep{perdikaris2017nonlinear}
	If $\X^h \subset \X^l$, the joint likelihood of AR can be decomposed into two independent likelihoods of the low- and high-fidelity. 
\end{lemma}
This lemma has been proven by \citep{legratiet2013bayesian}.
However, the notation and derivation is not easy to follow.
To layout the foundations of \ours, we prove it using a clearer way with friendly notations.
\begin{proof}
Following \Eqref{eq:lar joint prob}, the inversion of the covariance matrix is
\[\bSigma^{-1} = \left(
\begin{array}{cc}
	{(\K^l)^{-1}+\left(
		\begin{array}{cc}
			{0} & {0} \\
			{0} & \rho^2(\K^r)^{-1}
		\end{array}
		\right)} & {-\left(
		\begin{array}{c}
			{0} \\
			\rho(\K^r)^{-1}
		\end{array}
		\right)} \\
	{-\left({0} \quad\rho(\K^r)^{-1}\right)} & {(\K^r)^{-1}}
\end{array}
\right).\]
	
We can write down the log-likelihood for all the low- and high-fidelity observations as,
\begin{equation}
	\begin{aligned}
	&\log p({\Y}^l, {\Y}^h)\\
	=&-\frac{N^h + N^l}{2}\log(2\pi)-\frac{1}{2}\log|\bSigma|-\frac{1}{2}({\Y}^l, \rho\E^T{\Y}^l+{\Y}^r)^T\bSigma^{-1}\left(\begin{array}{c}{\Y}^l\\ \rho\E^T{\Y}^l+{\Y}^r\end{array}\right)\\
	=&-\frac{1}{2}\log|\bSigma|-\frac{N^h + N^l}{2}\log(2\pi)
	-\frac{1}{2}[({\Y}^l)^T(\K^l)^{-1}\Y^l
	+({\Y}^l)^T\left(\begin{array}{cc}\0 & \0 \\ \0 & \rho^2{(\K^r)}^{-1}\end{array}\right)\Y^l\\
	&-\rho({\Y}^l)^T\E{\left({0},\ \rho(\K^r)^{-1}\right)}\Y^l
	-{\left({0},\ (\Y^r)^T\rho(\K^r)^{-1}\right)}\Y^l -({\Y}^l)^T\E\rho{\K_r}^{-1}(\rho\E^T{\Y}^l+{\Y}^r)\\
	& + \rho\Y^l\E(\K^r)^{-1}(\rho\E^T{\Y}^l+{\Y}^r) + \Y^r(\K^r)^{-1}(\rho\E^T{\Y}^l+{\Y}^r)]\\
	=&-\frac{1}{2}\log|\bSigma|-\frac{N^h + N^l}{2}\log(2\pi)
	-\frac{1}{2}[({\Y}^l)^T(\K^l)^{-1}\Y^l
	-{\left({0},\ (\Y^r)^T\rho(\K^r)^{-1}\right)}\Y^l\\
	&+\Y^r(\K^r)^{-1}\rho\E^T{\Y}^l+\Y^r(\K^r)^{-1}{\Y}^r]\\
	=&-\frac{1}{2}\log|\K^l|-\frac{1}{2}\log|\K^r|-\frac{N^l+N^h}{2}log(2\pi)-\frac{1}{2}({\Y}^l)^T(\K^l)^{-1}\Y^l-\frac{1}{2}\Y^r(\K^r)^{-1}{\Y}^r\\
	=&\underbrace{-\frac{N^L}{2}log(2\pi)-\frac{1}{2}\log|\K^l|-\frac{1}{2}({\Y}^l)^T(\K^l)^{-1}{\Y}^l}_{\Lcal^l}
	\underbrace{-\frac{N^h}{2}log(2\pi)-\frac{1}{2}\log|\K^r|-\frac{1}{2}({\Y}^r)^T(\K^r)^{-1}{\Y}^r}_{\Lcal^r}
	\end{aligned}
\end{equation}
\end{proof}
where $\Y^r =  \Y^h- \rho\E^T \Y^l$, ${\Lcal^l}$ is the log-likelihood of the low-fidelity data with the lower fidelity kernel, and ${\Lcal^r}$ is the log-likelihood of the residual data with the residual kernel; ${\Lcal^l}$ and ${\Lcal^r}$ are independent and thus can be trained in parallel.

Based on the joint probability \Eqref{eq:lar joint prob}, we can similarly derive the predictive posterior distribution of the high-fidelity using the standard GP posterior derivation. 
Conditioning on $\Y^h$ and $\Y^l$, the predictive high-fidelity posterior for a new input $\x_*$ is also a Gaussian $\Ncal({\mu}_*^h, \sigma^h_*)$:
\begin{equation}
	\begin{aligned}
		{\mu^h_*}=&\left(\rho \k^l(\x_*,\X^l),\ \ 
		\rho^2 \k^l(\x_*,\X^h) + \k^r(\x_*,\X^h)\right)\K^{-1}\left(\begin{array}{c}{\Y}^l \\ \rho\E^T {\Y}^l + {\Y}^r\end{array}\right)\\
		=&\rho\k^l(\x_*,\X^l)\K^l(\X^l,\X^l)^{-1}{\Y}^l\\
		&+\rho^3\k^l(\x_*,\X^l)\E(\K^r)^{-1}\E^T{\Y}^l-
		\rho^3\k^l(\x_*,\X^h)(\K^r)^{-1}\E^T{\Y}^l\\
		&-\rho\k^r(\x_*,\X^h)(\K^r)^{-1}\E^T{\Y}^l-\rho^2\k^l(\x_*,\X^l)\E(\K^r)^{-1}\left[\rho\E^T{\Y}^l+{\Y}^r\right]\\
		&+\rho^2\k^l(\x_*,\X^h)(\K^r)^{-1}\left[\rho\E^T{\Y}^l +{\Y}^r\right]
		+\k^r(\x_*,\X^h)(\K^r)^{-1}\left[\rho\E^T{\Y}^l + {\Y}^r\right]\\
		=&\left[\rho\k^l(\x_*,\X^l)(\K^l)^{-1}\right]{\Y}^l+\k^r(\x_*,\X^h)(\K^r)^{-1}{\Y}^r\\
		&-\left[\rho \k^r(\x_*,\X^h)(\K^r)^{-1}\right]\E{\Y}^l+\left[\rho\k^r(\x_*,\X^h)(\K^r)^{-1}\right]\E{\Y}^l\\
		=&\left[\rho\k^l(\x_*,\X^l)(\K^l)^{-1}\right]{\Y}^l+\k^r(\x_*,\X^h)(\K^r)^{-1}{\Y}^r\\
	\end{aligned}
\end{equation}
and 
\begin{equation}
	\begin{aligned}
		{\sigma_*^h} = &\left(\rho^2 \k^l(\x_*, \x_*)+\k^r(\x_*,\x_*)\right) - (\rho\k_*^l, \rho^2\k_*^l(\X^h)+\k_*^r)^T\K^{-1}(\rho\k_*^l, \rho^2\k_*^l(\X^h)+\k_*^r)\\
		=& \left(\rho^2\k^l(\x_*, x_*)+\k^r(\x_*, \x_*)\right)- \left(\rho(\k^l_*)^T(\K^l)^{-1}\rho\k^l_*\right)
		+\left(\0, \rho(\k^r_*)^T(\K^r)^{-1}\right)\rho\k^l_*\\
		&-(\k^r_*)^T(\K^r)^{-1}\rho^2\k^l_*(\X^h)-(\k^r_*)^T(\K^r)^{-1}\k_*^r\\
		=& \rho^2\left(\k^l(\x_*, \x_*) - (\k^l_*)^T(\K^l)^{-1}\k^l_*\right)
		+\left(\k^r(\x_*, \x_*) - (\k^r_*)^T(\K^r)^{-1}\k^r_*\right)
	\end{aligned}
\end{equation}
where, $ \k_*^l(\X^h) = \k^l(\x_*, \X^h)$ is the covariance vector between the new inputs $\x_*$ and $\X^h$. Notice that the predictive posterior is also decomposed into two independent parts that related to the low-fidelity GP and the residual GP, which is convenient for parallel computing and saving computational resources.

\begin{lemma}\label{lemma2: K}
	Given tensor GP priors for $\tY^l(\x,\x')$ and $\tY^r(\x,\x')$ and the Tucker transformation of \Eqref{eq:tensor LAR}, the joint probability for $ \y=[\vecrm({\tY^l})^T,\vecrm(\tY^h)^T]^T $ is $\y \sim \Ncal(\0, \bSigma)$, where $\bSigma = $\\
  $ \left(\begin{array}{cc}
			 \K^l(\X^l, \X^l) \kron \left(\bigotimes_{m=1}^M\S^l_m\right) &
			 \K^l(\X^l,\X^h) \kron \left(\bigotimes_{m=1}^M\S^l_m\W_m^T\right)\\
			\K^l(\X^h,\X^l)\kron \left(\bigotimes_{m=1}^M\W_m\S^l_m\right) & \K^l(\X^h,\X^h)\kron \left(\bigotimes_{m=1}^M\W_m\S^l_m\W_m\right) + 
			\K^r(\X^h,\X^h )\kron\left(\bigotimes_{m=1}^M \S^r_m \right)
		\end{array}
		\right) $
\end{lemma}

\begin{proof}
Since the $ \bSigma $ is the covariance matrix of $ \y $, it can be expressed in block form as:
\begin{equation*}
	\bSigma = \left(\begin{array}{cc} \cov(\vecrm(\tY^l), \vecrm(\tY^l)) & \cov(\vecrm(\tY^l), \vecrm(\tY^h))\\ \cov(\vecrm(\tY^h), \vecrm(\tY^l)) & \cov(\vecrm(\tY^h), \vecrm(\tY^h))\end{array}\right),	
\end{equation*}
where $\cov(\vecrm(\tY^l), \vecrm(\tY^h)) = \cov(\vecrm(\tY^l), \vecrm(\tY^h))^T $ is the cross covariance between $\tY^l$ and $\tY^h$.
Assuming $ \tY^h \in \mathbb{R}^{N^h\times d^h_1\times ... \times d^h_M}$ and $ \tY^l \in \mathbb{R}^{N^l\times d^l_1\times ... \times d^l_M } $, together with the property of the Tucker operator in \Eqref{eq:tensor LAR}, the high-fidelity data and low-fidelity data have the following transformation,
\begin{equation}\label{Assumption of Y}
	\begin{aligned}
		{\tY^h} &= \tY^l \times_1 \E \times_2 \W_1 \times_3 ... \times_{M}\W_{M-1} \times_{M+1} \W_M \\
		\vecrm({\tY^h}) &= \left[\E \kron\left(\bigotimes_{m=1}^M\W_m\right)\right] \vecrm({\tY^l}) + \vecrm({\tY^r}),
	\end{aligned}
\end{equation}
where $ \forall i = 1, 2, ..., M, \W_i \in \mathbb{R}^{d^h_i\times d^l_i} $, and $ \E^T = \left(\0, \I_{N^h}\right) \in \mathbb{R}^{N^h\times N^l} $ is the selection matrix such that $\X^h=\E^T \X^l$.
By definition our GP prior, the low-fidelity data has the joint probability:
\[\vecrm({\tY^l}) \sim \Ncal \left(0,\K^l(\X^l,\X^l)\kron\left( \bigotimes_{m=1}^M \S^l_m \right) \right)\]
Thus the covariance matrix of low-fidelity data is $ \cov(\vecrm(\tY^l), \vecrm(\tY^l)) = \K^l(\X^l,\X^l)\kron\left( \bigotimes_{m=1}^M \S^l_m \right) $. After that, we can derive the other part of the $ \bSigma $. Firstly, assuming the residual information $ \vecrm(\tY^r) $ is independent from $ \vecrm(\tY^l) $, the covariance between $ \vecrm(\tY^h) $ and $ \vecrm(\tY^l) $ is 
\begin{equation}
	\begin{aligned}
		\cov(\vecrm(\tY^l), \vecrm(\tY^h))
		 =& \cov\left(\vecrm(\tY^l), \left[\E \kron\left(\bigotimes_{m=1}^M\W_m\right)\right] \vecrm({\tY^l}) + \vecrm({\tY^r})\right)\\
		=&\cov\left(\vecrm({\tY^l}), \vecrm(\tY^r)\right)+\cov\left(\vecrm(\tY^l), \left[\E \kron\left(\bigotimes_{m=1}^M\W_m\right)\right]\vecrm({\tY}^l)\right)\\
		=&\cov(\vecrm({\tY}^l), \vecrm({\tY^l}))\left[\E \kron\left(\bigotimes_{m=1}^M\W_m\right)\right]^T\\
		=&\left[\K^l(\X^l,\X^l)\kron\left( \bigotimes_{m=1}^M \S^l_m \right)\right]\left[\E^T \kron\left(\bigotimes_{m=1}^M\W_m^T\right)\right]\\
		=&\K^l(\X^l,\X^h)\kron\left(\bigotimes_{m=1}^M \S^l_m\W_m^T\right).\\
	\end{aligned}
\end{equation}
Since $ \cov(\vecrm(\tY^l), \vecrm(\tY^h)) $ is the transpose of $ \cov(\vecrm(\tY^h), \vecrm(\tY^l)) $, so the upper right part of $ \bSigma $ is 
\begin{equation}
	\begin{aligned}
		\cov(\vecrm(\tY^h), \vecrm(\tY^l))
		=\cov(\vecrm(\tY^l), \vecrm(\tY^h))^T
		=&\K^l(\X^h,\X^l)\kron\left(\bigotimes_{m=1}^M \W_m\S^l_m\right).\\
	\end{aligned}
\end{equation}
For the lower and right part of $ \bSigma $, the covariance between $ \cov(\vecrm(\tY^h), \vecrm(\tY^h)) $ is 
\begin{equation}
	\begin{aligned}
		&\cov(\vecrm(\tY^h), \vecrm(\tY^h))\\
		=& \cov\left(\left[\E \kron\left(\bigotimes_{m=1}^M\W_m\right)\right] \vecrm({\tY^l}) + \vecrm({\tY^r}), \left[\E \kron\left(\bigotimes_{m=1}^M\W_m\right)\right] \vecrm({\tY^l}) + \vecrm({\tY^r})\right)\\
		=& \cov\left(\left[\E \kron\left(\bigotimes_{m=1}^M\W_m\right)\right] \vecrm({\tY^l}), \left[\E \kron\left(\bigotimes_{m=1}^M\W_m\right)\right] \vecrm({\tY^l})\right) + \cov\left(\vecrm(\tY^r), \vecrm(\tY^r)\right)\\
		&+
		\cov\left(\left[\E \kron\left(\bigotimes_{m=1}^M\W_m\right)\right] \vecrm({\tY^l}) , \vecrm(\tY^r)\right)
		+\cov\left( \vecrm(\tY^r), \left[\E \kron\left(\bigotimes_{m=1}^M\W_m\right)\right] \vecrm({\tY^l})\right)\\
		=&\left[\E \kron\left(\bigotimes_{m=1}^M\W_m\right)\right]\left(\cov(\vecrm({\tY}^l), \vecrm({\tY^l}))\right)\left[\E \kron\left(\bigotimes_{m=1}^M\W_m\right)\right]^T+\cov\left( \vecrm({\tY}^r), \vecrm({\tY}^r)\right)\\
		=&\K^l(\X^h,\X^h)\kron\left(\bigotimes_{m=1}^M \W_m \S^l_m \W_m^T\right) + \K^r(\X^h,\X^h) \left(\bigotimes_{m=1}^M \S^r_m\right).
	\end{aligned}
\end{equation}
Assembling the several parts together, we have the joint covariance matrix $ \bSigma $:
\begin{equation}
	\begin{aligned}
		\left(\begin{array}{cc}
			\K^l(\X^l, \X^l) \kron \left(\bigotimes_{m=1}^M\S^l_m\right) &
			\K^l(\X^l,\X^h) \kron \left(\bigotimes_{m=1}^M\S^l_m\W_m^T\right)\\
			\K^l(\X^h,\X^l)\kron \left(\bigotimes_{m=1}^M\W_m\S^l_m\right) & \K^l(\X^h,\X^h)\kron \left(\bigotimes_{m=1}^M\W_m\S^l_m\W_m\right) + 
			\K^r(\X^h,\X^h )\kron\left(\bigotimes_{m=1}^M \S^r_m \right)
		\end{array}
		\right)
	\end{aligned}
\end{equation}
\end{proof}

Before we move on to the next proof, we introduce the matrix inversion property, which will become handy later.
\begin{Property}
	\label{prop1: block inv trick}
	For any invertible block matrixes $ \left(\begin{array}{cc}\A & \B \\ \B^T & \C\end{array}\right) $, where the sub-matrixes are also invertible, we have  $ (\B^T, \C)\left(\begin{array}{cc}\A & \B \\ \B^T & \C\end{array}\right)^{-1}=(\0,\I) $ and $ \left(\begin{array}{cc}\A & \B \\ \B^T & \C\end{array}\right)^{-1}\left(\begin{array}{c}\B \\ \C\end{array}\right)=\left(\begin{array}{c}\0 \\ \I \end{array}\right) $.
\end{Property}

\begin{proof}
	The inversion of a block matrix (if it is invertible) following the Sherman-Morrison formula is
\begin{equation}
	\begin{aligned}
		&\left(\begin{array}{cc}\A & \B \\ \B^T & \C\end{array}\right)^{-1}=\left(
		\begin{array}{cc}
			\bf P^{-1} &  -\bf P^{-1}\B\C^{-1}\\
			-\C^{-1}\B^T\bf P^{-1} & \C^{-1}+\C^{-1}\B^T\bf P^{-1}\B\C^{-1} 
		\end{array}
		\right),a
	\end{aligned}
\end{equation}
where $\P=\A-\B\C^{-1}\B^T$, we can then derive the multiplication in Property \autoref{prop1: block inv trick} by the rule of block matrix multiplication:
\begin{equation}
	\begin{aligned}
		&(\B^T, \C)\left(\begin{array}{cc}\A & \B \\ \B^T & \C\end{array}\right)^{-1}\\
		=&(\B^T, \C)\left(
		\begin{array}{cc}
			\bf P^{-1} &  -\bf P^{-1}\B\C^{-1}\\
			-\C^{-1}\B^T\bf P^{-1} & \C^{-1}+\C^{-1}\B^T\bf P^{-1}\B\C^{-1} 
		\end{array}
		\right)\\
		=&\left(\B^T\bf P^{-1}-\C(\C^{-1}\B^T\bf P^{-1}),-\B^T\bf P^{-1}\B\C^{-1}+\C\C^{-1}+\C(\C^{-1}\B^T\bf P^{-1}\B\C^{-1})\right)\\
		=&\left(\B^T\bf P^{-1}-\B^T\bf P^{-1},-\B^T\bf P^{-1}\B\C^{-1}+\I+\B^T\bf P^{-1}\B\C^{-1}\right)\\
		=&\left(\0,\I \right).
	\end{aligned}
\end{equation}
Similarly, the other part of the conclusion can also be derived
\begin{equation}
\begin{aligned}
		&\left(\begin{array}{cc}\A & \B \\ \B^T & \C\end{array}\right)^{-1}\left(\begin{array}{c}\B \\ \C\end{array}\right)\\
		=&\left(
		\begin{array}{cc}
			\bf P^{-1} &  -\bf P^{-1}\B\C^{-1}\\
			-\C^{-1}\B^T\bf P^{-1} & \C^{-1}+\C^{-1}\B^T\bf P^{-1}\B\C^{-1} 
		\end{array}
		\right)\left(\begin{array}{c}\B\\ \C \end{array}\right)\\
		=&\left(\begin{array}{c}\bf P^{-1}\B-\bf P^{-1}\B\C^{-1}\C \\ -\C^{-1}\B^T\bf P^{-1}\B+(\C^{-1}+\C^{-1}\B^T\bf P^{-1}\B\C^{-1})\C\end{array}\right)\\
		=&\left(\begin{array}{c}\bf P^{-1}\B-\bf P^{-1}\B \\ -\C^{-1}\B^T\bf P^{-1}\B+ \I+\C^{-1}\B^T\bf P^{-1}\B\end{array}\right)\\
		=&\left(\begin{array}{c}\0 \\ \I \end{array}\right),
	\end{aligned}
\end{equation}
\end{proof}
which seems quite obvious and intuitive if we assume that the matrix is symmetric.

\begin{lemma}
	Generalization of Lemma \ref{lemma1} in \ours.
	If $\X^h \subset \X^l$, the joint likelihood $\Lcal$ for $\y=[\vecrm({\tY^l})^T,\vecrm(\tY^h)^T]^T$ admits two independent separable likelihoods $\Lcal = \Lcal^l + \Lcal^r$, where
	\[
	  \Lcal^l = -\frac{1}{2}\vec{\tY^l}^T (\K^l \kron \S^l)^{-1} \vec{\tY^l} - \frac{1}{2} \log|\K^l \kron \S^l| - \frac{N^l D^l}{2} \log(2\pi),
	  \]
	\[ 
	  \Lcal^r = -\frac{1}{2} \vec{\tY^h- \tY^l \times \hat{\tW}}^T ( \K^r \kron \S^r )^{-1} \vec{\tY^h- \tY^l \times \hat{\tW}} - \frac{1}{2} \log|\K^r \kron \S^r| - \frac{N^h D^h}{2}  \log(2\pi),
	  \]
	  where $\hat{\tW} = [\E,\tW]$ is the original weight tensor concatenated with an selection matrix $\X^h=\E^T \X^l$. 
	\end{lemma}

\textit{Proof.}
Let the kernel matrix be partitioned into four blocks. We again make use of the matrix inversion of Sherman-Morrison formula in Property \ref{prop1: block inv trick}, with a slight modification as follows:
\[
	\left(\begin{array}{cc}
		\T & \U \\
		\V & \M
	\end{array}\right)^{-1}
	= \left(\begin{array}{cc}
		\T^{-1}+\T^{-1}\U\Q^{-1}\V\T^{-1} &  -\T^{-1}\U\Q^{-1}\\
		-\Q^{-1}\V\T^{-1} & \Q^{-1}\end{array}\right)
\]
where \[
	\Q =\M-\V \T^{-1}\U.\]
We begin this proof with the matrix $ \V \T^{-1}\U $ within Property \ref{prop1: block inv trick}, which gives us:
$$ \K^l(\X^h,\X^l){\K^l}(\X^l, \X^l)^{-1}\K^l(\X^l,\X^h) = \left(\0,\I\right)\K^l(\X^l,\X^h)= \K^l(\X^h,\X^h).$$
Therefore, the last part $\V \T^{-1}\U$ of matrix $ \Q $ is
\begin{equation}\label{A.10}
	\begin{aligned}
		&\V \T^{-1}\U \\
		=& \left[\K^l(\X^h,\X^l)\kron \W\S^l \right] 
		\left[{\K^l}(\X^l, \X^l)\kron \S^l\right]^{-1}
		\left[{\K^l(\X^l,\X^h)}\kron\S^l\W^T\right]\\
		=& \left[\K^l(\X^h,\X^l)\kron \W\S^l \right] 
		\left[{\K^l}(\X^l, \X^l)^{-1}\kron (\S^l)^{-1}\right]
		\left[{\K^l(\X^l,\X^h)}\kron\S^l\W^T\right]\\
		=& \left[\K^l(\X^h,\X^l){\K^l}(\X^l, \X^l)^{-1}\K^l(\X^l,\X^h)\right]\kron\left[(\bigotimes_{m=1}^M \W_m\S^l_m)(\bigotimes_{m=1}^M(\S^l_m)^{-1})(\bigotimes_{m=1}^M \S^l_m\W_m^T)\right]\\
		=& \left[\K^l(\X^h,\X^l){\K^l}(\X^l, \X^l)^{-1}\K^l(\X^l,\X^h)\right]\kron\left[\bigotimes_{m=1}^M(\W_m\S^l_m(\S^l_m)^{-1}\S^l_m\W_m^T)\right]\\
		=&\K^l(\X^h,\X^h)\kron\left[\bigotimes_{m=1}^M(\W_m\S^l_m\W_m^T)\right].
	\end{aligned}
\end{equation}
Substituting \Eqref{A.10} back into the matrix inversion, we can derive matrix $ \Q^{-1} $, $ -\T^{-1}\U\Q^{-1} $, $ \Q^{-1}\V\T^{-1} $, and $ \T^{-1}+\T^{-1}\U\Q^{-1}\V\T^{-1} $ as
\begin{equation}\label{Q^{-1}}
	\begin{aligned}
		&\Q^{-1} \\
		=& (\M-\V \T^{-1}\U)^{-1}\\
		=&\left(\K^l(\X^h,\X^h)\kron\W\S^l\W^T  +  \K^r(\X^h,\X^h )\kron\S^r - \K^l(\X^h,\X^h)\kron\W\S^l\W^T\right)^{-1}\\
		=&\left( \K^r(\X^h,\X^h )\kron\S^r\right)^{-1},
	\end{aligned}
\end{equation}
\begin{equation}
	\begin{aligned}
		&-\T^{-1}\U\Q^{-1}\\
		=& -\left[{\K^l}(\X^l, \X^l)^{-1}\kron(\S^l)^{-1}\right]
		\left[{\K^l(\X^l,\X^h)}\kron\S^l\W^T\right]
		\left[\K^r(\X^h,\X^h )^{-1}\kron(\S^r)^{-1}\right]\\
		=& -\left[{\K^l}(\X^l, \X^l)^{-1}\K^l(\X^l,\X^h)\K^r(\X^h,\X^h )^{-1}\right]\kron\left[(\S^l)^{-1}\S^l\W^T(\S^r)^{-1}\right]\\
		=& -\left[\left(\begin{array}{c}\0 \\ \I \end{array}\right)\K^r(\X^h,\X^h )^{-1}\right]\left[\W^T(\S^r)^{-1}\right]\\
		=&  -\left(\begin{array}{c}\0 \\ 
			\K^r(\X^h,\X^h )^{-1}\kron\W^T(\S^r)^{-1} \end{array}\right),
	\end{aligned}
\end{equation}
\begin{equation}
	\begin{aligned}
		&\Q^{-1}\V\T^{-1}\\
		=& -\left[\K^r(\X^h,\X^h )^{-1}(\S^r)^{-1}\right]
		\left[\K^l(\X^h,\X^l)\kron\W\S^l\right]
		\left[{\K^l}(\X^l, \X^l)^{-1}\kron(\S^l)^{-1}\right]\\
		=& -\left[\K^r(\X^h,\X^h )^{-1}\K^l(\X^h,\X^l){\K^l}(\X^l, \X^l)^{-1}\right]\kron\left[(\S^r)^{-1}\W\S^l(\S^l)^{-1}\right]\\
		=& -\left(\0,\quad \K^r(\X^h,\X^h )^{-1}\kron(\S^r)^{-1}\W\right),
	\end{aligned}
\end{equation}
and
\begin{equation}
	\begin{aligned}
		&\T^{-1}+\T^{-1}\U\Q^{-1}\V\T^{-1}\\
		=& \left[{\K^l}(\X^l, \X^l)^{-1}\kron(\S^l)^{-1}\right]
		+\left[{\K^l}(\X^l, \X^l)^{-1}\kron(\S^l)^{-1}\right]
		\left[{\K^l(\X^l,\X^h)}\kron\S^l\W^T\right]\\
		&\times\left[\K^r(\X^h,\X^h )^{-1}\kron(\S^r)^{-1}\right]
		\left[\K^l(\X^h,\X^l)\kron\W\S^l\right]
		\left[{\K^l}(\X^l, \X^l)^{-1}\kron(\S^l)^{-1}\right]\\
		=& \left[{\K^l}(\X^l, \X^l)^{-1}\kron(\S^l)^{-1}\right]+
		\left[{\K^l}(\X^l, \X^l)^{-1}{\K^l(\X^l,\X^h)}\K^r(\X^h,\X^h )^{-1}\K^l(\X^h,\X^l){\K^l}(\X^l, \X^l)^{-1}\right]\\
		&\kron\left[(\S^l)^{-1}\S^l(\S^r)^{-1}\S^l(\S^l)^{-1}\right]\\
		=&{\K^l}(\X^l, \X^l)^{-1}\kron (\S^l)^{-1}
		+\left(\begin{array}{cc}\0&\0\\ \0 & \K^r(\X^h,\X^h)^{-1}\kron\W^T\S^r\W\end{array}\right).
		\end{aligned}
\end{equation}
Putting all these elements together, we get the inversion of joint kernel matrix $\bSigma^{-1}=$
\begin{equation}\label{A.15}
	\left[
	\begin{array}{cc}
		(\K^l)^{-1}\kron{(\S^l)}^{-1} + \left(\begin{array}{cc}\0 &\0 \\ \0 & (\K^r)^{-1}\kron{\W}^T({\S}^r)^{-1}{\W}\end{array}\right) & -\left(\begin{array} {c} \0 \\ (\K^r)^{-1}\kron{\W}^T({\S}^r)^{-1}\end{array} \right) \\
		-\left(\0, (\K^r)^{-1}\kron({\S}^r)^{-1}\W\right) & (\K^r)^{-1}\kron({\S}^r)^{-1}
	\end{array}\right],
\end{equation}
where $ \S^l = \bigotimes_{m=1}^M \S^l_m $, $ \S^r = \bigotimes_{m=1}^M \S_m^r $, $ \W = \bigotimes_{m=1}^M \W_m $, $ \K^l = \K^l(\X^l, \X^l) $, and $ \K^r = \K^r(\X^h, \X^h) $ as defined in the main paper.
With the property in \Eqref{Assumption of Y}, and defining $\y=[\vecrm({\tY^l})^T,\vecrm(\tY^h)^T]^T $, we can substitute \Eqref{A.15} into the joint likelihood to derive the data fitting part of the joint likelihood
\begin{equation}\label{A.16}
	\begin{aligned}
		&\y^T \bSigma^{-1} \y\\
		=& \left(\vecrm({\tY^l})^T, \vecrm({\tY^l})^T(\E\kron\W^T)+\vecrm(\tY^r)^T\right)
		\bSigma^{-1}
		\left(\begin{array}{c}
			\vecrm({\tY^l}) \\ 
			(\E^T \kron \W) \vecrm({\tY^l}) + \vecrm({\tY^r}) 
		\end{array} \right)\\
		=& \vecrm({\tY^l})^T \left( {\K^l} \kron \S^l\right)^{-1} \vecrm({\tY^l})
		+\vecrm({\tY^l})^T \E \left( (\K^r)^{-1} \kron \W^T(\S^r)^{-1}\W\right)\E^T\vecrm({\tY^l})\\
		&- \vecrm(\tY^l)^T \left(\E \kron \W^T \right) \left( (\K^r)^{-1} \E^T \kron (\S^r)^{-1}\W \right) \vecrm({\tY^l})\\
		&- \vecrm(\tY^r)^T \left( (\K^r)^{-1} \E^T \kron (\S^r)^{-1}\W \right) \vecrm({\tY^l})\\
		&- \vecrm({\tY^l})^T\left( \E(\K^r)^{-1} \kron \W^T (\S^r)^{-1} \right) 
		\left[\left( \E^T \kron \W \right) \vecrm(\tY^l) + \vecrm(\tY^r) \right] \\
		&+ \vecrm({\tY^l})^T \left(\E \kron \W^T \right) \left( \K^r \kron \S^r \right)^{-1} 
		\left[\left( \E^T \kron \W \right) \vecrm(\tY^l) + \vecrm(\tY^r) \right] \\
		&+ \vecrm({\tY^r})^T \left( \K^r \kron \S^r \right)^{-1} 
		\left[\left( \E^T \kron \W \right) \vecrm(\tY^l) + \vecrm(\tY^r) \right] \\
		=& \vecrm({\tY^l})^T \left( {\K^l} \kron \S^l\right)^{-1} \vecrm({\tY^l}) + \vecrm({\tY^r})^T \left( \K^r \kron \S^r \right)^{-1} \vecrm(\tY^r) \\
		& - \vecrm(\tY^r)^T \left( (\K^r)^{-1} \E^T \kron (\S^r)^{-1}\W \right) \vecrm({\tY^l})
		 + \vecrm({\tY^r})^T \left( \K^r \kron \S^r \right)^{-1} 
		\left( \E^T \kron \W \right) \vecrm(\tY^l) \\
		=& \vecrm({\tY^l})^T \left( {\K^l} \kron \S^l\right)^{-1} \vecrm({\tY^l}) + \vecrm({\tY^r})^T \left( \K^r \kron \S^r \right)^{-1} \vecrm(\tY^r).
	\end{aligned}
\end{equation}
With the block matrix's determinant formula, we can also derive the determinant of joint kernel matrix,
\begin{equation}
	\begin{aligned}
		\left| \bSigma \right| = & \left| \K^l \kron \S^l \right| \times \left| \Q \right|
		= \left| \K^l \kron \S^l \right| \times \left| \K^r \kron \S^r \right|.
	\end{aligned}
\end{equation}
where we do not decompose them further with the purpose of forming to independent GPs for the low- and high-fidelity.
With the conclusion of \Eqref{A.16}, the full joint log-likelihood is
\begin{equation}
	\begin{aligned}
		&\log p(\tY^l,\tY^h) \\
		=& -\frac{1}{2}\y^T \bSigma^{-1}\y - \frac{1}{2} \log|\bSigma| - \frac{d^lN^l+d^hN^h}{2} \log(2\pi)\\
		=&\underbrace{-\frac{1}{2}\vecrm({\tY^l})^T\left(\K^l \kron \S^l\right)^{-1}\vecrm({\tY^l})-\frac{1}{2}\log\left| \K^l \kron \S^l \right|- \frac{N^ld^l}{2} \log(2\pi)}_{TGP\ for\ low-fidelity\ data} \\
		&\underbrace{-\frac{1}{2}\vecrm({\tY^r})^T\left(\K^r \kron \S^r\right)^{-1}\vecrm({\tY^r}) -\frac{1}{2}\log\left|\K^r \kron \S^r\right| - \frac{N^hd^h}{2} \log(2\pi)}_{TGP\ for\ residual\ information}\\
		=& \log p(\tY^l) + \log p(\tY^h |\tY^l)
	\end{aligned}
\end{equation}
The meaning of $ \S^l $, $\S^r$, $\W $, and $\K^r $ remain the same as defined in Eq.\ref{A.15}.

\subsection{Posterior distribution}

For the posterior distribution we compute the mean function and covariance matrix with the assumption that the low- and high-fidelity data have a very strict subset requirement, $ \X^h \subseteq \X^l $. With the conclusion in Lemma \ref{lemma2: K} and rule of block matrix multiplication, the mean function and covariance matrix have the following expression,
\begin{equation}
	\begin{aligned}
		&\vecrm(\tZ^h_*)\\
		&=\left(\k^l_* \kron \W\S^l, \k^l_*(\X^h) \kron \W\S^l\W^T + \k^r_* \kron \S^r \right)
		\bSigma^{-1}\left(\begin{array}{c}\vecrm({\tY^l}) \\ \vecrm({\tY^h})\end{array}\right)\\
		&=\left(\k^l_* \kron \W\S^l\right)\left(\K^l\kron\S^l\right)^{-1}\vecrm({\tY^l})
		 + \left(\k^l_* \kron \W\S^l\right)\left( \E(\K^r)^{-1}\E^T \kron \W^T(\S^r)^{-1}\W \right) \vecrm({\tY^l}) \\
		&\quad-\left(\k^l_*(\X^h) \kron \W\S^l\W^T\right)\left( \K^r\E^T \kron (\S^r)^{-1}\W \right)  \vecrm({\tY^l})\\
		&\quad - \left(\k^r_* \kron \S^r\right)\left( \K^r\E^T \kron (\S^r)^{-1}\W \right)  \vecrm({\tY^l})\\
		&\quad - \left(\k^l_* \kron \W\S^l\right) \left(\E(\K^r)^{-1}\kron\W^T(\S^r)^{-1}\right) \vecrm({\tY^h}) \\
		&\quad + \left(\k^l_*(\X^h) \kron \W\S^l\W^T\right)\left( \K^r \kron (\S^r)^{-1} \right)  \vecrm({\tY^h})
		 + \left(\k^r_* \kron \S^r\right)\left( \K^r \kron (\S^r)^{-1} \right)  \vecrm({\tY^h})\\
		&= \left(\k^l_* \kron \W\S^l\right)\left(\K^l\kron\S^l\right)^{-1}\vecrm({\tY^l}) - \left(\k^r_* \kron \S^r\right)\left( \K^r\E^T \kron (\S^r)^{-1}\W \right)  \vecrm({\tY^l})\\
		&\quad + \left(\k^r_* \kron \S^r\right)\left( \K^r \kron (\S^r)^{-1} \right) \left(\E^T \kron \W \right)\vecrm({\tY^l})
		 + \left(\k^r_* \kron \S^r\right)\left( \K^r \kron (\S^r)^{-1} \right)  \vecrm({\tY^r})\\
		&=\left(\k^l_*(\K^l)^{-1} \kron \W\right) \vecrm({\tY^l}) +  \left(\k^r_*(\K^r)^{-1} \kron \I_r\right) \vecrm({\tY^r}),
	\end{aligned}
\end{equation}
\begin{equation}
	\begin{aligned}
		&\S_*^h
		=\left( \k^l(\x_*, \x_*) \kron \W\S^l\W^T + \k^r(\x_*, \x_*) \kron \S^r \right) - \\
		&\left(\k^l_* \kron \W\S^l, \k^l_*(\X^h) \kron \W\S^l\W^T + \k^r_* \kron \S^r \right)
		\bSigma^{-1}
		\left(\begin{array}{c}
			(\k^l_*)^T \kron \S^l\W^T \\ 
			\k^l_*(\X^h)^T \kron \W^T \S^l \W + (\k^r_*)^T \kron \S^r
		\end{array}\right)\\
		&=\left( \k^l(\x_*, \x_*) \kron \W\S^l\W^T + \k^r(\x_*, \x_*) \kron \S^r \right) \\
		&\quad-\left(\k^l_* \kron \W\S^l\right)\left(\K^l\kron\S^l\right)^{-1}\left((\k^l_*)^T \kron \S^l\W^T\right) \\
		&\quad-\left(\k^l_* \kron\W\S^l\right) \left(\E(\K^l)^{-1}\E^T\kron(\S^l)^{-1}\right)\left((\k^l_*)^T \kron \S^l\W^T\right)\\
		&\quad+\left(\k^l_*(\X^h) \kron \W\S^l\W^T\right)\left( \K^r\E^T \kron (\S^r)^{-1}\W \right) \left((\k^l_*)^T \kron \S^l\W^T\right)\\
		&\quad+\left(\k^r_* \kron \S^r\right)\left( \K^r\E^T \kron (\S^r)^{-1}\W \right)\left((\k^l_*)^T \kron \S^l\W^T\right)\\
		&\quad+\left(\k^l_* \kron \W\S^l\right) \left(\E(\K^r)^{-1}\kron\W^T(\S^r)^{-1}\right) \left(\k^l_*(\X^h)^T \kron \W^T \S^l \W + (\k^r_*)^T \kron \S^r\right)\\
		&\quad- \left(\k^l_*(\X^h) \kron \W\S^l\W^T\right)\left( \K^r \kron (\S^r)^{-1} \right) \left(\k^l_*(\X^h)^T \kron \W^T \S^l \W + (\k^r_*)^T \kron \S^r\right)\\
		&\quad-\left(\k^r_* \kron \S^r\right)\left( \K^r \kron (\S^r)^{-1} \right)\left(\k^l_*(\X^h)^T \kron \W^T \S^l \W + (\k^r_*)^T \kron \S^r\right)\\
		&= \left( \k^l(\x_*, \x_*) \kron \W\S^l\W^T + \k^r(\x_*, \x_*) \kron \S^r \right) \\
		&\quad-\left(\k^l_* \kron \W\S^l\right)\left(\K^l\kron\S^l\right)^{-1}\left((\k^l_*)^T \kron \S^l\W^T\right) \\
		&\quad+\left(\k^r_* \kron \S^r\right)\left( \K^r\E^T \kron (\S^r)^{-1}\W \right)\left((\k^l_*)^T \kron \S^l\W^T\right)\\
		&\quad-\left(\k^r_* \kron \S^r\right)\left( \K^r \kron (\S^r)^{-1} \right)\left(\k^l_*(\X^h)^T \kron \W^T \S^l \W\right)\\
		&\quad-\left(\k^r_* \kron \S^r\right)\left( \K^r \kron (\S^r)^{-1} \right)\left((\k^r_*)^T \kron \S^r\right)\\
		&= \left(k^l_{**} - (\k^l_*)^T(\K^l)^{-1}\k^l_*\right) \kron \W\S^l\W^T +\left(k^r_{**} - (\k^r_*)^T(\K^r)^{-1}\k^r_*\right) \kron \S^r,
	\end{aligned}
\end{equation}
where the $ {\W}$, ${\S^l}$ and $ {\S}^r $ have the same meaning with the main paper.

\subsection{Joint Likelihood for Non-Subset Multi-Fidelity Data}
In the main paper and the subset section, we decompose the joint likelihood $ \log p(\tY^h, \tY^l) $ into two parts as 
\begin{equation*}
	\begin{aligned}
		&\log  p(\tY^l,\tY^h)
		=\log p(\tY^l) + \log \int p({\tY}^h | \hat{\tY}^l, \tY^l) p(\hat{\tY}^l | \tY^l ) d\hat{\tY}^l
	\end{aligned}
\end{equation*}
where $p({\tY}^h | \hat{\tY}^l, \tY^l)$ is the derived predictive posterior probability if the high-fidelity data are subset to the low-fidelity data.
\begin{equation}\label{Yh}
	\begin{aligned}
		&p({\tY}^h | \hat{\tY}^l, \tY^l) = 2\pi^{-\frac{N^hd^h}{2}}\times \left|{\K}^r\kron\S^r\right|^{-\frac{1}{2}}\\
		&\times\exp\left[-\frac{1}{2}\left[\left(\begin{array}{c}
			\vecrm({\check{\tY}^h})\\
			\vecrm({\hat{\tY}^h})
		\end{array}\right) - {\hat{\W}}
	\left(\begin{array}{c}
			\vecrm({\tY}^l)\\
			 {\vecrm({\hat{\tY}^l})}\end{array}\right)\right]^T({\K}^r\kron\S^r)^{-1}\left[\left(\begin{array}{c}
		\vecrm({\check{\tY}^h})\\
		\vecrm({\hat{\tY}^h})
		\end{array}\right) - {\hat{\W}}\left(\begin{array}{c}
			\vecrm({\tY}^l)\\
			 {\vecrm({\hat{\tY}^l})}\end{array}\right)\right] \right]
			\end{aligned}
\end{equation}
where we define $ \hat{\W} = \E \kron \bigotimes_{m=1}^M \W_m $. 
Based on the low-fidelity training data, we also have 
 $ p(\hat{\tY}^l|\tY^l) \sim \N(\bar{\tY}^l,\hat{\S}^l\kron\S^l) $ being a Gaussian.
\begin{equation}\label{hatY}
	\begin{aligned}
		&p(\hat{\tY}^l| \tY^l) \\
		=& 2\pi^{-\frac{N^md^l}{2}} \times \left|\hat{\S}^l\kron \S^l\right|^{-\frac{1}{2}} \times \exp\left[-\frac{1}{2}
		\left( {\vecrm({\hat{\tY}^l})}-\vecrm(\bar{\tY}^l)\right)^T
		(\hat{\S}^l\kron\S^l)^{-1}
		\left( {\vecrm({\hat{\tY}^l})}-\vecrm(\bar{\tY}^l)\right)\right],
	\end{aligned}
\end{equation}
where the $ \hat{\S}^l\kron\S^l $ is the posterior covariance matrix of the $ \hat{\tY}^l $. We can combine \Eqref{Yh} and \Eqref{hatY} to derive the integral part of the joint likelihood
\begin{equation}\label{main part of non-subset}
	\begin{aligned}
		&\log \int p({\tY}^h | \hat{\tY}^l, \tY^l) p(\hat{\tY}^l | \tY^l ) d\hat{\tY}^l\\
		=&-\frac{N^hd^h+N^md^l}{2}\log(2\pi)-\frac{1}{2}\log\left|\K^r\kron\S^r\right|-\frac{1}{2}\log\left|\hat{\S}^l\kron\S^l\right|\\
		&+\log \int \exp\left\lbrace-\frac{1}{2}\left[\left(\begin{array}{c}
			\vecrm({\check{\tY}^h})\\
			\vecrm(\hat{\tY}^h)
		\end{array}\right) - \left(\E_n^T \kron {\W}\right)\left(\begin{array}{c}
			\vecrm(\check{\tY}^l)\\
			 {\vecrm({\hat{\tY}^l})}\end{array}\right)\right]^T(\K^r\kron\S^r)^{-1} \right.\\
		& \quad \left[\left(\begin{array}{c}
			\vecrm({\check{\tY}^h})\\
			\vecrm(\hat{\tY}^h)
		\end{array}\right) - \left(\E_n^T \kron {\W}\right)\left(\begin{array}{c}
			\vecrm(\check{\tY}^l)\\
			 {\vecrm({\hat{\tY}^l})}\end{array}\right)\right]\\
		& \left. -\frac{1}{2}\left( {\vecrm({\hat{\tY}^l})^T}-\vecrm(\bar{\tY}^l)^T\right)(\hat{\S}^l\kron\S^l)^{-1}\left( {\vecrm({\hat{\tY}^l})}-\vecrm(\bar{\tY}^l)\right) \right\rbrace d \vecrm(\hat{\tY}^l),
	\end{aligned}
\end{equation}
where the $ \X^h = \E_n^T[\X^l, \hat{\X}^h] $.
Since we know that $ \hat{\X}^h = \hat{\E}^T\X^h $ and we assume that $ \check{\X}^h = \check{\E}^T\X^h $, we can derive 
\begin{equation}
	\label{A32}
	\begin{aligned}
		&\left(\begin{array}{c}
			\vecrm({\check{\tY}^h})\\
			\vecrm(\hat{\tY}^h)
		\end{array}\right) - (\E_n^T\kron{\W})\left(\begin{array}{c}
			\vecrm({\tY^l})\\
			 {\vecrm({\hat{\tY}^l})}\end{array}\right)\\
		&= \left(\begin{array}{c}
			\vecrm({\check{\tY}^h})\\
			\0
		\end{array}\right) - \tilde{\W}\left(\begin{array}{c}
			\vecrm(\check{\tY}^l)\\
			\0\end{array}\right)+\left(\begin{array}{c}
			\0\\
			\vecrm(\hat{\tY}^h)
		\end{array}\right) - \hat{\W}\left(\begin{array}{c}
			\0\\
			 {\vecrm({\hat{\tY}^l})}\end{array}\right)\\
		&= \left(\check{\E} \kron \I_h\right) \vecrm(\check{\tY}^h) 
		- \left(\check{\E} \kron \W \right) \vecrm(\check{\tY}^l)+ \left(\hat{\E} \kron \I_h\right) \vecrm(\check{\tY}^h) - \left(\hat{\E} \kron \W \right) \vecrm(\check{\tY}^l).
	\end{aligned}
\end{equation}
in which the $ \tilde{\W} = \I_{N^h} \kron \W $, and $ \check{\tY}^l $ denotes the corresponding part of the $ \check{\tY}^h $.
For convenience, we choose to compute the exponential part in \Eqref{main part of non-subset} as our first step. We try to decompose it into the subset part, \ie ($ \check{\tY}^h $ and $ \check{\tY}^l $) and the non-subset part, \ie ($ \hat{\tY}^h $ and $ \hat{\tY}^l $).
\Eqref{A32} will become handy for the later derivation.
Let's first consider the data fitting part by substituting \Eqref{A32} into \Eqref{main part of non-subset}, 

\begin{equation}
	\label{A33}
	\begin{aligned}
		&-\frac{1}{2}\left[\left(\begin{array}{c}
			\vecrm({\check{\tY}^h})\\
			\vecrm(\hat{\tY}^h)
		\end{array}\right) - \hat{\W}\left(\begin{array}{c}
			\vecrm(\check{\tY}^l)\\
			 {\vecrm({\hat{\tY}^l})}\end{array}\right)\right]^T(\K^r\kron\S^r)^{-1}\left[\left(\begin{array}{c}
			\vecrm({\check{\tY}^h})\\
			\vecrm(\hat{\tY}^h)
		\end{array}\right) - \hat{\W}\left(\begin{array}{c}
			\vecrm(\check{\tY}^l)\\
			 {\vecrm({\hat{\tY}^l})}\end{array}\right)\right]\\
		&=
		-\frac{1}{2}\left[ \vecrm(\check{\tY}^h)^T\left(\check{\E}^T \kron \I_h\right)  - \vecrm(\check{\tY}^l)^T\left(\check{\E}^T \kron \W^T\right)\right]
		(\K^r\kron\S^r)^{-1}
		\left[\left(\check{\E} \kron \I_h\right)\vecrm(\check{\tY}^h) - \left(\check{\E} \kron \W\right)\vecrm(\check{\tY}^h)\right]\\
		&-\left[ \vecrm(\check{\tY}^h)^T\left(\check{\E}^T \kron \I_h\right)  - \vecrm(\check{\tY}^l)^T\left(\check{\E}^T \kron \W^T\right)\right]
		(\K^r\kron\S^r)^{-1}
		\left[\left(\hat{\E} \kron \I_h\right)\vecrm(\hat{\tY}^h) - \left(\hat{\E} \kron \W\right)\vecrm(\hat{\tY}^l)\right]\\
		&-\frac{1}{2}\left[\vecrm(\hat{\tY}^h)^T\left(\hat{\E}^T \kron \I_h\right) - \vecrm(\hat{\tY}^l)^T\left(\hat{\E}^T \kron \W\right)\right]
		(\K^r\kron\S^r)^{-1}
		\left[\left(\hat{\E} \kron \I_h\right)\vecrm(\hat{\tY}^h) - \left(\hat{\E} \kron \W\right)\vecrm(\hat{\tY}^l)\right],
	\end{aligned}
\end{equation}
which gives us the decomposition as the subset part, the non-subset part, and the interaction part between them.
Now we can substitute \Eqref{A33} into the integral part in \Eqref{main part of non-subset},
\begin{equation}\label{log of non-subset}
	\begin{aligned}
		\log& \int \exp[-\frac{1}{2}\left[ \vecrm(\check{\tY}^h)^T\left(\check{\E}^T \kron \I_h\right)  - \vecrm(\check{\tY}^l)^T\left(\check{\E}^T \kron \W^T\right)\right]
		(\K^r\kron\S^r)^{-1}
		\left[\left(\check{\E} \kron \I_h\right)\vecrm(\check{\tY}^h) - \left(\check{\E} \kron \W\right)\vecrm(\check{\tY}^h)\right]\\
		&-\left[ \vecrm(\check{\tY}^h)^T\left(\check{\E}^T \kron \I_h\right)  - \vecrm(\check{\tY}^l)^T\left(\check{\E}^T \kron \W^T\right)\right]
		(\K^r\kron\S^r)^{-1}
		\left[\left(\hat{\E} \kron \I_h\right)\vecrm(\hat{\tY}^h) - \left(\hat{\E} \kron \W\right)\vecrm(\hat{\tY}^l)\right]\\
		&-\frac{1}{2}\left[\vecrm(\hat{\tY}^h)^T\left(\hat{\E}^T \kron \I_h\right) - \vecrm(\hat{\tY}^l)^T\left(\hat{\E}^T \kron \W\right)\right]
		(\K^r\kron\S^r)^{-1}
		\left[\left(\hat{\E} \kron \I_h\right)\vecrm(\hat{\tY}^h) - \left(\hat{\E} \kron \W\right)\vecrm(\hat{\tY}^l)\right]\\
		&-\frac{1}{2} {\vecrm(\hat{\tY}^l)^T}(\hat{\S}^l\kron\S^l)^{-1} {\vecrm(\hat{\tY}^l)}
		-\frac{1}{2}\vecrm(\bar{\tY}^l)^T(\hat{\S}^l\kron\S^l)^{-1}\vecrm(\bar{\tY}^l)
		+ {\vecrm(\hat{\tY}^l)^T}(\hat{\S}^l\kron\S^l)^{-1}\vecrm(\bar{\tY}^l) ]d \vecrm(\hat{\tY}^l)\\
		=&-\frac{1}{2}\left[ \vecrm(\check{\tY}^h)^T\left(\check{\E}^T \kron \I_h\right)  - \vecrm(\check{\tY}^l)^T\left(\check{\E}^T \kron \W^T\right)\right]
		(\K^r\kron\S^r)^{-1}
		\left[\left(\check{\E} \kron \I_h\right)\vecrm(\check{\tY}^h) - \left(\check{\E} \kron \W\right)\vecrm(\check{\tY}^h)\right]\\
		&-\left[ \vecrm(\check{\tY}^h)^T\left(\check{\E}^T \kron \I_h\right)  - \vecrm(\check{\tY}^l)^T\left(\check{\E}^T \kron \W^T\right)\right]
		(\K^r\kron\S^r)^{-1}
		\left(\hat{\E} \kron \I_h\right)\vecrm(\hat{\tY}^h)\\
		&-\frac{1}{2}\vecrm(\hat{\tY}^h)^T\left(\hat{\E}^T \kron \I_h\right)
		(\K^r\kron\S^r)^{-1}
		\left(\hat{\E} \kron \I_h\right)\vecrm(\hat{\tY}^h)
		-\frac{1}{2}\vecrm(\bar{\tY}^l)^T(\hat{\S}^l\kron\S^l)^{-1}\vecrm(\bar{\tY}^l)\\
		&+\log \int \exp[ \left[ \vecrm(\check{\tY}^h)^T\left(\check{\E}^T \kron \I_h\right)  - \vecrm(\check{\tY}^l)^T\left(\check{\E}^T \kron \W^T\right)\right]
		(\K^r\kron\S^r)^{-1}
		\left(\hat{\E} \kron \W\right)\vecrm(\hat{\tY}^l)\\
		&+\vecrm(\hat{\tY}^h)^T\left(\hat{\E}^T \kron \I_h\right)
		(\K^r\kron\S^r)^{-1}\left(\hat{\E} \kron \W\right)\vecrm(\hat{\tY}^l)\\
		&-\frac{1}{2}\vecrm(\hat{\tY}^l)^T\left(\hat{\E}^T \kron \W^T\right)
		(\K^r\kron\S^r)^{-1}\left(\hat{\E} \kron \W\right)\vecrm(\hat{\tY}^l)\\
		&-\frac{1}{2} {\vecrm(\hat{\tY}^l)^T}(\hat{\S}^l\kron\S^l)^{-1} {\vecrm(\hat{\tY}^l)}
		+\vecrm(\bar{\tY}^l)^T(\hat{\S}^l\kron\S^l)^{-1} {\vecrm(\hat{\tY}^l)} ]d \vecrm(\hat{\tY}^l)\\
		=&-\frac{1}{2}\bphi^T(\K^r\kron\S^r)^{-1}\bphi-\frac{1}{2}\vecrm(\bar{\tY}^l)^T(\hat{\S}^l\kron\S^l)^{-1}\vecrm(\bar{\tY}^l)\\
		&+\frac{1}{2}\left(\bPsi^T(\K^r\kron\S^r)^{-1}\bphi+(\hat{\S}^l\kron\S^l)^{-1}\vecrm(\bar{\tY}^l)\right)^T
		\left(\bPsi^T(\K^r\kron\S^r)^{-1}\bphi + (\hat{\S}^l\kron\S^l)^{-1}\right)^{-1}\\
		&\left(\bPsi^T(\K^r\kron\S^r)^{-1}\bphi+(\hat{\S}^l\kron\S^l)^{-1}\vecrm(\bar{\tY}^l)\right)\\
		&+\frac{N^md^l}{2}\log 2\pi + \frac{1}{2}\log \left|\bPsi^T(\K^r\kron\S^r)^{-1}\bPsi + (\hat{\S}^l\kron\S^l)^{-1}\right|\\
		=&\frac{N^md^l}{2}\log 2\pi +\frac{1}{2} \log \left|\bPsi^T(\K^r\kron\S^r)^{-1}\bPsi + (\hat{\S}^l\kron\S^l)^{-1}\right|\\
		&-\frac{1}{2}\bphi^T\underbrace{\left[ (\K^r\kron\S^r)^{-1}-(\K^r\kron\S^r)^{-1}\bPsi\left(\bPsi^T(\K^r\kron\S^r)^{-1}\bPsi + (\hat{\S}^l\kron\S^l)^{-1}\right)^{-1}\bPsi^T(\K^r\kron\S^r)^{-1}\right]}_{\text{part 1}}\bphi\\
		&-\frac{1}{2}\vecrm(\bar{\tY}^l)^T \underbrace{\left[(\hat{\S}^l\kron\S^l)^{-1}-(\hat{\S}^l\kron\S^l)^{-1}\left(\bPsi^T(\K^r\kron\S^r)^{-1}\bPsi + (\hat{\S}^l\kron\S^l)^{-1}\right)^{-1}(\hat{\S}^l\kron\S^l)^{-1}\right]}_{\text{part 2}} \vecrm(\bar{\tY}^l)\\
		&+\bphi^T \underbrace{(\K^r\kron\S^r)^{-1}\bPsi\left(\bPsi^T(\K^r\kron\S^r)^{-1}\bPsi + (\hat{\S}^l\kron\S^l)^{-1}\right)^{-1}(\hat{\S}^l\kron\S^l)^{-1}}_{\text{part 3}}\vecrm(\bar{\tY}^l)
	\end{aligned}
\end{equation}
where $ \bphi $ is defined by Eq. (\ref{eq: components of non-subset}) ,
\begin{equation}
	\label{eq: components of non-subset}
	\begin{aligned}
	&\bphi = \left((\vecrm({\check{\tY}^h})^T,\vecrm({\hat{\tY}^h})^T) - (\vecrm(\check{\tY}^l)^T,\0)\tilde{\W}^T\right)\\
	&\bPsi = \hat{\E}\kron\W.
	 \end{aligned}
\end{equation}
With Sherman-Morrison formula, we can further simplify part 1 in \Eqref{log of non-subset} as
\begin{equation}\label{f1}
	\begin{aligned}
		&(\K^r\kron\S^r)^{-1}-(\K^r\kron\S^r)^{-1}\bPsi\left(\bPsi^T(\K^r\kron\S^r)^{-1}\bPsi + (\hat{\S}^l\kron\S^l)^{-1}\right)^{-1}\bPsi^T(\K^r\kron\S^r)^{-1}\\
		=&\left(\K^r\kron\S^r+\bPsi(\hat{\S}^l\kron\S^l)\bPsi^T\right)^{-1},\\
	\end{aligned}
\end{equation}
part 2 in \Eqref{log of non-subset} as
\begin{equation}\label{f2}
	\begin{aligned}
		&(\hat{\S}^l\kron\S^l)^{-1}-(\hat{\S}^l\kron\S^l)^{-1}\left(\bPsi^T(\K^r\kron\S^r)^{-1}\bPsi + (\hat{\S}^l\kron\S^l)^{-1}\right)^{-1}(\hat{\S}^l\kron\S^l)^{-1}\\
		=&(\hat{\S}^l\kron\S^l)^{-1}-\left((\hat{\S}^l\kron\S^l)\bPsi^T(\K^r\kron\S^r)^{-1}\bPsi(\hat{\S}^l\kron\S^l) + (\hat{\S}^l\kron\S^l)\right)^{-1}\\
		=&(\hat{\S}^l\kron\S^l)^{-1}-(\hat{\S}^l\kron\S^l)^{-1} + \bPsi^T(\K^r\kron\S^r+\bPsi^T(\hat{\S}^l\kron\S^l)\bPsi^T)^{-1}\bPsi\\
		=&\bPsi^T(\K^r\kron\S^r+\bPsi(\hat{\S}^l\kron\S^l)\bPsi^T)^{-1}\bPsi,
	\end{aligned}
\end{equation}
and part 3 in \Eqref{log of non-subset} as
\begin{equation}\label{f3}
	\begin{aligned}
		&(\K^r\kron\S^r)^{-1}\bPsi\left(\bPsi^T(\K^r\kron\S^r)^{-1}\bPsi + (\hat{\S}^l\kron\S^l)^{-1}\right)^{-1}(\hat{\S}^l\kron\S^l)^{-1}\\
		=&(\K^r\kron\S^r)^{-1}\bPsi\left(\hat{\S}^l\kron\S^l - (\hat{\S}^l\kron\S^l)\bPsi^T(\K^r\kron\S^r+\bPsi(\hat{\S}^l\kron\S^l)\bPsi^T)^{-1}\bPsi(\hat{\S}^l\kron\S^l)\right)(\hat{\S}^l\kron\S^l)^{-1}\\
		=&(\K^r\kron\S^r)^{-1}\bPsi-(\K^r\kron\S^r)^{-1}\bPsi (\hat{\S}^l\kron\S^l)\bPsi^T(\K^r\kron\S^r+\bPsi(\hat{\S}^l\kron\S^l)\bPsi^T)^{-1}\bPsi\\
		=&(\K^r\kron\S^r)^{-1}\bPsi-(\K^r\kron\S^r)^{-1}\left((\K^r\kron\S^r)+\bPsi (\hat{\S}^l\kron\S^l)\bPsi^T-(\K^r\kron\S^r)\right)\\
		&(\K^r\kron\S^r+\bPsi(\hat{\S}^l\kron\S^l)\bPsi^T)^{-1}\bPsi\\
		=&(\K^r\kron\S^r)^{-1}\bPsi-(\K^r\kron\S^r)^{-1}\left(\I-(\K^r\kron\S^r)(\K^r\kron\S^r+\bPsi(\hat{\S}^l\kron\S^l)\bPsi^T)^{-1}\right)\bPsi\\
		=&(\K^r\kron\S^r)^{-1}\bPsi-(\K^r\kron\S^r)^{-1}\bPsi+\left(\K^r\kron\S^r+\bPsi(\hat{\S}^l\kron\S^l)\bPsi^T\right)^{-1}\bPsi\\
		=&\left(\K^r\kron\S^r+\bPsi(\hat{\S}^l\kron\S^l)\bPsi^T\right)^{-1}\bPsi.
	\end{aligned}
\end{equation}
With the simplifications for part 1, 2, and 3 we get in \Eqref{f1}, \Eqref{f2} and \Eqref{f3}, the integral part will be more compact by substituting \Eqref{f1}, \Eqref{f2} and \Eqref{f3} back to \Eqref{log of non-subset} which is equal to 
\begin{equation}
	\begin{aligned}
		&\frac{N^md^l}{2}\log 2\pi +\frac{1}{2} \log \left|\bPsi^T(\K^r\kron\S^r)^{-1}\bPsi + (\hat{\S}^l\kron\S^l)^{-1}\right|
		-\frac{1}{2}\bphi^T\left(\K^r\kron\S^r+\bPsi(\hat{\S}^l\kron\S^l)\bPsi^T\right)^{-1}\bphi\\
		&-\frac{1}{2}\vecrm(\bar{\tY}^l)^T\bPsi^T(\K^r\kron\S^r+\bPsi(\hat{\S}^l\kron\S^l)\bPsi^T)^{-1}\bPsi\vecrm(\bar{\tY}^l)\\
		&+\bphi^T\left(\K^r\kron\S^r+\bPsi(\hat{\S}^l\kron\S^l)\bPsi^T\right)^{-1}\bPsi\vecrm(\bar{\tY}^l)\\
		=&\frac{N^md^l}{2}\log 2\pi + \frac{1}{2}\log \left|\bPsi^T(\K^r\kron\S^r)^{-1}\bPsi + (\hat{\S}^l\kron\S^l)^{-1}\right|\\ &-\frac{1}{2}(\bphi-\bPsi\vecrm(\bar{\tY}^l))^T\left(\K^r\kron\S^r+\bPsi(\hat{\S}^l\kron\S^l)\bPsi^T\right)^{-1}(\bphi-\bPsi\vecrm(\bar{\tY}^l)).
	\end{aligned}
\end{equation}
The determinant part of \Eqref{main part of non-subset} of the matrix can also be decomposed in the following way,
\begin{equation}
	\begin{aligned}
		&-\frac{1}{2}\log \left|{\K}^r\kron\S^r\right|-\frac{1}{2}\log \left|\hat{\S}^l\kron \S^l\right|+\frac{1}{2}\log \left|\bPsi^T(\K^r\kron\S^r)^{-1}\bPsi + (\hat{\S}^l\kron\S^l)^{-1}\right|\\ 
		=&-\frac{1}{2}\log \left|{\K}^r\kron\S^r\right|-\frac{1}{2}\log \left|\hat{\S}^l\kron \S^l\right|\\
		&+\frac{1}{2}\log \left|{\K}^r\kron\S^r\right|+\frac{1}{2}\log \left|\hat{\S}^l\kron \S^l\right| -\frac{1}{2}\log \left|\K^r\kron\S^r+\bPsi(\hat{\S}^l\kron\S^l)\bPsi^T\right|\\
		=&-\frac{1}{2}\log \left|\K^r\kron\S^r+\bPsi(\hat{\S}^l\kron\S^l)\bPsi^T\right|.
	\end{aligned}
\end{equation}
Putting everything we have derived up to this point, the joint likelihood for the non-subset data is:
\begin{equation}
	\label{A41}
	\begin{aligned}
		&\log p(\tY^l, \tY^h) \\
		=&\log p(\tY^l)
		-\frac{N^hd^h}{2}\log(2\pi)-\frac{1}{2}\log\left|\K^r\kron\S^r+\bPsi(\hat{\S}^l\kron\S^l)\bPsi^T\right|\\ &-\frac{1}{2}(\bphi-\bPsi\vecrm(\bar{\tY}^l))^T\left(\K^r\kron\S^r+\bPsi(\hat{\S}^l\kron\S^l)\bPsi^T\right)^{-1}(\bphi-\bPsi\vecrm(\bar{\tY}^l))\\
	\end{aligned}
\end{equation}
where $ \bphi-\bPsi\vecrm(\bar{\tY}^l) = \left(\begin{array}{c}
	\vecrm(\check{\tY}^h)\\
	\vecrm(\hat{\tY}^h)	\end{array}\right)-\tilde{\W}\left(\begin{array}{c}\vecrm(\check{\tY}^l) \\ \vecrm(\bar{\tY}^l)
\end{array}\right) $, and $ \K^r\kron\S^r+\bPsi(\hat{\S}^l\kron\S^l)\bPsi^T = {\K}^r\kron\S^r+\hat{\E}\hat{\S}^l\hat{\E}^T\kron\W^T\S^l{\W} $.

\subsection{Posterior Distribution for Non-Subset Multi-Fidelity Data}
We then explore the posterior distribution of this non-subset data structure. For the first, we use the integration to express the posterior distribution,
\begin{equation}
	\begin{aligned}
		p\left(\tY^*|\tY^l,\tY^h\right)
		=&\int p(\tY^*, \hat{\tY}^{l}|\tY^h,\tY^l)d\hat{\tY}^l\\
		=&\int p(\tY^*|\tY^h,\tY^l,\hat{\tY}^{l})p(\hat{\tY}^l) d\hat{\tY}^l\\
	\end{aligned}
\end{equation}
We try to express the integral by different parts. 
Once $ \hat{\tY}^l $ is decided,
the predictive posterior $ p(\tY^*|\tY^h,\tY^l,\hat{\tY}^{l})$ can be described using the 
standard subset way, which is $ p(\tY^*|\tY^h,\tY^l,\hat{\tY}^{l}) \sim \N\left(\vecrm(\bar{\tZ}_*^h),\S_*^h\right) $, where the mean function and covariance matrix are
\begin{equation}\label{given haty posterior}
	\begin{aligned}
		\vecrm(\bar{\tZ}_*^h) = & \left(\k^l_*(\K^l)^{-1}\kron\W\right)\left(\begin{array}{c} \vecrm({\tY^l}) \\ \vecrm({\hat{\tY}^l})\end{array}\right)+\left(\k^r_*(\K^r)^{-1}\kron\I^h\right)\vecrm({\tY^r}),\\
		\S_*^h = 
		& \left(k^l_{**} - (\k^l_*)^T(\K^l)^{-1}\k^l_*\right)\kron\W\S^l\W^T+\left(k^r_{**} - (\k^r_*)^T(\K^r)^{-1}\k^r_*\right)\kron\S^r.
	\end{aligned}
\end{equation}
We further simplify the situation and introduce definitions:
 \begin{equation*}
 	\begin{aligned}
 		&\K_*^l = k^l_{**} - (\k^l_*)^T(\K^l)^{-1}\k^l_*,\quad \\ 
		&\K_*^r = k^r_{**} - (\k^r_*)^T(\K^r)^{-1}\k^r_*
	\end{aligned}
\end{equation*}
Which simplify \Eqref{given haty posterior} as
\[ 
	\S_*^h = \K_*^l\kron \W\S^l\W^T + \K_*^r \kron \S^r.
\]
At the same time, since $ \hat{\tY}^l $ is the sample from $ \tY^l $, so it also follows the posterior distribution in subset way, which means 
$ \hat{\tY}^l \sim \N\left( \vecrm(\bar{\tY}^l), \hat{\S}^l\kron\S^l\right) $ where the $ \vecrm(\bar{\tY}^l) $ and $ \S_*^l\kron\S^l $are 
\begin{equation*}
	\begin{aligned}
	\vecrm(\bar{\tY}^l) &= \left(\k^l_*(\K^l)^{-1}\kron\I^l\right)\vecrm({\tY^l}),\quad \\
	\hat{\S}^l\kron\S^l &= \left(k^l_{**} - (\k^l_*)^T(\K^l)^{-1}\k^l_*\right)\kron\S^l.
\end{aligned}
\end{equation*}
Therefore the posterior distribution of non-subset data structure is
\begin{equation}\label{post non}
	\begin{aligned}
		&p\left(\tY^*|\tY^l,\tY^h\right)\\
		=& \int p(\tY^*, \hat{\tY}^{l}|\tY^h,\tY^l)d\hat{\tY}^l \\
		=&\int p(\tY^*|\tY^h,\tY^l,\hat{\tY}^{l})p(\hat{\tY}^l) d\hat{\tY}^l\\
		=& \int 2\pi^{-\frac{N^{p}d^h}{2}}\times \left|\S_*^h\right|^{-\frac{1}{2}}\\
		 &\times \exp \left[-\frac{1}{2}
		\left(\vecrm(\tY^*)-\left(\k_*^l(\hat{\K}^l)^{-1}\kron\W\right)
		\left(\begin{array}{c}
			\vecrm({\tY}^l) \\ \vecrm(\hat{\tY}^l)
		\end{array}\right)
		-\left(\k_*^r(\K^r)^{-1}\kron\I^h\right)\vecrm(\tY^r)\right)^T \right.\\
		& \left.\quad (\S_*^h)^{-1}
		\left(\vecrm(\tY^*)-\left(\k_*^l(\hat{\K}^l)^{-1}\kron\W\right)
		\left(\begin{array}{c}
			\vecrm({\tY}^l) \\ \vecrm(\hat{\tY}^l)
		\end{array}\right)
		-\left(\k_*^r(\K^r)^{-1}\kron\I^h\right)\vecrm(\tY^r)\right)\right]\\
		&\times 2\pi^{-\frac{N^{m}d^l}{2}}\times \left|\hat{\S}^l\kron\S^l\right|^{-\frac{1}{2}}\times\exp[-\frac{1}{2} (\vecrm({\hat{\tY}^l}) -\vecrm(\bar{\tY}^l))^T(\hat{\S}^l\kron\S^l)^{-1}(\vecrm({\hat{\tY}^l}) -\vecrm(\bar{\tY}^l))] d\hat{\tY}^l\\
		=& 2\pi^{-\frac{N^pd^h+N^md^l}{2}} \times \left|\S_*^h\right|^{-\frac{1}{2}} \times \left|\hat{\S}^l\kron\S^l\right|^{-\frac{1}{2}}
		\times\exp\left[-\frac{1}{2}\tilde{\tY}^T(\S_*^h)^{-1}\tilde{\tY}-\frac{1}{2}\vecrm(\bar{\tY}^l)^T(\hat{\S}^l\kron\S^l)^{-1}\vecrm(\bar{\tY}^l)\right]\\
		&\times \int \exp[+\tilde{\tY}^T(\S_*^h)^{-1}\bGamma\vecrm(\hat{\tY}^l)+
		\vecrm(\bar{\tY}^l)^T(\hat{\S}^l\kron\S^l)^{-1}\vecrm({\hat{\tY}^l})\\
		&-\frac{1}{2}\vecrm(\hat{\tY}^l)^T\bGamma^T(\S_*^h)^{-1}\bGamma\vecrm(\hat{\tY}^l)-\frac{1}{2}\vecrm({\hat{\tY}^l})^T(\hat{\S}^l\kron\S^l)^{-1}\vecrm({\hat{\tY}^l})]d\hat{\tY}^l\\
		=& 2\pi^{-\frac{N^pd^h+N^md^l}{2}} \times \left|\S_*^h\right|^{-\frac{1}{2}} \times \left|\hat{\S}^l\kron\S^l\right|^{-\frac{1}{2}}
		\times \exp \left[-\frac{1}{2}\tilde{\tY}^T(\S_*^h)^{-1}\tilde{\tY}-\frac{1}{2}\vecrm(\bar{\tY}^l)^T(\hat{\S}^l\kron\S^l)^{-1}\vecrm(\bar{\tY}^l) \right]\\
		&\times 2\pi^{\frac{N^md^l}{2}} \times \left|\bGamma^T(\S_*^h)^{-1}\bGamma+(\hat{\S}^l\kron\S^l)^{-1}\right|^{-\frac{1}{2}}
		\times \exp \left[\frac{1}{2}\left(\tilde{\tY}^T(\S_*^h)^{-1}\bGamma+\vecrm(\bar{\tY}^l)^T(\hat{\S}^l\kron\S^l)^{-1}\right) \right.\\
		& \left. \left(\bGamma^T(\S_*^h)^{-1}\bGamma+(\hat{\S}^l\kron\S^l)^{-1}\right)^{-1}\left(\tilde{\tY}^T(\S_*^h)^{-1}\bGamma+\vecrm(\bar{\tY}^l)^T(\hat{\S}^l\kron\S^l)^{-1}\right)^T \right]\\
		=& 2\pi^{-\frac{N^pd^h}{2}} \times \left|\S_*^h\right|^{-\frac{1}{2}} \times \left|\hat{\S}^l\kron\S^l\right|^{-\frac{1}{2}} \times \left|\bGamma^T(\S_*^h)^{-1}\bGamma+(\hat{\S}^l\kron\S^l)^{-1}\right|^{-\frac{1}{2}}
		\times\exp\left[-\frac{1}{2}\tilde{\tY}^T(\S_*^h)^{-1}\tilde{\tY} \right.\\
		&-\frac{1}{2}\vecrm(\bar{\tY}^l)^T(\hat{\S}^l\kron\S^l)^{-1}\vecrm(\bar{\tY}^l) 
		+\frac{1}{2}\tilde{\tY}^T(\S_*^h)^{-1}\bGamma\left(\bGamma^T(\S_*^h)^{-1}\bGamma+(\hat{\S}^l\kron\S^l)^{-1}\right)^{-1}\bGamma^T(\S_*^h)^{-1}\tilde{\tY}\\
		&+\frac{1}{2}\vecrm(\bar{\tY}^l)^T(\hat{\S}^l\kron\S^l)^{-1}\left(\bGamma^T(\S_*^h)^{-1}\bGamma+(\hat{\S}^l\kron\S^l)^{-1}\right)^{-1}(\hat{\S}^l\kron\S^l)^{-1}\vecrm(\bar{\tY}^l)\\
		&\left.+\tilde{\tY}^T(\S_*^h)^{-1}\bGamma\left(\bGamma^T(\S_*^h)^{-1}\bGamma+(\hat{\S}^l\kron\S^l)^{-1}\right)^{-1}(\hat{\S}^l\kron\S^l)^{-1}\vecrm(\bar{\tY}^l)\right]\\
		=& 2\pi^{-\frac{N^pd^h}{2}} \times \underbrace{ \left|\S_*^h\right|^{-\frac{1}{2}} \times \left|\hat{\S}^l\kron\S^l\right|^{-\frac{1}{2}} \times \left|\bGamma^T(\S_*^h)^{-1}\bGamma+(\hat{\S}^l\kron\S^l)^{-1}\right|^{-\frac{1}{2}}}_{\text{part d}}\\
		&\times\exp \left[-\frac{1}{2}\tilde{\tY}^T \underbrace{ \left((\S_*^h)^{-1}-(\S_*^h)^{-1}\bGamma\left(\bGamma^T(\S_*^h)^{-1}\bGamma+(\hat{\S}^l\kron\S^l)^{-1}\right)^{-1}\bGamma^T(\S_*^h)^{-1}\right)}_{\text{part a}}\tilde{\tY} \right. \\ 
		&-\frac{1}{2}\vecrm(\bar{\tY}^l)^T \underbrace{ \left((\hat{\S}^l\kron\S^l)^{-1}-(\hat{\S}^l\kron\S^l)^{-1}\left(\bGamma^T(\S_*^h)^{-1}\bGamma+(\hat{\S}^l\kron\S^l)^{-1}\right)^{-1}(\hat{\S}^l\kron\S^l)^{-1}\right) }_{\text{part b}} \vecrm(\bar{\tY}^l)\\
		& \left.+\tilde{\tY}^T \underbrace{(\S_*^h)^{-1}\bGamma\left(\bGamma^T(\S_*^h)^{-1}\bGamma+(\hat{\S}^l\kron\S^l)^{-1}\right)^{-1}(\hat{\S}^l\kron\S^l)^{-1} }_{\text{part c}} \vecrm(\bar{\tY}^l)\right]\\
	\end{aligned}
\end{equation}
where 
$ \tilde{\tY} $ and $ \bGamma $ is defined by the following equation,
\begin{equation}
	\begin{aligned}
		&\tilde{\tY} = \left(\vecrm(\tY^*)-\left(\bk_*^l(\hat{\K}^l)^{-1}\kron\W\right)
		\left(\begin{array}{c}
			\vecrm({\tY}^l) \\ \0
		\end{array}\right)
		-\left(\bk_*^r(\K^r)^{-1}\kron\I^h\right)\left(\vecrm({\tY}^h) - \left(\begin{array}{c}
			\vecrm(\check{\tY}^l) \\ \0
		\end{array}\right) \right)\right),\\
		&\bGamma = \left([ \bk_{*}^r (\K^r)^{-1}\E_n^T- \bk_{*}^l(\hat{\K}^l)^{-1}]\kron\W\right)\E_m\kron\I^l.
	\end{aligned}	
\end{equation}

We then utilize the Sherman-Morrison formula to simplify part a, b, and c in \Eqref{post non} as follows.
For part a in \Eqref{post non},
\begin{equation}
	\begin{aligned}
		&(\S_*^h)^{-1}-(\S_*^h)^{-1}\bGamma\left(\bGamma^T(\S_*^h)^{-1}\bGamma+(\hat{\S}^l\kron\S^l)^{-1}\right)^{-1}\bGamma^T(\S_*^h)^{-1}\\
		=&\left(\S_*^h+\bGamma(\hat{\S}^l\kron\S^l)\bGamma^T\right)^{-1},\\
	\end{aligned}	
\end{equation}
for part b in \Eqref{post non}, 
\begin{equation}
	\begin{aligned}
		&(\hat{\S}^l\kron\S^l)^{-1}-(\hat{\S}^l\kron\S^l)^{-1}\left(\bGamma^T(\S_*^h)^{-1}\bGamma+(\hat{\S}^l\kron\S^l)^{-1}\right)^{-1}(\hat{\S}^l\kron\S^l)^{-1}\\
		=&(\hat{\S}^l\kron\S^l)^{-1}-\left((\hat{\S}^l\kron\S^l)\bGamma^T(\S_*^h)^{-1}\bGamma(\hat{\S}^l\kron\S^l)+(\hat{\S}^l\kron\S^l)\right)^{-1}\\
		=&(\hat{\S}^l\kron\S^l)^{-1}-\left((\hat{\S}^l\kron\S^l)^{-1} - \bGamma^T(\S_*^h + \bGamma(\hat{\S}^l\kron\S^l)\bGamma^T )^{-1}\bGamma\right)\\
		=&\bGamma^T(\S_*^h + \bGamma(\hat{\S}^l\kron\S^l)\bGamma^T )^{-1}\bGamma,\\
	\end{aligned}	
\end{equation}
and for part c in \Eqref{post non},
\begin{equation}
	\begin{aligned}
		&(\S_*^h)^{-1}\bGamma\left(\bGamma^T(\S_*^h)^{-1}\bGamma+(\hat{\S}^l\kron\S^l)^{-1}\right)^{-1}(\hat{\S}^l\kron\S^l)^{-1}\\
		=&(\S_*^h)^{-1}\bGamma
		\left((\hat{\S}^l\kron\S^l) - (\hat{\S}^l\kron\S^l)\bGamma^T(\S_*^h+\bGamma(\hat{\S}^l\kron\S^l)\bGamma^T)^{-1}\bGamma(\hat{\S}^l\kron\S^l)\right)
		(\hat{\S}^l\kron\S^l)^{-1}\\
		=&(\S_*^h)^{-1}\bGamma - (\S_*^h)^{-1}\bGamma(\hat{\S}^l\kron\S^l)\bGamma^T(\S_*^h+\bGamma(\hat{\S}^l\kron\S^l)\bGamma^T)^{-1}\bGamma\\
		=&(\S_*^h)^{-1}\bGamma - (\S_*^h)^{-1}\left(\S_*^h+\bGamma(\hat{\S}^l\kron\S^l)\bGamma^T-\S_*^h\right)(\S_*^h+\bGamma(\hat{\S}^l\kron\S^l)\bGamma^T)^{-1}\bGamma\\
		=&(\S_*^h)^{-1}\bGamma - (\S_*^h)^{-1} \left(\I-\S_*^h(\S_*^h+\bGamma(\hat{\S}^l\kron\S^l)\bGamma^T)^{-1}\right)\bGamma\\
		=&(\S_*^h)^{-1}\bGamma - (\S_*^h)^{-1}\bGamma -(\S_*^h+\bGamma(\hat{\S}^l\kron\S^l)\bGamma^T)^{-1}\bGamma\\
		=& -(\S_*^h+\bGamma(\hat{\S}^l\kron\S^l)\bGamma^T)^{-1}\bGamma.\\
	\end{aligned}	
\end{equation}

And the determinant (part d in \Eqref{post non}) can also use the Sherman-Morrison formula to derive a more compact version,
\begin{equation}
	\begin{aligned}
		&\left|\S_*^h\right|^{-\frac{1}{2}} \times \left|\hat{\S}^l\kron\S^l\right|^{-\frac{1}{2}} \times \left|\bGamma^T(\S_*^h)^{-1}\bGamma+(\hat{\S}^l\kron\S^l)^{-1}\right|^{-\frac{1}{2}}\\
		=&\left|\S_*^h\right|^{-\frac{1}{2}} \times \left|\hat{\S}^l\kron\S^l\right|^{-\frac{1}{2}} \times \left|(\hat{\S}^l\kron\S^l)^{-1}\right|^{-\frac{1}{2}} \times 
		\left|(\S_*^h)^{-1}\right|^{-\frac{1}{2}} \times
		\left|(\S_*^h)^{-1}+\bGamma(\hat{\S}^l\kron\S^l)^{-1}\bGamma^T\right|^{-\frac{1}{2}}\\
		=&\left|(\S_*^h)^{-1}+\bGamma(\hat{\S}^l\kron\S^l)^{-1}\bGamma^T\right|^{-\frac{1}{2}}
	\end{aligned}	
\end{equation}
Taking part a, b, c, and d back into  \Eqref{post non}, we have the compact form
\begin{equation}
	\begin{aligned}
		&p\left(\tY^*|\tY^l,\tY^h\right)\\
		=&2\pi^{-\frac{N^pd^h}{2}} \times \left|(\S_*^h)^{-1}+\bGamma(\hat{\S}^l\kron\S^l)^{-1}\bGamma^T\right|^{-\frac{1}{2}}
		\times\exp[-\frac{1}{2}\tilde{\tY}^T\left(\S_*^h+\bGamma(\hat{\S}^l\kron\S^l)\bGamma^T\right)^{-1}\tilde{\tY}\\
		&-\frac{1}{2}\vecrm(\bar{\tY}^l)^T\bGamma^T\left(\S_*^h+\bGamma(\hat{\S}^l\kron\S^l)\bGamma^T\right)^{-1}\bGamma\vecrm(\bar{\tY}^l)
		-\tilde{\tY}^T\left(\S_*^h+\bGamma(\hat{\S}^l\kron\S^l)\bGamma^T\right)^{-1}\bGamma\vecrm(\bar{\tY}^l)]\\
		=&2\pi^{-\frac{N^pd^h}{2}} \times \left|(\S_*^h)^{-1}+\bGamma(\hat{\S}^l\kron\S^l)^{-1}\bGamma^T\right|^{-\frac{1}{2}}\\
		&\times\exp[-\frac{1}{2}\left(\tilde{\tY}-\bGamma\vecrm(\bar{\tY}^l)\right)^T\left(\S_*^h+\bGamma(\hat{\S}^l\kron\S^l)\bGamma^T\right)^{-1}\left(\tilde{\tY}-\bGamma\vecrm(\bar{\tY}^l)\right)]\\
		=&2\pi^{-\frac{d^h}{2}}\times \left|\S_*^h+\bGamma\left(\hat{\S}^l\kron\S^l\right)\bGamma^T\right|^{-\frac{1}{2}}\\
		& \quad \times \exp\left[-\frac{1}{2}\left(\vecrm(\tZ^h_*)-\vecrm(\bar{\tZ})\right)^T
		\left(\S_*^h+\bGamma\left(\hat{\S}^l\kron\S^l\right)\bGamma^T\right)^{-1}\left(\vecrm(\tZ^h_*)-\vecrm(\bar{\tZ})\right)\right].
	\end{aligned}
\end{equation}
We can see the joint likelihood ends up with a elegant formulation about the low-fidelity TGP and residual TGP.

\subsection{CIGAR}
As we mentioned in the Section \ref{cigar}, we assume the output covariance matrixes $ \S_m^h $ and $ \S_m^l $ are identical matrixes and orthogonal weight matrixes, \ie $ \W_m^T\W_m = \I $.
Substituting these assumptions into \eqref{A41}, we get the simplified covariance matrix,
\begin{equation}
	\begin{aligned}
		&{\K}^r\kron\I^r+\hat{\E}\hat{\S}^l\hat{\E}^T\kron\W^T\I^l{\W}\\
		=&{\K}^r\kron\I^r+\hat{\E}\hat{\S}^l\hat{\E}^T\kron\W^T{\W}\\
		=&{\K}^r\kron\I^r+\hat{\E}\hat{\S}^l\hat{\E}^T\kron\I^h\\
		=&({\K}^r+\hat{\E}\hat{\S}^l\hat{\E}^T)\kron\I^h
	\end{aligned}
\end{equation}
where $\I^r$ is a identical matrix of size $d^r \times d^r$; the same rules apply to $\I^r$; and $ \I^r = \I^h $. 
The joint likelihood of non-subset data becomes 
\begin{equation}
	\begin{aligned}
		\log p(\tY^l, \tY^h) 
		=&\log p(\tY^l)
		-\frac{N^hd^h}{2}\log(2\pi)-\frac{1}{2}\log\left|({\K}^r+\hat{\E}\hat{\S}^l\hat{\E}^T)\kron\I^h\right|\\ &-\frac{1}{2}(\bphi-\bPsi\vecrm(\bar{\tY}^l))^T\left(({\K}^r+\hat{\E}\hat{\S}^l\hat{\E}^T)\kron\I^h\right)^{-1}(\bphi-\bPsi\vecrm(\bar{\tY}^l))\\
	\end{aligned}
\end{equation}
where $ \left(\bphi-\bPsi\vecrm(\bar{\tY}^l)\right) = \left(\begin{array}{c}
	\vecrm(\check{\tY}^h)\\
	\vecrm(\hat{\tY}^h)	\end{array}\right)-\tilde{\W}\left(\begin{array}{c}\vecrm(\check{\tY}^l) \\ \vecrm(\bar{\tY}^l)
\end{array}\right) $. 

We can see that the complexity of kernel matrix inversion is reduced to $ \Ocal((N^h)^3) $.

\subsection{$\tau$-Fidelity Autoregression Model}
As we mentioned in Section \ref{sec:Statement}, we can apply the AR to more levels of fidelity, so the GAR does. In this section, we try to expand the GAR into more levels of fidelity. Assuming the $ \tF^\tau(\x) = \tF^{\tau-1}(\x) \times_1 \W^{\tau-1}_1 \times_2 \cdots \times_M \W_M^{\tau-1} + \tF^{r}_\tau(\x) $, we can derive the joint covariance matrix, 
\[ \bSigma^\tau =  
 \left(\begin{array}{cc}
	\K^{\tau-1}(\X^{\tau-1}, \X^{\tau-1}) \kron \S^{\tau-1} &
	\K^{\tau-1}(\X^{\tau-1},\X^\tau) \kron \S^{\tau-1}(\W^{\tau-1})^T\\
	\K^{\tau-1}(\X^\tau,\X^{\tau-1})\kron \W^{\tau-1}\S^{\tau-1} & \K^{\tau-1}(\X^\tau,\X^\tau)\kron \W^{\tau-1}\S^{\tau-1}(\W^{\tau-1})^T + 
	\K^r_\tau(\X^\tau,\X^\tau)\kron \S_\tau^r
\end{array}
\right),  \]
where $ \S^{\tau-1} = \bigotimes_{m=1}^M\S^{\tau-1}_m $ and $ \W^{\tau-1} = \bigotimes_{m=1}^M\W^{\tau-1}_m $.\\
As same as the proof of GAR, we can derive the inversion of the joint covariance matrix,
$ (\bSigma^\tau)^{-1} =  $
$ \left[
\begin{array}{cc}
	(\K^{\tau-1})^{-1}\kron{(\S^{\tau-1})}^{-1} + \left(\begin{array}{cc}\0 &\0 \\ \0 & (\K^r_\tau)^{-1}\kron{\W^{\tau-1}}^T({\S}^r_\tau)^{-1}{\W^{\tau-1}}\end{array}\right) & -\left(\begin{array} {c} \0 \\ (\K^r_\tau)^{-1}\kron{\W^{\tau-1}}^T({\S}^r_\tau)^{-1}\end{array} \right) \\
	-\left(\0, (\K^r_\tau)^{-1}\kron({\S}^r_\tau)^{-1}\W^{\tau-1}\right) & (\K^r_\tau)^{-1}\kron({\S}^r_\tau)^{-1}
\end{array}\right]  $\\
Therefore, we have shown here that building an s-level TGP is equivalent to building s independent TGPs. We present the mean function and covariance matrix of the posterior distribution,
\begin{equation}
	\begin{aligned}
		\vecrm(\tZ^\tau_*) = & \left(\k^{\tau-1}_*(\K^{\tau-1})^{-1} \kron \W^{\tau-1} \right)\vecrm({\tY^{\tau-1}}) +  \left((\k^r_\tau)_*(\K^r_\tau)^{-1} \kron \I_r \right)\vecrm({\tY^r_\tau})\\
		\S_*^\tau = 
		& \left(k^{\tau-1}_{**} - (\k^{\tau-1}_*)^T(\K^{\tau-1})^{-1}\k^{\tau-1}_*\right) \kron \W^{\tau-1}\S^{\tau-1}(\W^{\tau-1})^T +\left((k^r_\tau)_{**} - (\k^r_\tau)_*^T(\K^r_\tau)^{-1}(\k^r_\tau)_*\right) \kron \S^r_\tau.\\
		\\
	\end{aligned}
\end{equation}

\section{Summary of the SOTA methods}
\label{appe: sota}
We compare and conclude the capability and complexity of the SOTA methods, \ours, and \ourss in Table~\ref{tab:caption}. 

\begin{table}[]
  \centering
  \caption{Comparison of SOTA multi-fidelity fusion for high-diemsnosional problems} 
  \label{tab:caption}
  \begin{tabular}{ r|c|c|l } 
    \hline
    Model & Arbitrary outputs? & Non-subset data? & Complexity \\
    \hline
    NAR \citep{perdikaris2017nonlinear} & Yes & No & $\Ocal(\sum_i(N^i)^3)$ \\
    ResGP\citep{xing2021residual} & No & No & $\Ocal(\sum_i(N^i)^3)$ \\
    MF-BNN \citep{li2020deep} & Yes & Yes &  $\Ocal( \sum_i(N^i) (A_i^2+ \omega)   )$* \\
    DC \citep{xing2021deep} & Yes & No & $\Ocal(\sum_i(N^i)^3)$ \\
    AR \citep{kennedy2000predicting} & No & No & $\Ocal(\sum_i(N^i d^i)^3)$ \\
    GAR  & Yes & Yes & $\Ocal(\sum_i \sum_{m=1}^M (d^i_m)^3+(N^i)^3)$ \\
    CIGAR & Yes & Yes & $\Ocal(\sum_i (N^i)^3)$ \\
\hline
\multicolumn{4}{l}{\small *$A_i$ is the total weight size of NN for i-th fidelity and $\omega$ is the number of all parameters} \\
  \end{tabular}
\end{table}%

\section{Implementation and Complexity}
We now present the training and prediction algorithm for \ours and \ourss using tensor algebra so that the full covariance matrix is never assembled or explicitly computed to improve computational efficiency.
We use a normal TGP as example, given the dataset $ (\X, \tY) $, $ \vecrm(\tY) \sim \N\left(\0,\K(\X, \X)\kron\left(\bigotimes_{m=1}^M\S_m\right)\right) $. The inference needs to estimate all the covariance matrix $ \bigotimes_{m=1}^M\S_m $ and $ \K(\X, \X) $. For compactness, we use $ \S $ and $ \K $ to denote $ \bigotimes_{m=1}^M\S_m $ and $ \K(\X, \X) $, and $ \bSigma = \K\kron\S + \epsilon^{-1}\I $. We estimate parameters by minimizing the negative log likelihood of the model,
\[
\mathcal{L} = \frac{Nd}{2}\log (2\pi) +  \frac{1}{2}\log\left|\bSigma\right|+\frac{1}{2}\vecrm(\tY)^T\bSigma^{-1}\vecrm(\tY).
\]
However, since the $ \S $ is a matrix of size $ Nd \times Nd $, when the size of outputs is large, it will be unable to compute the inversion of $ \K\kron\S $. So for the TGP, we exploit the Kronecker product in $ \K\kron\S $ to calculate the negative log-likelihood efficiently.
Firstly, we use eigendecomposition to denote the joint kernel matrix, $ \K = \U^T \text{diag}({\lambda}) \U $ and $ \S_m = \U^T_m \text{diag}({\lambda_m}) \U_m $. Then we use $ \bSigma $ to denote the joint kernel matrix, $ \bSigma = \K\kron\S + \epsilon^{-1}\I = \left(\U^T \text{diag}({\lambda}) \U\right) \kron \left( \U^T_1 \text{diag}({\lambda_1}) \U_1 \right) \kron \cdots \kron \left( \U^T_M \text{diag}({\lambda_M}) \U_M \right) + \epsilon^{-1}\I $. With the Kronecker product property, we can have that
\begin{equation}
	\bSigma = \P^T\Lambda\P + \epsilon^{-1}\I
\end{equation}
where $ \P = \U \kron \U_1 \kron \cdots \kron \U_M $ and $ \Lambda = \text{diag}(\lambda \kron \lambda_1 \kron \cdots \kron \lambda_M ) $ since $ \U $ and $ \U_m $ is eigenvectors and orthogonal, so $ \P^T\P = \P\P^T = \I $. Therefore, we can have that 
\begin{equation}
	\begin{aligned}
	\log\left|\bSigma\right| 
	&= \log\left|\P^T\Lambda\P + \epsilon^{-1}\I\right| 
	= \log \left|\P^T(\Lambda + \epsilon^{-1}\I)\P\right|
	= \log \left|\Lambda + \epsilon^{-1}\I\right|.
	\end{aligned}
\end{equation}
Therefore, we only need to compute $ Nd $ diagonal elements to calculate part of the negative log-likelihood.

After that, we compute the $ \vecrm(\tY)^T\bSigma^{-1}\vecrm(\tY) $ part in the negative log likelihood. First, we have $ \tA = \lambda \circ \lambda_1 \circ \cdots \circ \lambda_M + \epsilon^{-1} \mathbbm{1} $, where $ \mathbbm{1} $ is a tensor of full ones and $ \circ $ is the Kruskal operator. Then we have
\begin{equation}
	\begin{aligned}
		\vecrm(\tY)^T\bSigma^{-1}\vecrm(\tY)
		&= \vecrm(\tY)^T\bSigma^{-\frac{1}{2}}\bSigma^{-\frac{1}{2}}\vecrm(\tY)\\
		&= \vecrm(\tY)^T\P^T(\Lambda + \epsilon^{-1}\I)^{-\frac{1}{2}}\P
		\P(\Lambda + \epsilon^{-1}\I)^{-\frac{1}{2}}\P^T\vecrm(\tY)\\
		&= \eta^T\eta,
	\end{aligned}
\end{equation}
where $ \eta = \P(\Lambda + \epsilon^{-1}\I)^{-\frac{1}{2}}\P^T\vecrm(\tY) $.
Since $ \P $ is a Kronecker product matrix, we can apply the property of Tucker operator~\citep{kolda2006multilinear} to compute $ \b $.
\begin{equation}
	\begin{aligned}
		&\tT_1 = \tY \times_1 \U^T \times_2 \U_1^T \times_3 \cdots \times_{M+1}\ \U_M^T\\
		&\tT_2 = \tT_1 \odot \tA^{\cdot-\frac{1}{2}}\\
		&\tT_3 = \tT_2 \times_1 \U \times_2 \U_1 \times_3 \cdots \times_{M+1}\ \U_M\\
		&\eta = \vecrm(\tT_3)
	\end{aligned}
\end{equation}
where $ \odot $ means element-wise product, and $ (\cdot)^{\cdot-\frac{1}{2}} $ means take power of $ -\frac{1}{2} $ element wisely. Therefore the complexity of negative log likelihood is $ \Ocal(\sum_{m=1}^M (d_m)^3+(N)^3)$.\\
Based on the above conclusions, we can also calculate the \ours more efficiently. According to Lemma \ref{lemma3: GAR}, the joint likelihood admits two separable likelihoods $ \Lcal^l $ and $ \Lcal^r $. For each of these two, we can use the tricks to reduce the complexity to $ \Ocal(\sum_{m=1}^M (d_m^l)^3+(N^l)^3) + \Ocal(\sum_{m=1}^M (d_m^r)^3+(N^r)^3) $. Since, 
\begin{equation}
	\begin{aligned}
		&\log \left|\K^l\kron\S^l\right| = \log \left|\Lambda^l+\epsilon^{-1}\I^l\right|,\\
		&\vecrm(\tY^l)^T(\K^l\kron\S^l)^{-1}\vecrm(\tY^l) = (\eta^l)^T\eta^l;
	\end{aligned}
\quad
	\begin{aligned}
		&\log \left|\K^r\kron\S^r\right| = \log \left|\Lambda^r+\epsilon^{-1}\I^r\right|,\\
		&\vecrm(\tY^r)^T(\K^r\kron\S^r)^{-1}\vecrm(\tY^r) = (\eta^r)^T\eta^r,
	\end{aligned}
\end{equation}
in which $ \eta^h $, $ \eta^l $ and $ \Lambda^h $, $ \Lambda^l $ are low-fidelity data and residuals corresponding vectors and eigenvalues.
Therefore, the joint log-likelihood will be,
\begin{equation}
	\begin{aligned}
		\Lcal &= \Lcal^l + \Lcal^r\\
		&= \text{const} -\frac{1}{2} \log \left|\Lambda^l+\epsilon^{-1}\I^l\right| -\frac{1}{2}(\eta^l)^T\eta^l
		-\frac{1}{2} \log \left|\Lambda^r+\epsilon^{-1}\I^r\right| -\frac{1}{2}(\eta^r)^T\eta^r
	\end{aligned}
\end{equation}

Given a new input $ \x_* $, the prediction of the output tensorized as $ \vecrm(\tZ^h_*) $ is a conditional Gaussian distribution
$\vecrm(\tZ_*) \sim \mathcal{N}(\vecrm(\bar{\tZ}_*), {\S_*})$, where
\begin{equation}
	\label{eq gar post}
	\begin{aligned}
		\vecrm({\bar{\tZ}_*})
		&=\left( \k_* \left(\K\right)^{-1} \kron \I \right)\vecrm({\tY})\\
		\S_*
		&=  \left(k_{**} - (\k_*)^T \left(\K\right)^{-1} \k_* \right) \kron \S.
	\end{aligned}
\end{equation}
We can use the Tucker operator to compute the predictive mean $ \vecrm(\bar{\tZ}_*) $ and $ {\S_*} $ in a more efficient way. Using the eigendecomposition of kernel matrix, we can derive that
$ {\S_*} = k_{**} \kron \S - \L\L^T $, where $ \L = ((\k_*)^T \left(\K\right)^{-1}\U\kron\U_1\kron\cdots\kron\U_M)(\Lambda(\Lambda+\epsilon^{-1}\I)^{-\frac{1}{2}}) $. Therefore, the $ \text{diag}({\S_*}) = k_{**}\kron\text{diag}(\S) - \text{diag}(\L\L^T) $. We can also use tensor algebra to calculate the predictive covariance matrix
\[ \text{diag}({\S_*}) = \vecrm(\tM), \] where $ \tM = k_{**}(\text{diag}(\S_1)\circ\cdots\circ\text{diag}(\S_M))+\left((\lambda \circ \lambda_1 \circ \cdots \circ \lambda_M)\odot\tA^{\cdot-\frac{1}{2}}\right)^{\cdot2}\times_1(\k_*\K^{-1}\U)^{\cdot2}\times_2(\U_1)^{\cdot2}\times_3\cdots\times_{M+1}(\U_M)^{\cdot2} $. Therefore, we can also compute the predictive covariance matrix $ {\S_*^h} $ in \ours efficiently.
\begin{equation}
	\begin{aligned}
		\text{diag}({\S_*^h}) = \vecrm(\tM^l) + \vecrm(\tM^r)
	\end{aligned}
\end{equation}
where the $ \vecrm(\tM^l) $ and $ \vecrm(\tM^r) $ are vectors for low-fidelity and residual data. When we calculate the $ \vecrm(\tM^l) $, we need to be careful that the output kernel matrix should be $ \W\S^l\W^T $.

\section{Experiment in Detail}
\subsection{Canonical PDEs}
\label{appe pdes}
We consider three canonical PDEs: Poisson's equation, the heat equation, and Burger's equation, 
These PDEs have crucial roles in scientific and technological applications~\citesupp{chapra2010numerical,chung2010computational,burdzy2004heat}.
They offer common simulation scenarios, such as high-dimensional spatial-temporal field outputs, nonlinearities, and discontinuities, and are frequently used as benchmark issues for surrogate models~\citep{xing2021deep,tuo2014surrogate,efe2003proper,raissi2017machine}. 
$ x $ and $ y $ denote the spatial coordinates, and $ t $ specifies the time coordinate, which contradicts the notation in the main paper. This notation in the appendix serves merely to make the information clear; it has no bearing on or connections to the main article.

\noindent \textbf{Burgers' equation} is regarded as a standard nonlinear hyperbolic PDE; it is commonly used to represent a variety of physical phenomena, including fluid dynamics~\citep{chung2010computational}, nonlinear acoustics~\citep{sugimoto1991burgers}, and traffic flows~\citep{nagel1996particle}. It serves as a benchmark test case for several numerical solvers and surrogate models~\citep{kutluay1999numerical,shah2017reduced,raissi2017physics} since it can generate discontinuities (shock waves) based on a normal conservation equation.
The viscous version of this equation is given by 
$$\frac{\partial u}{\partial t} + u \frac{\partial u}{\partial x} = v \frac{\partial^2 u}{\partial x^2},$$
 where $u$ indicates volume, $x$ represents a spatial location, $t$ indicates the time, and $v$ denotes the viscosity. We set $x\in[0,1]$\cmt{(in $\si{\meter}$)}, $t \in [0,3]$\cmt{ (in $\si{\second}$)}, and $u(x,0)=\sin(x\pi/2)$ with homogeneous Dirichlet boundary conditions.  We uniformly sampled viscosities $v \in [0.001,0.1]$\cmt{ (in $ \si[inter-unit-product = \ensuremath{{}\cdot{}}] {\milli\pascal\second}$)} as the input parameter to generate the solution field. 

In the space and time domains, the problem is solved using finite elements with hat functions and backward Euler, respectively. For the first (lowest-fidelity) solution, the spatial-temporal domain is discretized into $16\times16$ regular rectangular mesh. Higher-fidelity solvers double the number of nodes in each dimension of the mesh, \eg $32\times32$ for the second fidelity and $64\times64$ for the third fidelity.
The result fields (\ie outputs) are calculated using a $128by128$ regular spatial-temporal mesh.%

\noindent \textbf{Poisson's equation} is a typical elliptic PDE in mechanical engineering and physics for modeling potential fields, such as gravitational and electrostatic fields~\citep{chapra2010numerical}. Written as
\[
\frac{\partial^{2} u}{\partial x^{2}}+\frac{\partial^{2} u}{\partial y^{2}}=0.
\]
It is a generalization of Laplace's equation~\citep{persides1973laplace}.
Despite its simplicity, Poisson's equation is commonly encountered in physics and is regularly used as a fundamental test case for surrogate models~\citep{tuo2014surrogate,lagarisSept./1998artificial}. 
In our experiments, we impose Dirichlet boundary conditions on a 2D spatial domain with $\textbf{x} \in [0,1] \times [0,1]$. The input parameters consist of the constant values of the four borders and the center of the rectangular domain, which vary from $0.1$ to $0.9$ each. 
We sample the input parameters equally in order to create the matching potential fields as outputs. Using the finite difference approach with a first-order center differencing scheme and regular rectangular meshes, the PDE is solved. For the coarsest level solution, we utilized an $8\times8$ mesh. The improved solver employs a finer mesh with twice as many nodes in each dimension. The resultant potential fields are estimated using a spatial-temporal regular grid of $ 32 \times 32 $ cells.

\noindent \textbf{Heat equation} is a fundamental PDE that defines the time-dependent evolution of heat fluxes. Despite having been established in 1822 to describe just heat fluxes, the heat equation is prevalent in many scientific domains, including probability theory~\citep{spitzer1964electrostatic,burdzy2004heat} and financial mathematics~\citep{black1973pricing}. Consequently, it is commonly utilized as a stand-in model. 
This is the heat equation: 
$$
\frac{\partial}{\partial x}\left(k \frac{\partial T}{\partial x}\right)+\frac{\partial}{\partial y}\left(k \frac{\partial T}{\partial y}\right)+\frac{\partial}{\partial z}\left(k \frac{\partial T}{\partial z}\right)+q_{V}=\rho c_{p} \frac{\partial T}{\partial t}
$$
where
$k$ is the materials conductivity
$q_{V}$ is the rate at which energy is generated per unit volume of the medium
$\rho$ is the density
and
$c_{p}$ is the specific heat capacity.
The input parameters are the flux rate of the left boundary at $x=0$ (ranging from 0 to 1), the flux rate of the right boundary at $x=1$ (ranging from  -1 to 0), and the thermal conductivity (ranging from 0.01 to 0.1).

We establish a 2D spatial-temporal domain $x\in[0,1]$, $t \in [0,5]$ with the Neumann boundary condition at$x=0$ and $x=1$, and $u(x,0)=H(x-0.25)-H(x-0.75)$, where $H(\cdot)$ is the Heaviside step function.

The equation is solved using the finite difference in space and backward Euler in time domains. The spatial-temporal domain is discretized into a $16\times16$ regular rectangular mesh for the first (lowest) fidelity solver. A refined solver uses a $32\times32$ mesh for the second fidelity. The result fields are computed on a $100\times100$ spatial-temporal grid.

The equation is solved using a finite difference in the spatial domain and reverse Euler in the temporal domain. The spatial-temporal domain is discretized into an $ 8\times8 $ regular rectangular mesh for the first (least accurate) solution. The second fidelity of an improved solver's mesh is a $ 32\times32 $ grid. On a $ 100\times100 $ spatial-temporal grid, the result fields are calculated.
\subsection{Multi-Fidelity Fusion for Canonical PDEs}
\label{appe: PDEs}
\begin{figure}[]
	\centering
	\begin{subfigure}[b]{0.32\linewidth}
		\includegraphics[width=1\textwidth]{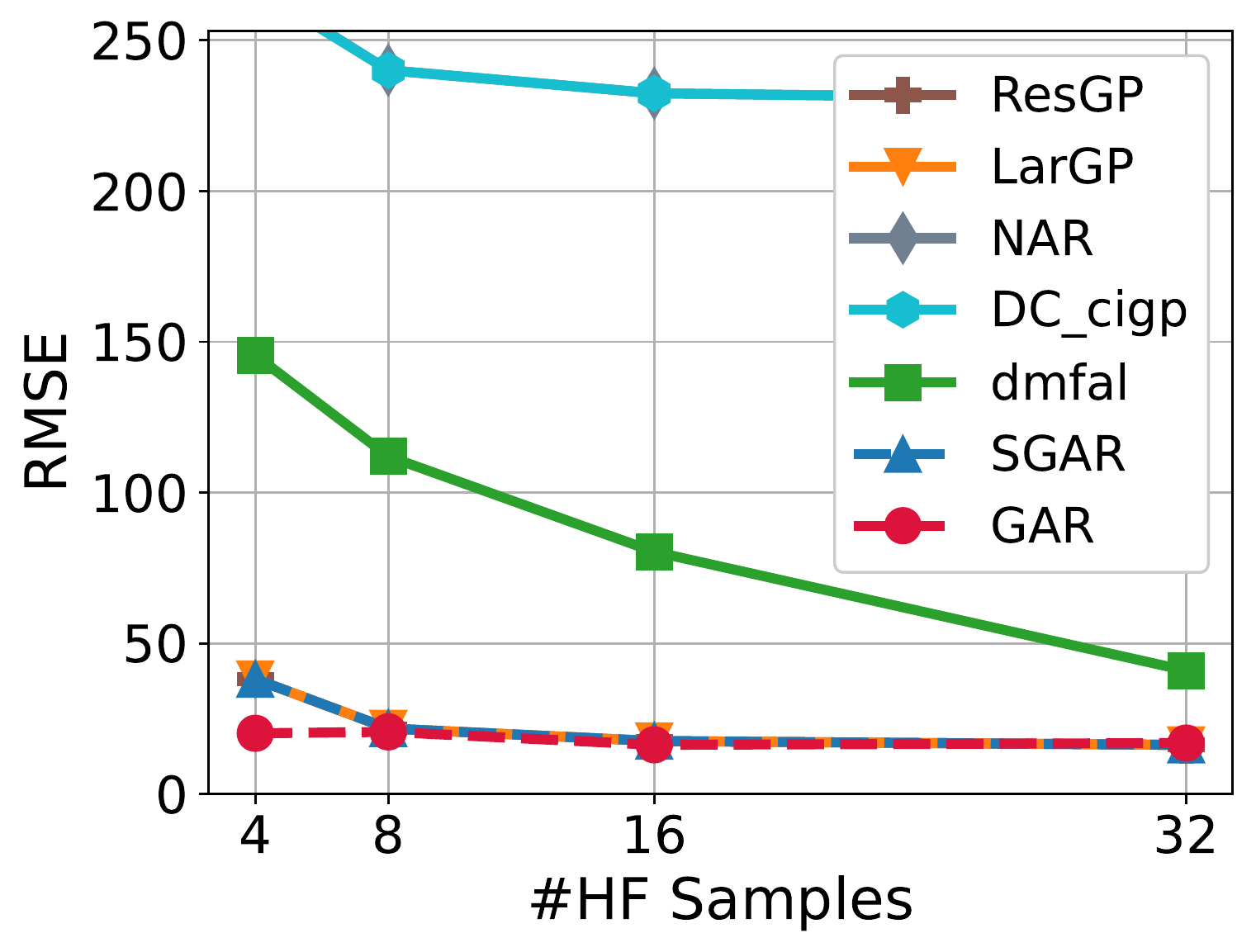}
		\caption{Poisson's}
	\end{subfigure}
	\begin{subfigure}[b]{0.32\linewidth}
		\includegraphics[width=1\textwidth]{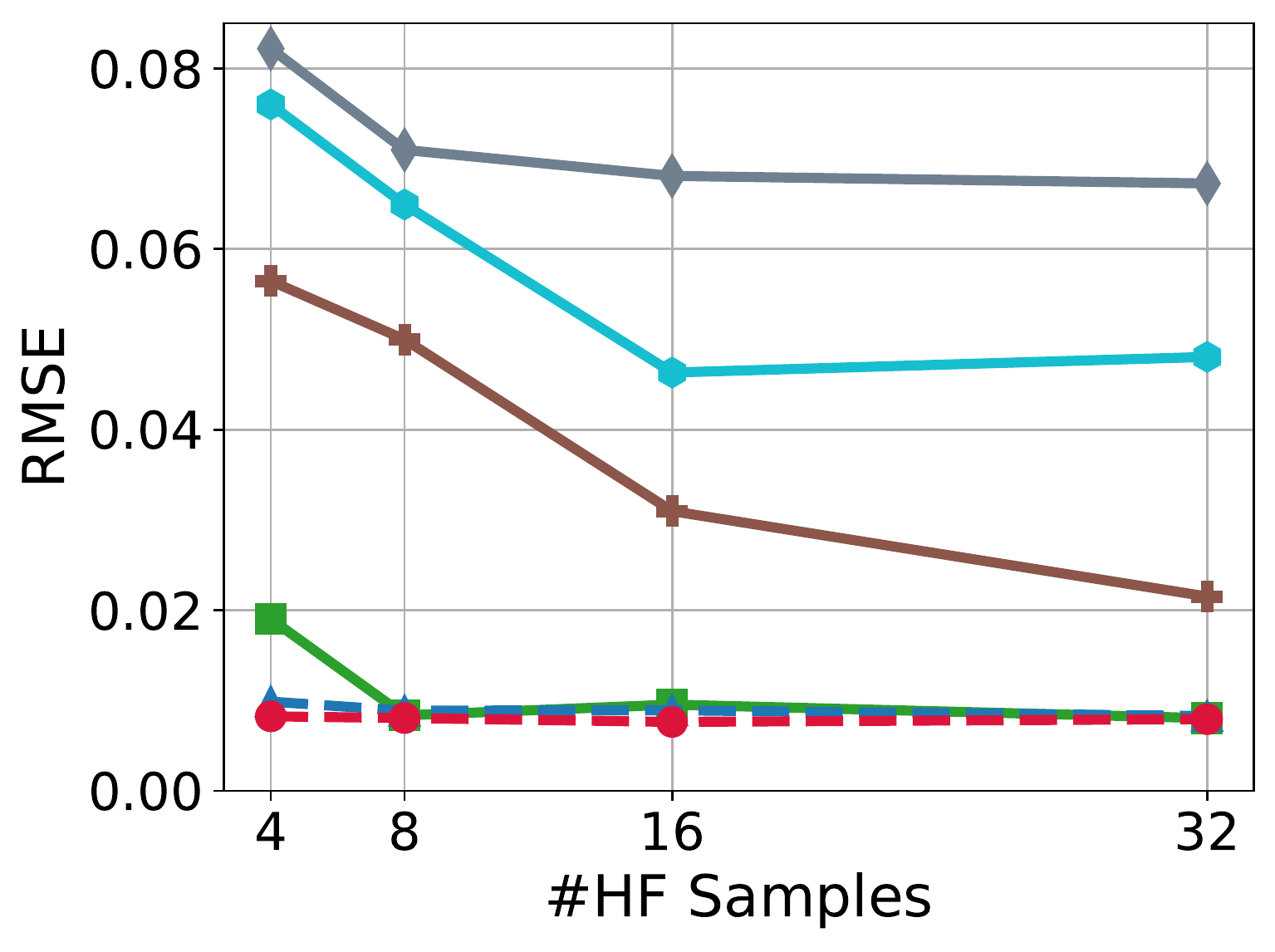}
		\caption{Burger's}
	\end{subfigure}
    \begin{subfigure}[b]{0.32\linewidth}
		\centering
		\includegraphics[width=1\textwidth]{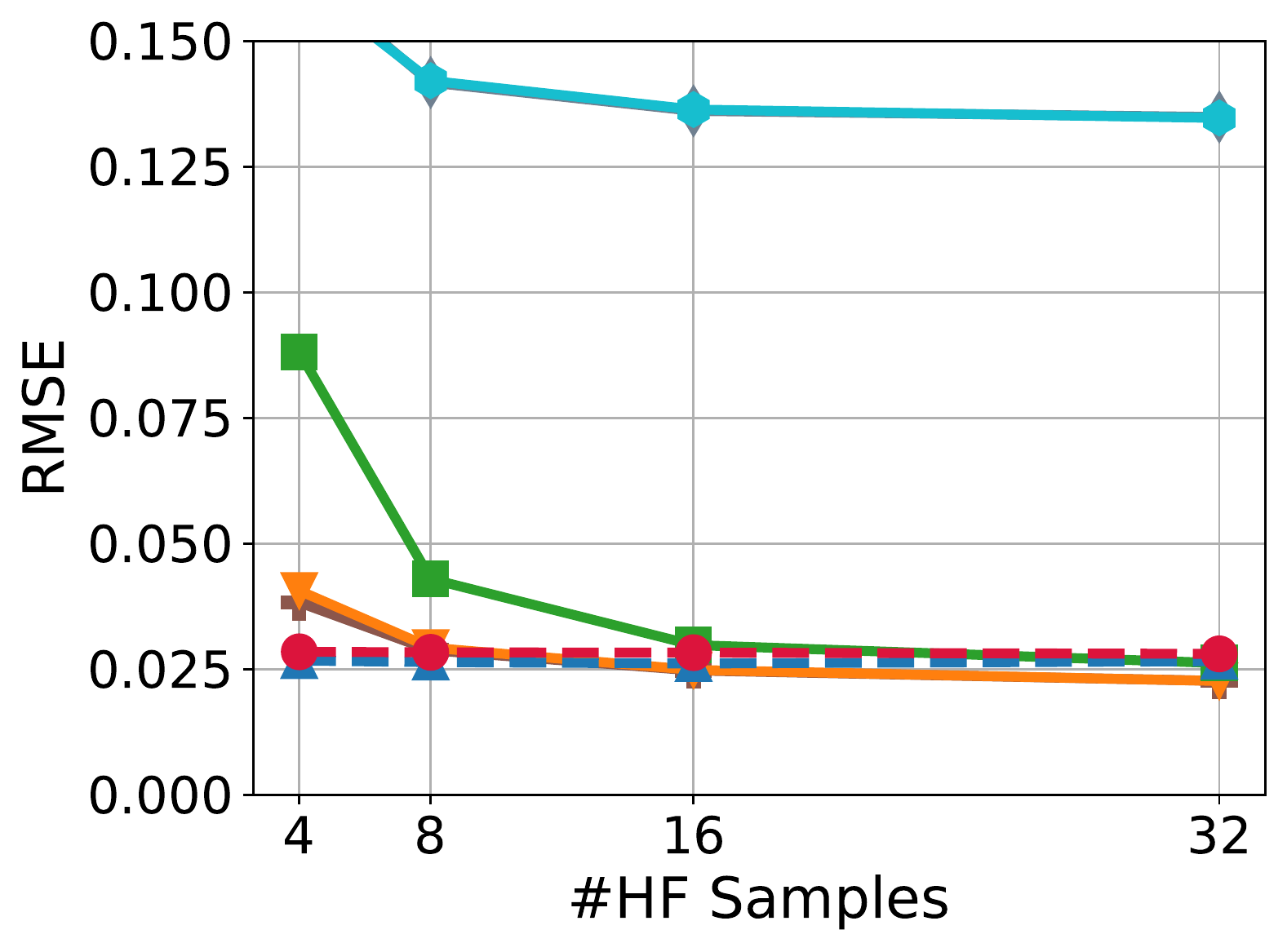}
		\caption{Heat's}
	\end{subfigure}
	\caption{RMSE against increasing number of high-fidelity training samples with training samples increased using Sobol sequence and aligned (interpolated) outputs.}
	\label{fig3}
\end{figure}

\begin{figure}[]
	\centering
	\begin{subfigure}[b]{0.32\linewidth}
		\includegraphics[width=1\textwidth]{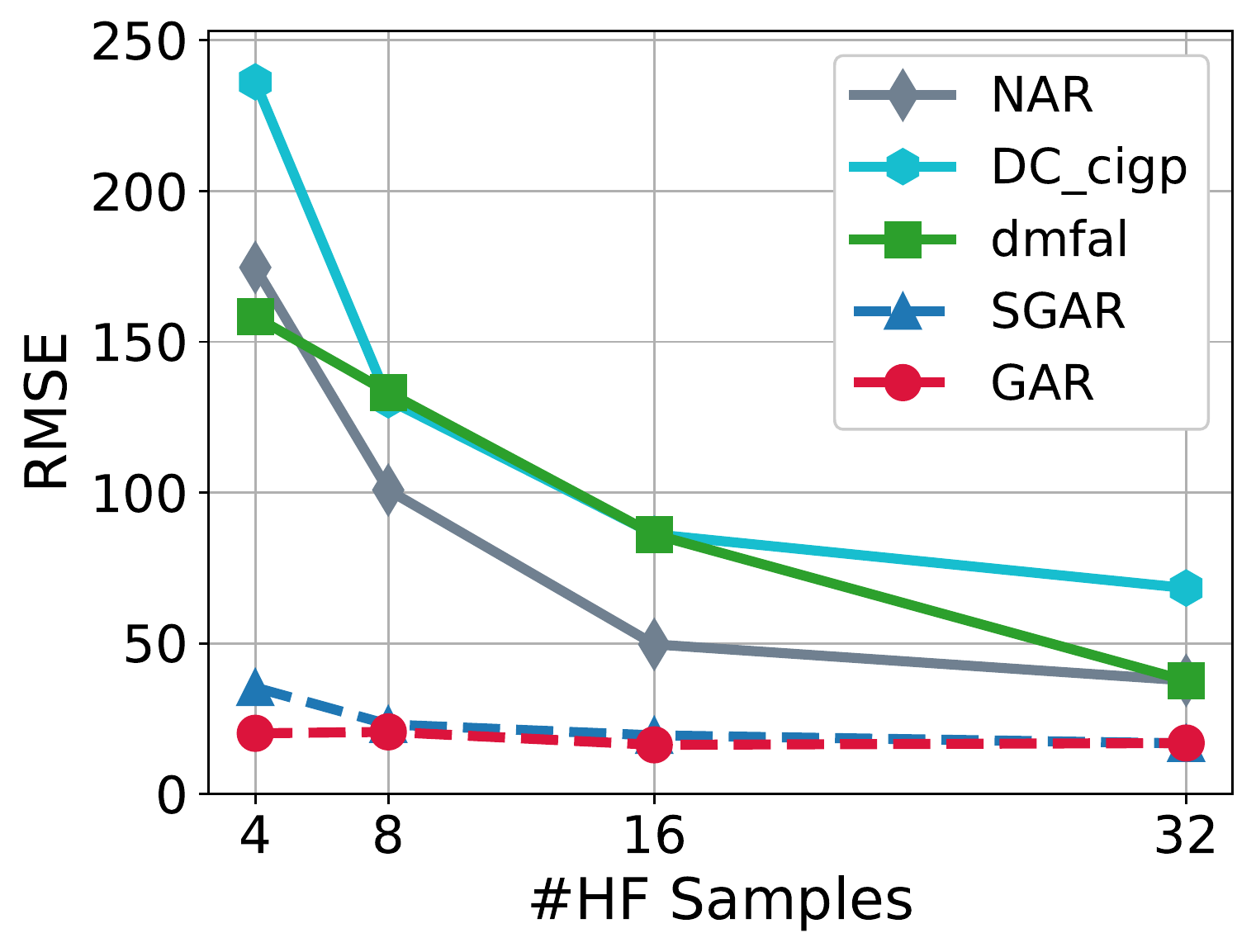}
		\caption{Poisson's}
	\end{subfigure}
	\begin{subfigure}[b]{0.32\linewidth}
		\includegraphics[width=1\textwidth]{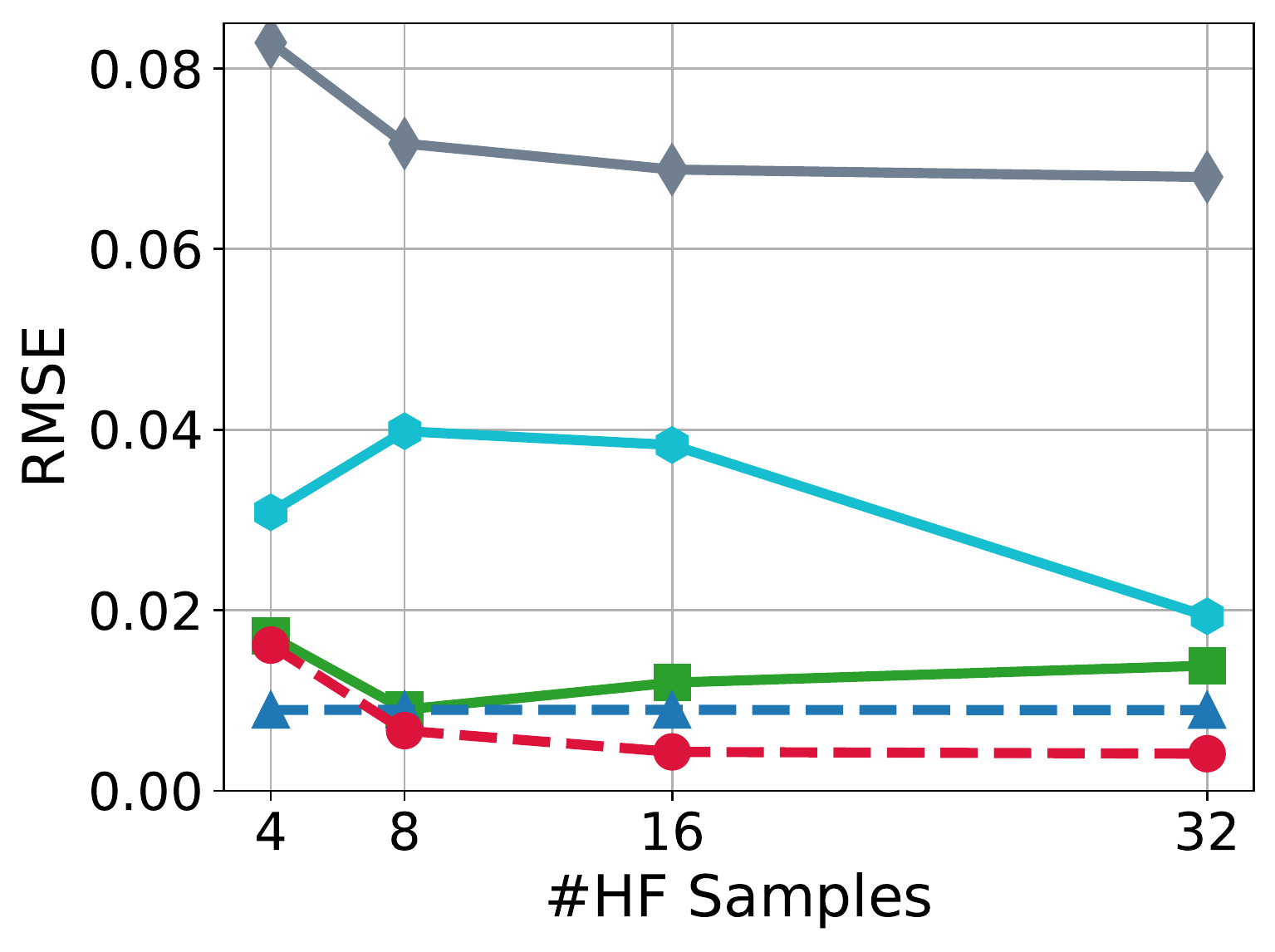}
		\caption{Burger's}
	\end{subfigure}
	\begin{subfigure}[b]{0.32\linewidth}
		\includegraphics[width=1\textwidth]{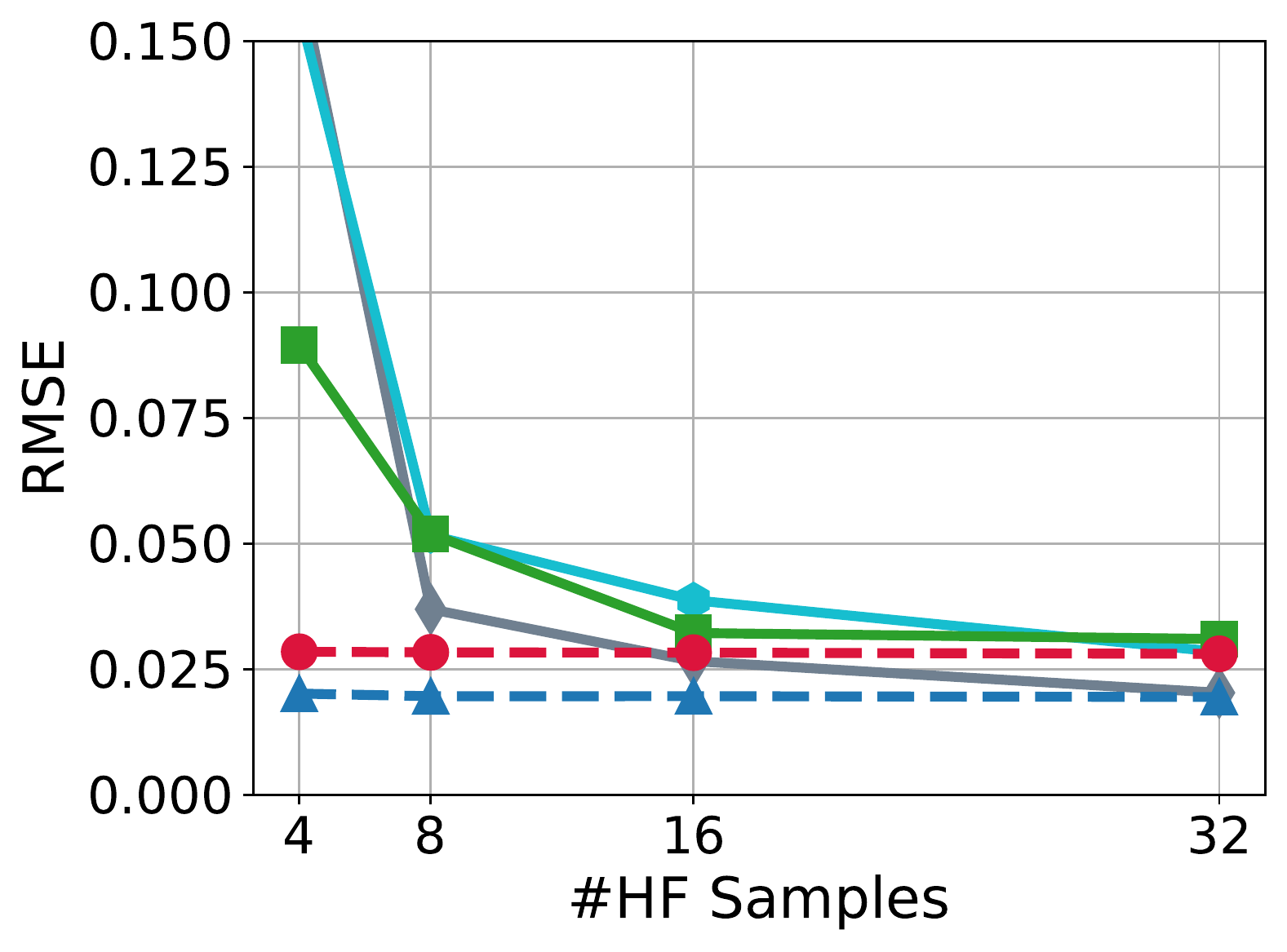}
		\caption{Heat's}
	\end{subfigure}
	\caption{RMSE against increasing number of high-fidelity training samples with training samples increased using Sobol sequence and unaligned  outputs.}
	\label{fig4}
\end{figure}
We use the same experimental setup as in Section \ref{pdes} for these experiments with the only difference being that the training data is generated using a Sobol sequence.
We generated 256 data samples for testing and 32 samples for training. We increased the number of high-fidelity training data gradually from 4 to 32 with the high-fidelity training data fixed to 32.
\Figref{fig3} and \Figref{fig4} show the RMSE statistical results for aligned outputs using interpolated and original unaligned outputs.
\ours and \ourss outperform the competitors with a large margin with scarce high-fidelity training data as in the main paper.
Similarly, the advantage of \ours and \ourss are more obvious when dealing with non-aligned outputs, where \ours and \ourss demonstrate a 5x reduction in RMSE with 4 and 8 high-fidelity training samples, surpassing the competitors by a wide margin.

\subsection{Multi-Fidelity Fusion for Topology Optimization}
\label{appe: sec exp topop}
We use GAR in a topology structure optimization problem, where the output is the best topology structure (in terms of maximum mechanical metrics like stiffness) of a layout of materials, such as alloy and concrete, given some design parameters like external force and angle.
Topology structure optimization is a significant approach in mechanical designs, such as airfoils and slab bridges, especially with recent 3D printing processes in which material is deposited in minute quantities. 
However, it is well known that topology optimization is computationally intensive due to the gradient-based optimization and simulations of the mechanical characteristics involved. A high-fidelity solution, which necessitates a huge discretization mesh and imposes a significant computing overhead in space and time, makes matters worse.

Utilizing data-driven ways to aid in the process by offering the appropriate structures~\citep{xing2020shared,li2020deep} is subsequently gaining popularity. 
Here, we investigate the topology optimization of a cantilever beam (shown in the appendix).  We employ the rapid implementation~\cite{Andreassen2011} to carry out density-based topology optimization by reducing compliance $C$ subject to volume limitations $V \leq \bar{V}$.

The SIMP scheme~\cite{bendsoe_topology_2004} is used to convert continuous density measurements to discrete, optimal topologies. 
We set the position of point load $P 1$, the angle of point load $P 2$, and the filter radius $P 3$~\cite{BRUNS20013443} as system input. 
We solve this challenge for low-fidelity with a $40\times80$ regular mesh and high-fidelity with a $40\times80$ regular mesh. This experiment only includes techniques that can process arbitrary outputs.

\begin{figure}
	\centering
    \includegraphics[width=0.5\linewidth]{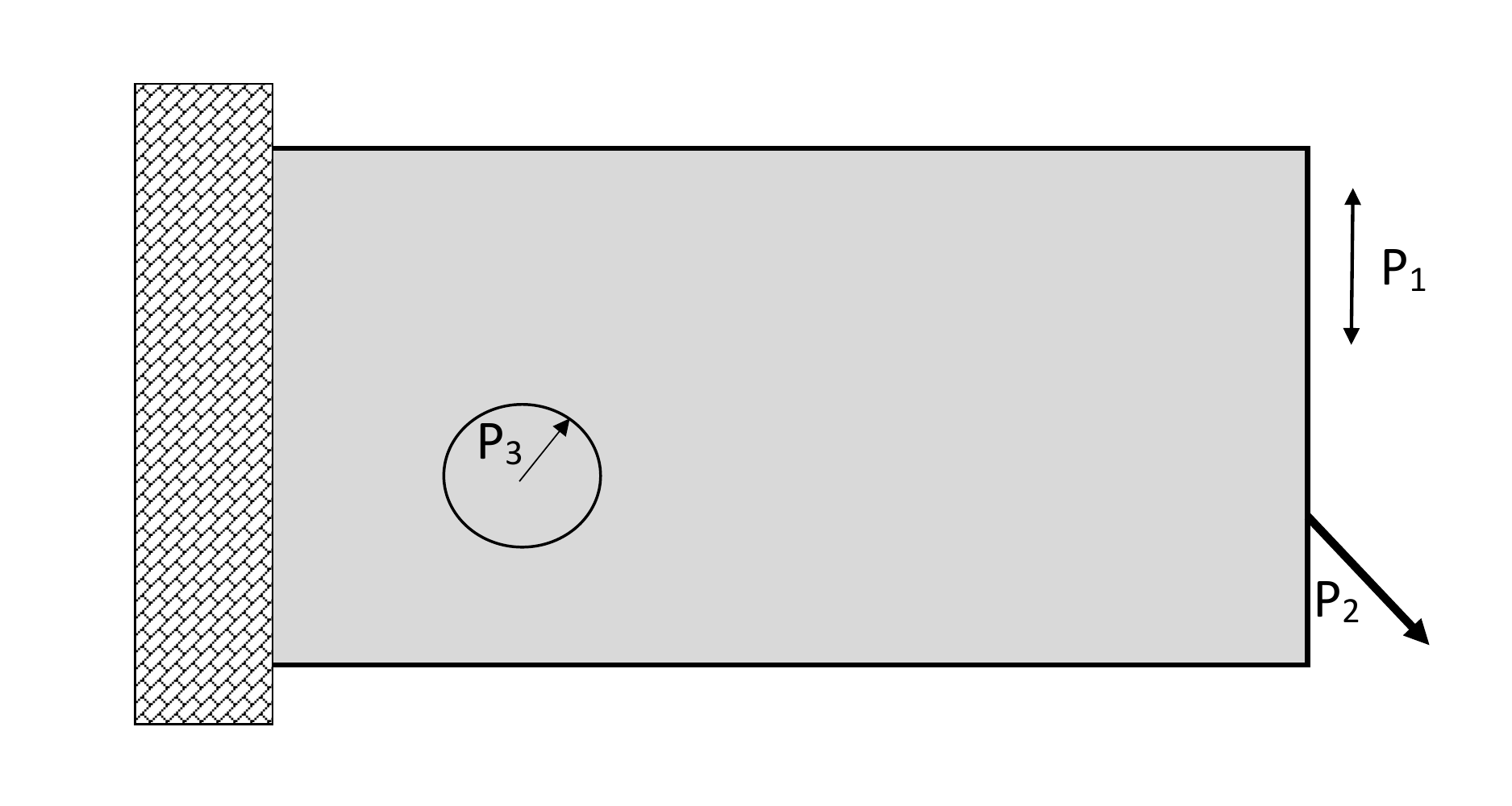}
	\caption{Geometry, boundary conditions, and simulation parameters for cantilever beam}
	\label{fig:topo_geom}
\end{figure}

\begin{figure}[h]
	\centering
	\begin{subfigure}[b]{0.42\linewidth}
		\includegraphics[width=1\textwidth]{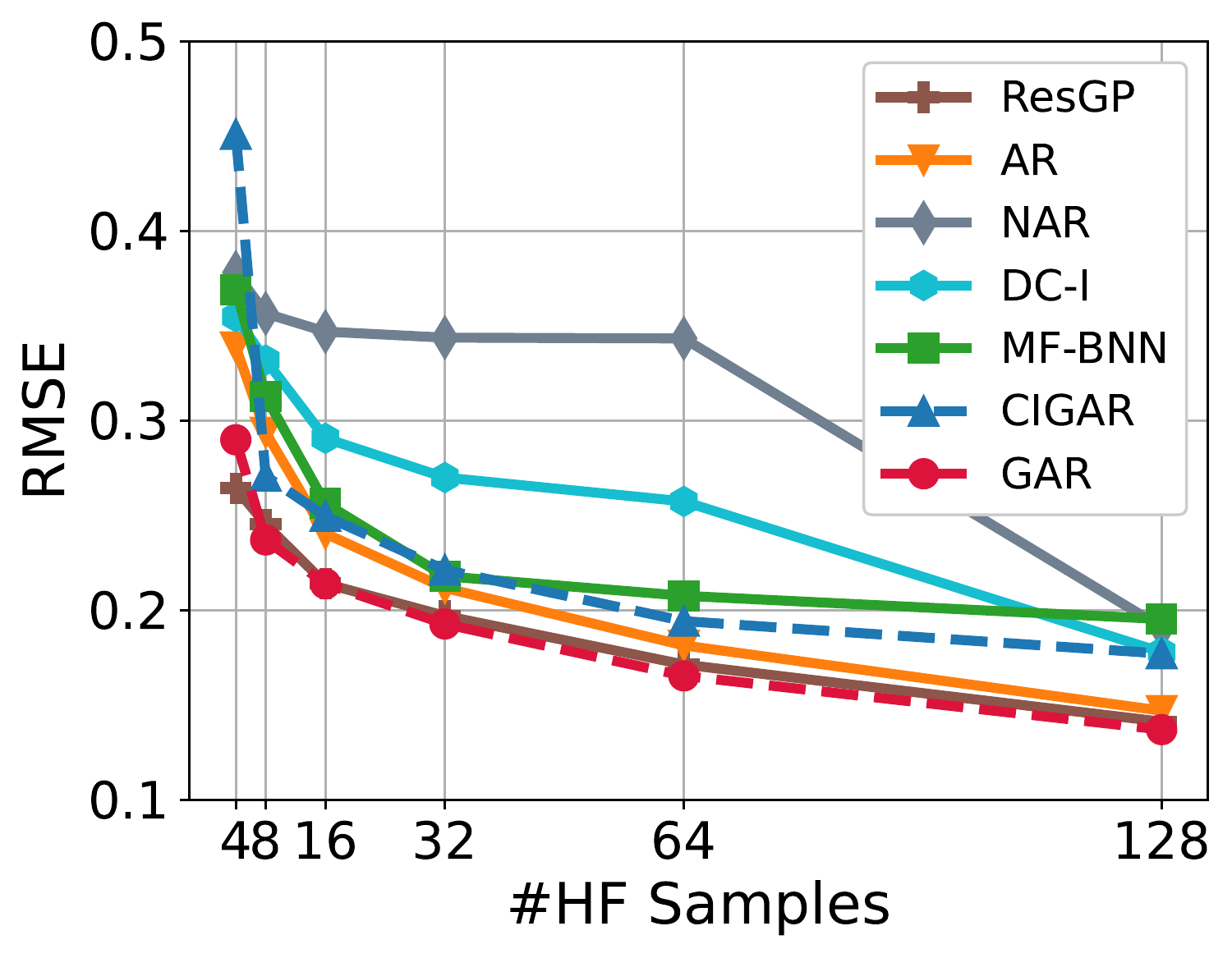}
		\caption{Aligned outputs}
		\label{Top 128 interp}
	\end{subfigure}
	\begin{subfigure}[b]{0.42\linewidth}
		\includegraphics[width=1\textwidth]{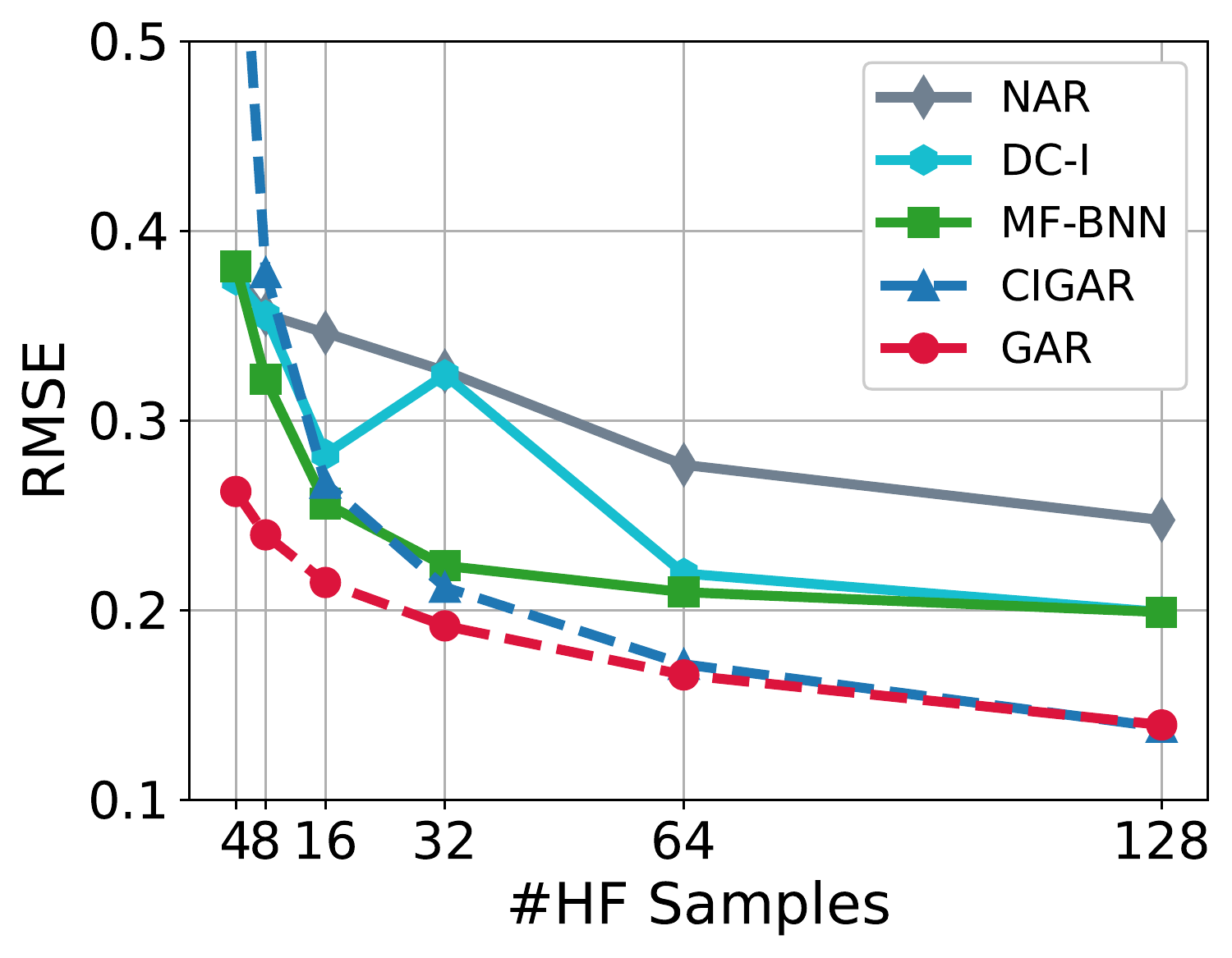}
		\caption{{Raw outputs}}
		\label{Top 128 raw}
	\end{subfigure}
	\caption{RMSE against increasing number of high-fidelity training samples for topology optimization using Sobol sequence.}
	\label{fig: topop}
\end{figure}

As with the early experiments, we generate 128 testing samples and 64 training samples using a Sobol sequence to approximately assess the ance in active learning. 
The results are shown in Figure \ref{fig: topop}.
We can see that all available methods show similar performance for both raw outputs that are not aligned by interpolation and the aligned outputs.
Nevertheless, \ours consistently outperforms the competitors with a clear margin.
\ourss, in contrast, performs better for the raw outputs.

\subsection{Multi-Fidelity Fusion for Solid Oxide Fuel Cell}
\label{appe:sofc}

\begin{figure}[h!]
	\centering
	\includegraphics[width=0.65\textwidth]{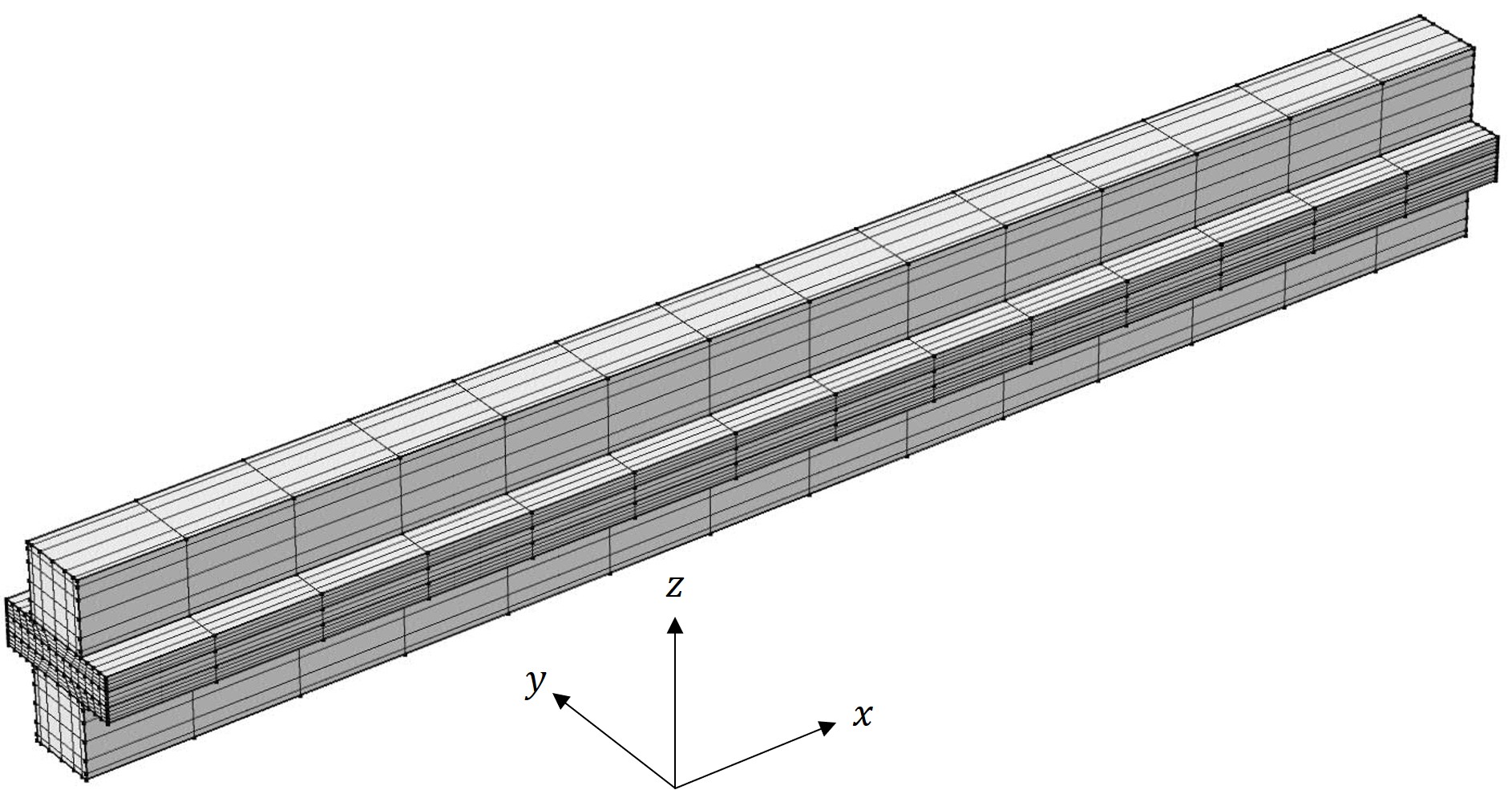}
	\caption{The cathode is at the top of the computational domain for the SOFC example, which consists of gas channels, electrodes, and electrolyte. The layers are, from top to bottom, a channel, an electrode, an electrolyte, an electrode, and a channel. The dimensions of the channel are (x * y * z) 1 cm * 0.5 mm * 0.5 mm, the dimensions of the electrode are 1 cm * 1 mm * 0.1 mm, and the dimensions of the electrolyte are 1 cm * 1 mm * 0.1 mm. The cathode intake is placed at $ x=1 $ cm while the anode inlet is located at $ x=0 $ cm.}
	\label{fig:sofc} 
\end{figure}

In this test problem, a steady-state 3-D solid oxide fuel cell model is considered. Fig.\ref{fig:sofc} illustrates the geometry. 
The model incorporates electronic and ionic charge balances (Ohm's law), flow distribution in gas channels (Navier-Stokes equations), flow in porous electrodes (Brinkman equation), and gas-phase mass balances in both gas channels and porous electrodes (Maxwell-Stefan diffusion and convection). Butler-Volmer charge transfer kinetics is assumed for reactions in the anode
 ($\mbox{H}_2+\mbox{O}^{2-}\rightarrow \mbox{H}_2\mbox{O}+2\mbox{e}^{-}$)
and cathode ($\mbox{O}_2+4\mbox{e}^{-}\rightarrow 2\mbox{O}^{2-}$).
The cell functions in a potentiostat manner (constant cell voltage). COMSOL Multiphysics\footnote{\url{https://www.comsol.com/model/current-density-distribution-in-a-solid-oxide-fuel-cell-514}} (Application ID: 514), which uses the finite-element approach, was used to solve the model.
 
The assumed inputs are the electrode porosities $\epsilon\in[0.4,0.85]$, the cell voltage $E_c\in[0.2,0.85]$ V, the temperature $T\in[973,1273]$ K, and the channel pressure $P\in[0.5,2.5]$ atm. A Sobol sequence is used to choose 60 inputs within the ranges specified for the low-fidelity and high-fidelity simulations. 40 high-fidelity test points are chosen at random (from the ranges above) to complete the test. The low-fidelity F1 model used 3164 mapped elements and relative tolerance of 0.1, while the high-fidelity model employed 37064 elements and relative tolerance of 0.001. Additionally, the COMSOL model employs a V cycle geometric multigrid. The quantities of interest are profiles of electrolyte current density (A m$-2$) and ionic potential (V) in the $x-z$ plane centered on the channels (Fig. \ref{fig:sofc}). In both instances, $d=100\times50=5000$ points are captured, and both profiles are vectorized to provide the training and test outputs.
\begin{figure}[h]
	\centering
	\begin{subfigure}[b]{0.32\linewidth}
			\includegraphics[width=1\textwidth]{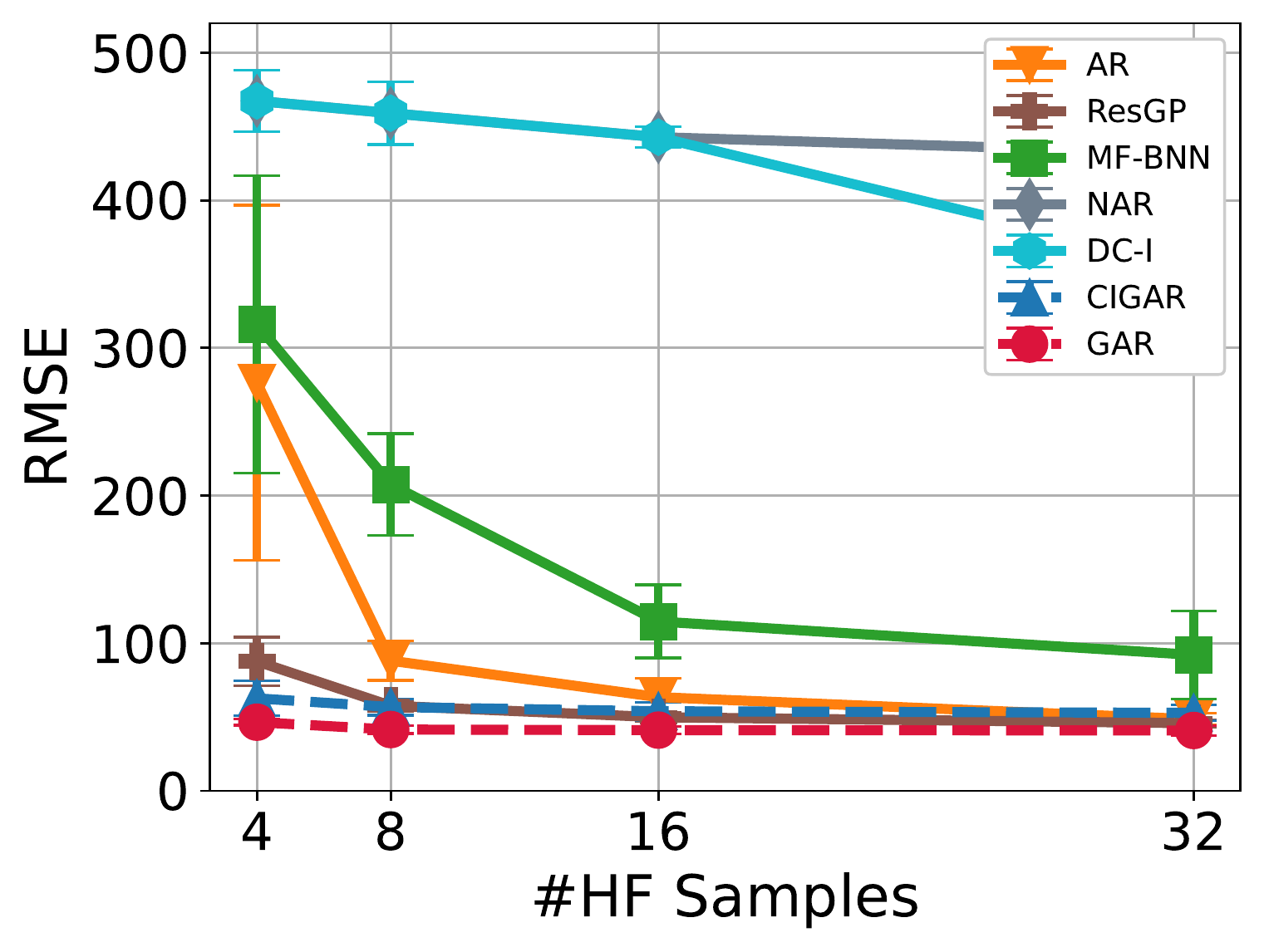}
		\end{subfigure}
	\begin{subfigure}[b]{0.32\linewidth}
			\includegraphics[width=1\textwidth]{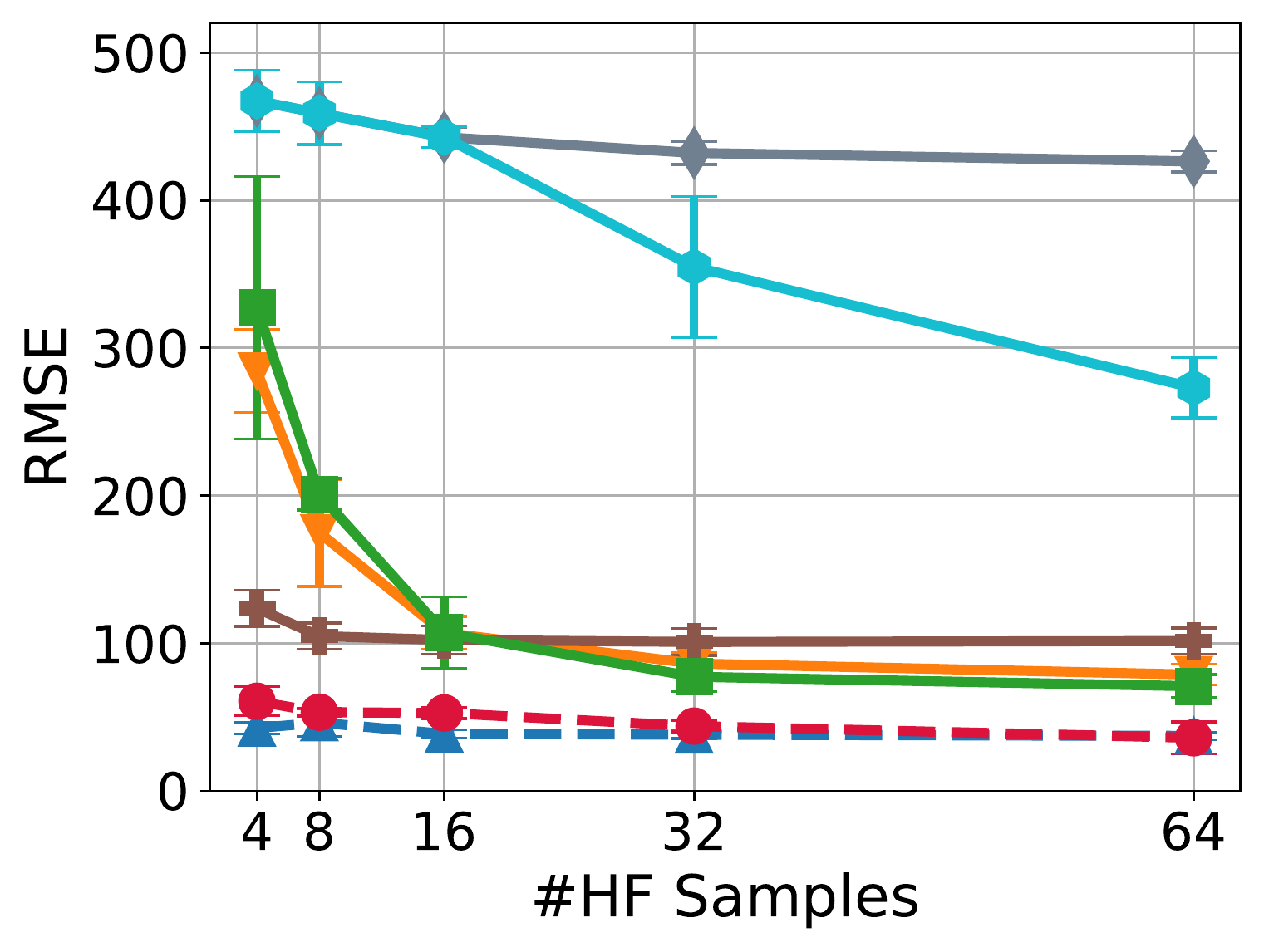}
		\end{subfigure}
	\begin{subfigure}[b]{0.32\linewidth}
			\includegraphics[width=1\textwidth]{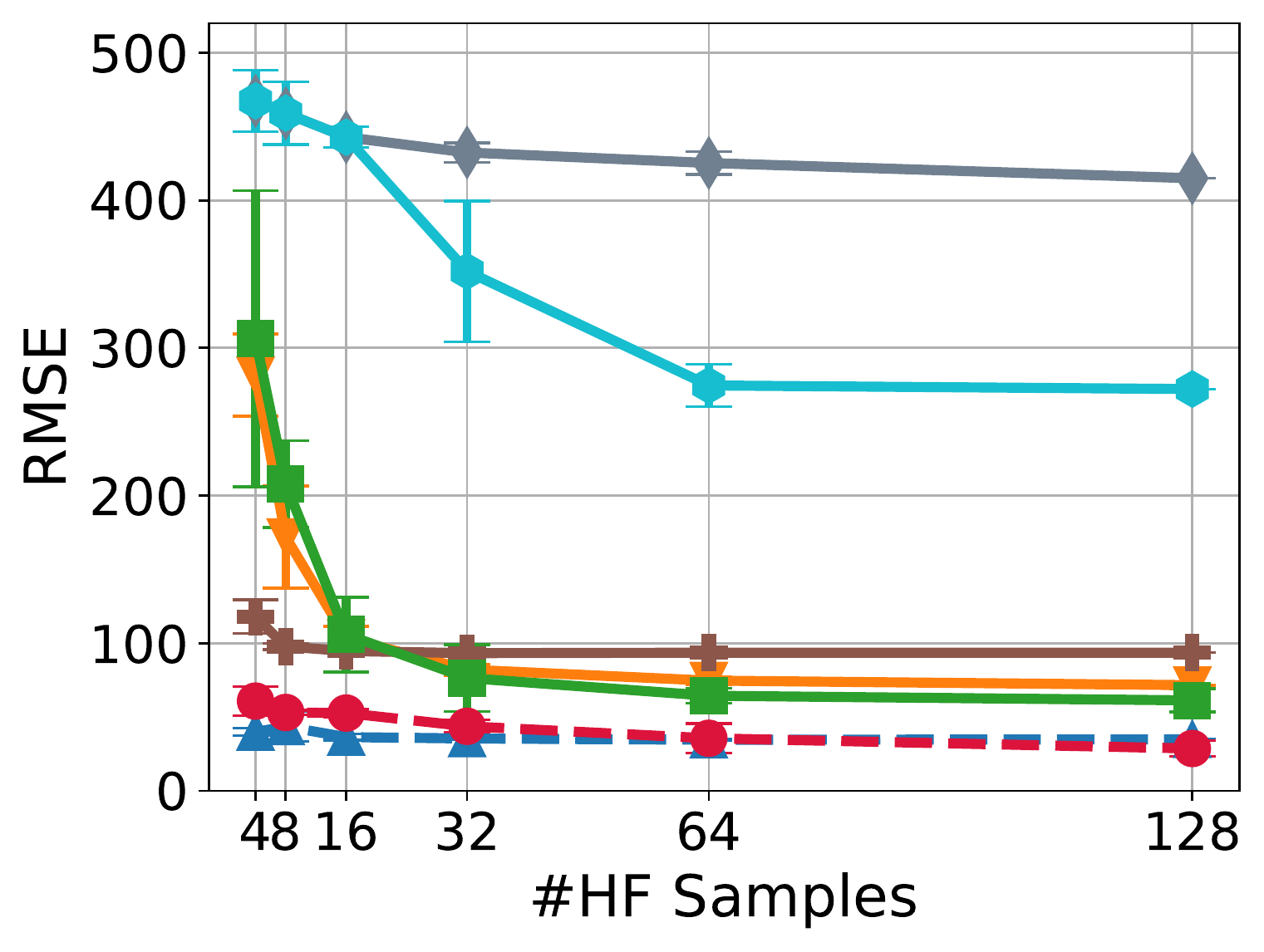}
		\end{subfigure}
	\caption{RMSE against increasing number of high-fidelity training samples for SOFC with low-fidelity training sample number fixed to \{32,64,128\}.}
	\label{fig: topop rmse2}
\end{figure}

We add the classic experiment where the number of low-fidelity training samples was fixed to \{32,64,128\} and the high-fidelity training samples are gradually increased from 4 to \{32,64,128\}.
The outputs are aligned using interpolation, and the experimental results are shown in \Figref{fig: topop rmse2}.
We can see that the GAR and CIGAR methods always perform better than the other methods, especially when only a few high-fidelity training data are used. This is consistent with the previous experiment.
We can also see that AR also performs well indicating that these data are not highly nonlinear and complex, making it relatively easy to solve.
However, both AR and MF-BNN converge to a higher error whereas \ours and \ourss converge to a lower error bound.

\begin{figure}[h]
	\centering
	\begin{subfigure}[b]{0.32\linewidth}
		\includegraphics[width=1\textwidth]{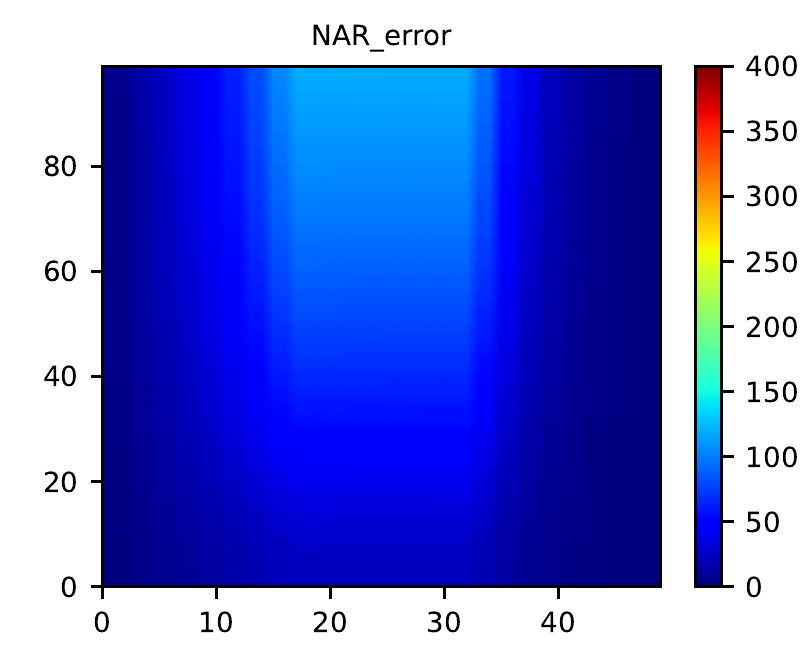}
		\caption{NAR}
	\end{subfigure}
	\begin{subfigure}[b]{0.32\linewidth}
		\includegraphics[width=1\textwidth]{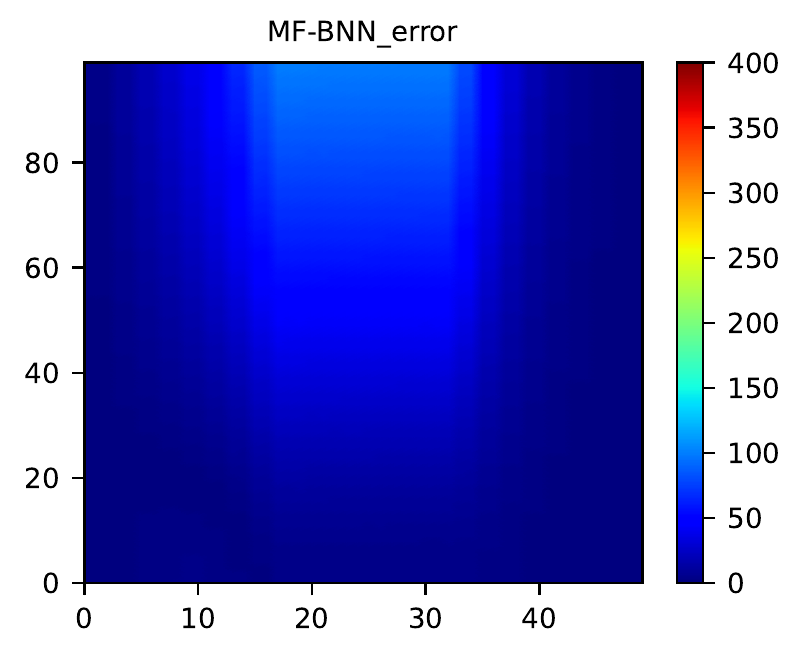}
		\caption{MF-BNN}
	\end{subfigure}
	\begin{subfigure}[b]{0.32\linewidth}
		\includegraphics[width=1\textwidth]{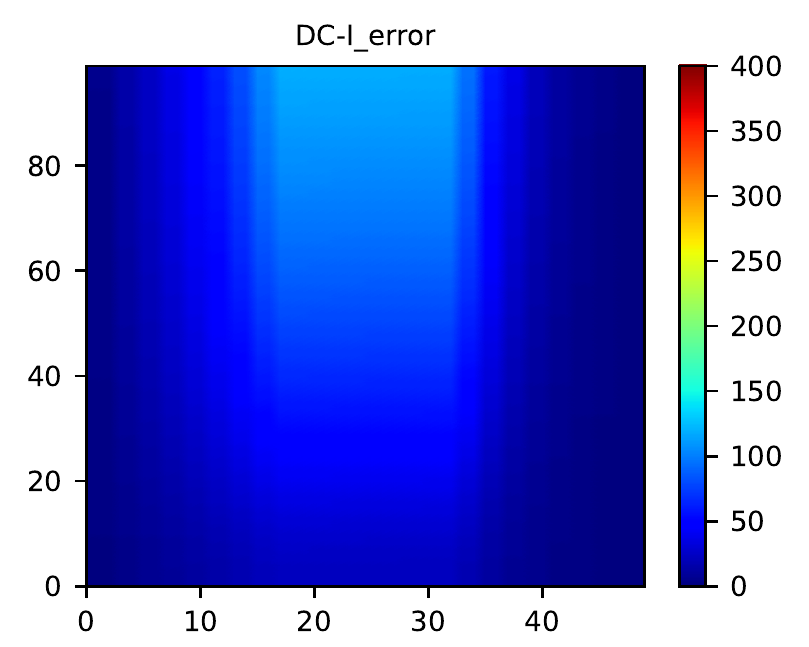}
		\caption{DC}
	\end{subfigure}
	\caption{RMSE fields of ECD for 128 testing samples, using 32 low-fidelity and 16 high-fidelity training samples.}
	\label{fig: 16 Y1 SOFC error}
\end{figure}
\begin{figure}[h]
	\centering
	\begin{subfigure}[b]{0.32\linewidth}
		\includegraphics[width=1\textwidth]{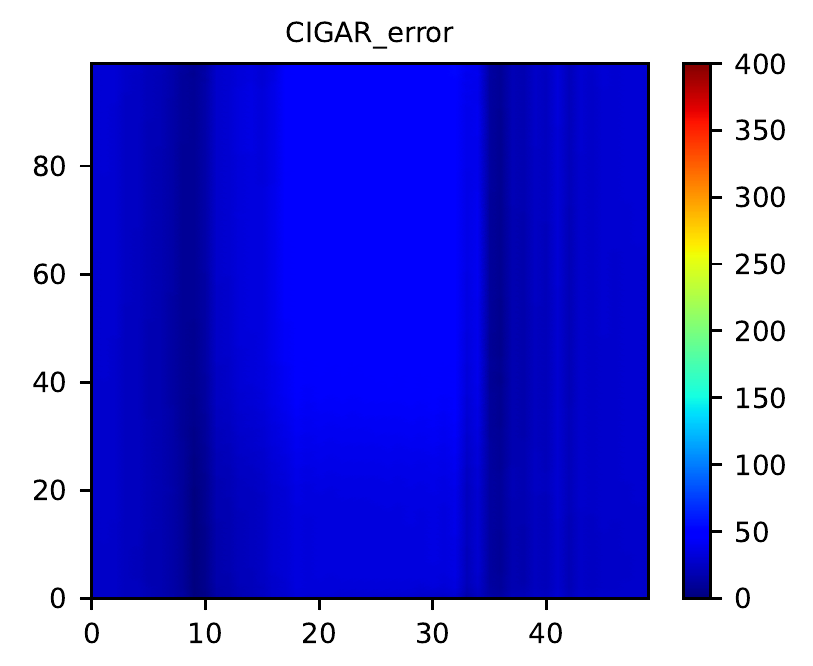}
		\caption{CIGAR}
	\end{subfigure}
	\begin{subfigure}[b]{0.32\linewidth}
		\includegraphics[width=1\textwidth]{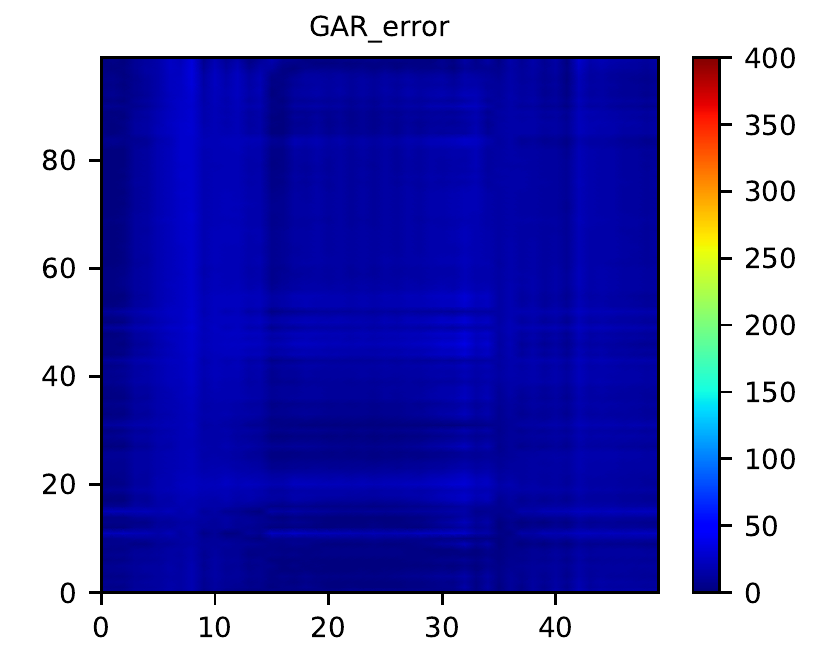}
		\caption{GAR}
	\end{subfigure}
	\caption{RMSE fields of ECD for 128 testing samples, using 32 low-fidelity and {\bf 4} high-fidelity training samples.}
	\label{fig: 4 Y1 SOFC error}
\end{figure}

To investigate the prediction error in detail, we define the average RMSE field $\tZ^{(\mathrm{AEF})}$ by 
\[
	\tZ^{(\mathrm{AEF})} = \sqrt{\frac{1}{N}\sum_{i=1}^N (\tZ_i - \tilde{\tZ}_{i})^2},
\]
where $\tilde{\tZ}_{i}$ is the prediction, $\tZ_i$ is the ground true value, and the square root is element-wise operation.
\Figref{fig: 16 Y1 SOFC error} shows the average RMSE field of NAR, MF-BNN, and DC methods on the ECD in SOFC data with 32 low-fidelity training samples, 16 high-fidelity training samples, and 128 test samples.
To highlight the advantage of \ours and \ourss,
\Figref{fig: 4 Y1 SOFC error} shows the average RMSE field of the same setup with \textbf{only 4 high-fidelity training samples}.
It can be seen clearly that \ours and \ourss have a smaller error field even with only 4 high-fidelity training samples compared to NAR, MF-BNN, and DC with 16 high-fidelity training samples. 
Also note that \ours seems to have some checkerboard artifacts, which might be caused by the over-parameterization using a full transfer matrix. We leave this issue to our further work to resolve.
\ourss have fewer checkerboard artifacts with the price of a slight increase in the RMSE.

\begin{figure}[h]
	\begin{subfigure}[b]{0.32\linewidth}
		\includegraphics[width=1\textwidth]{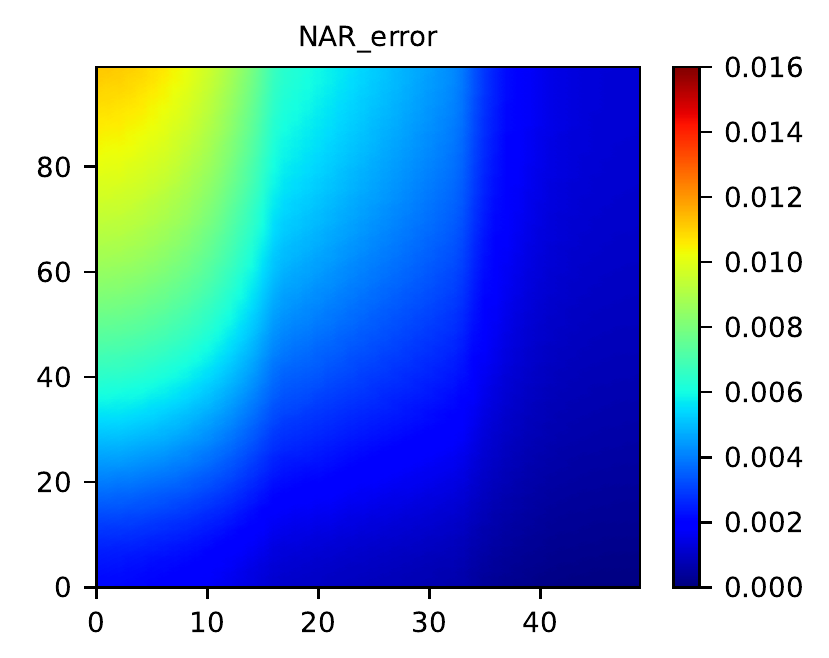}
		\caption{NAR}
	\end{subfigure}
	\begin{subfigure}[b]{0.32\linewidth}
		\includegraphics[width=1\textwidth]{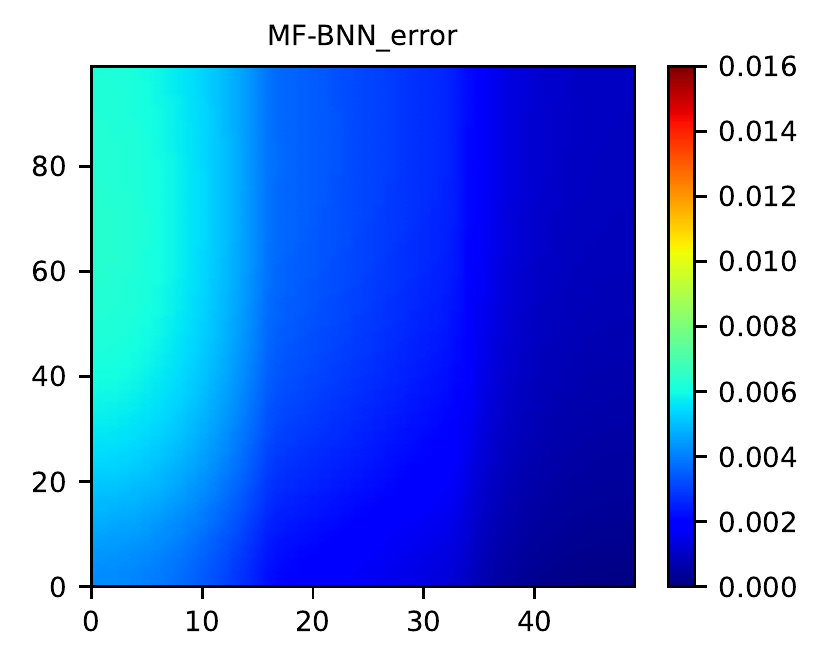}
		\caption{MF-BNN}
	\end{subfigure}
	\begin{subfigure}[b]{0.32\linewidth}
		\includegraphics[width=1\textwidth]{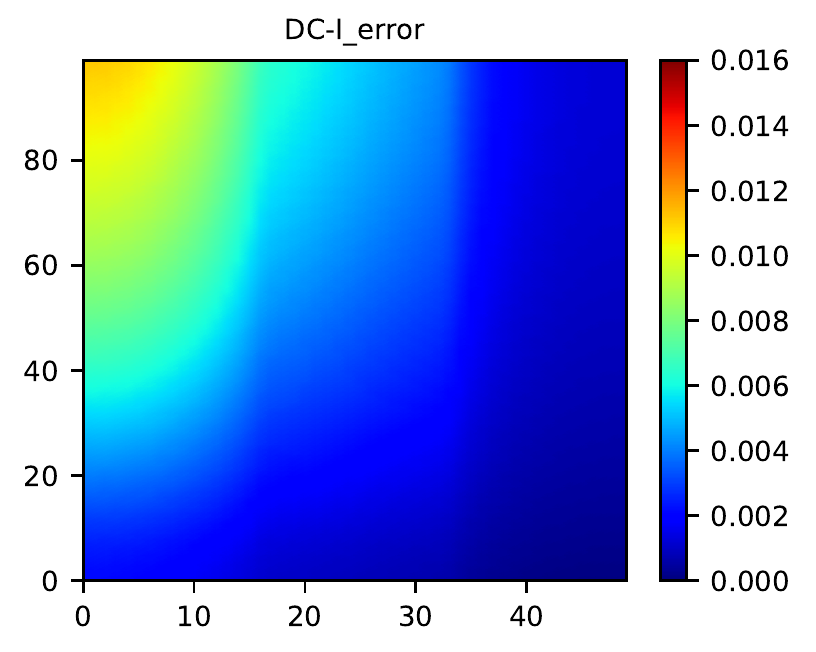}
		\caption{DC}
	\end{subfigure}
	\caption{RMSE fields of IP for 128 testing samples, using 32 low-fidelity and 16 high-fidelity training samples.}
	\label{fig: 16 Y2 SOFC error}
\end{figure}

\begin{figure}[h]
	\centering
	\begin{subfigure}[b]{0.32\linewidth}
		\includegraphics[width=1\textwidth]{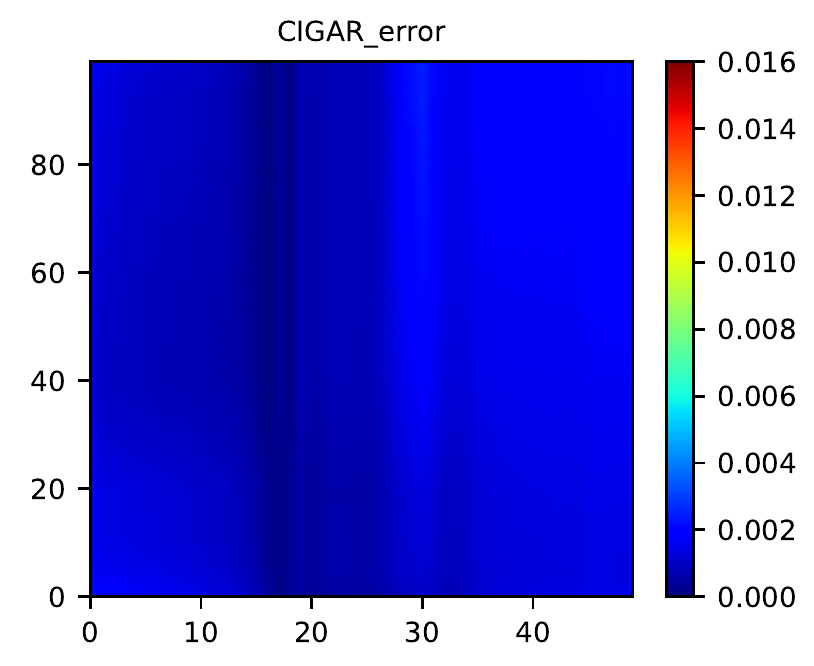}
		\caption{CIGAR}
	\end{subfigure}
	\begin{subfigure}[b]{0.32\linewidth}
		\includegraphics[width=1\textwidth]{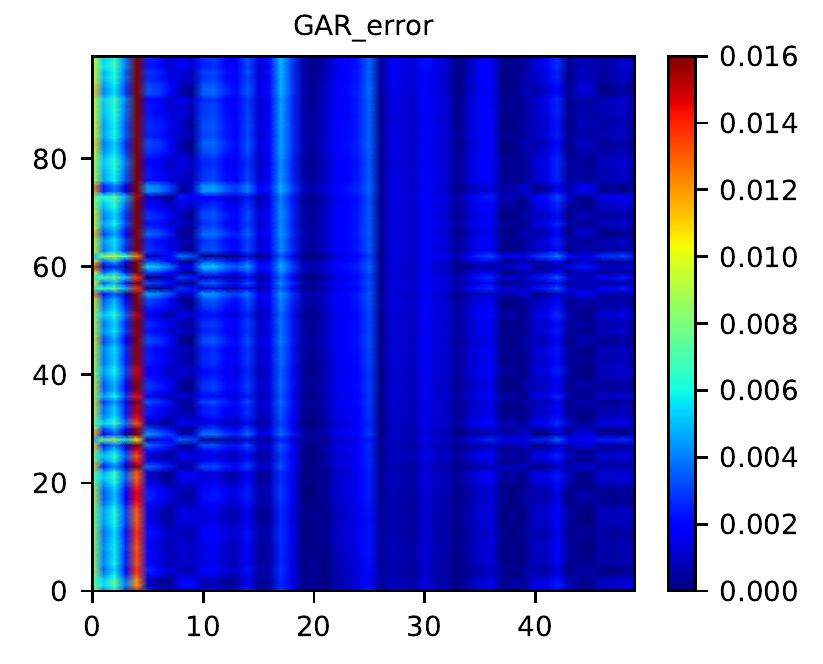}
		\caption{GAR}
	\end{subfigure}
	\caption{RMSE fields of IP for 128 testing samples, using 32 low-fidelity and {\bf 4} high-fidelity training samples.}
	\label{fig: 4 Y2 SOFC error}
\end{figure}
In \Figref{fig: 16 Y2 SOFC error} and \Figref{fig: 4 Y2 SOFC error}, similar to the previous experimental setup, we draw the average RMSE with 128 testing samples on the IP fields from the SOFC dataset. The NAR, MF-BNN and DC are trained with 16 high-fidelity samples, while \ours and \ourss are trained with only 4 high-fidelity samples. We can see that our methods outperform other methods by a clear margin.
However, the checkerboard artifact is even worse for \ours in this case, whereas \ourss successfully reduces such an artifact with also low error.

\begin{figure}[!htbp]
	\centering
	\includegraphics[width=0.49\textwidth,trim={8cm 0 8cm 0},clip]{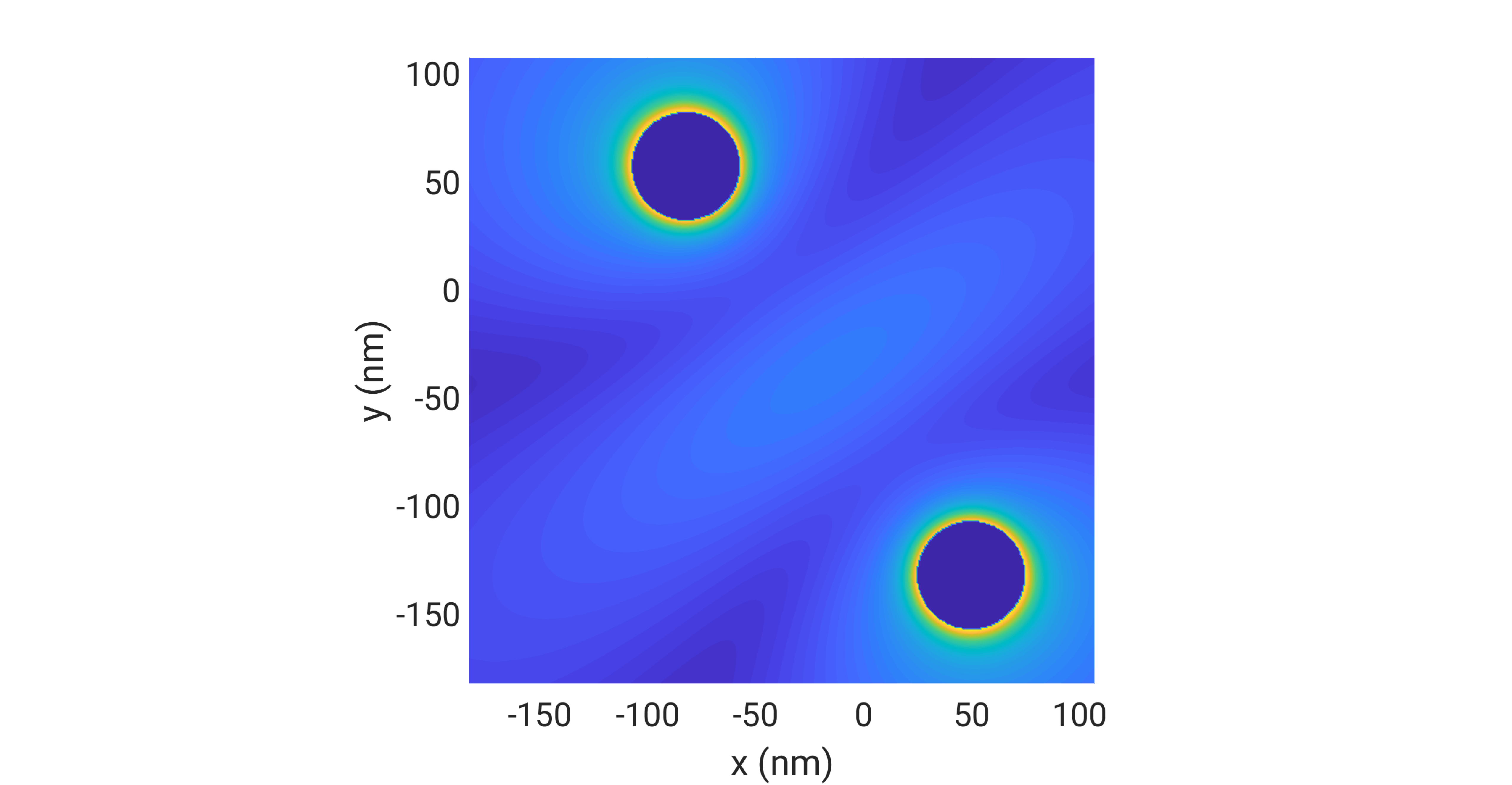}
	\includegraphics[width=0.49\textwidth,trim={8cm 0 8cm 0},clip]{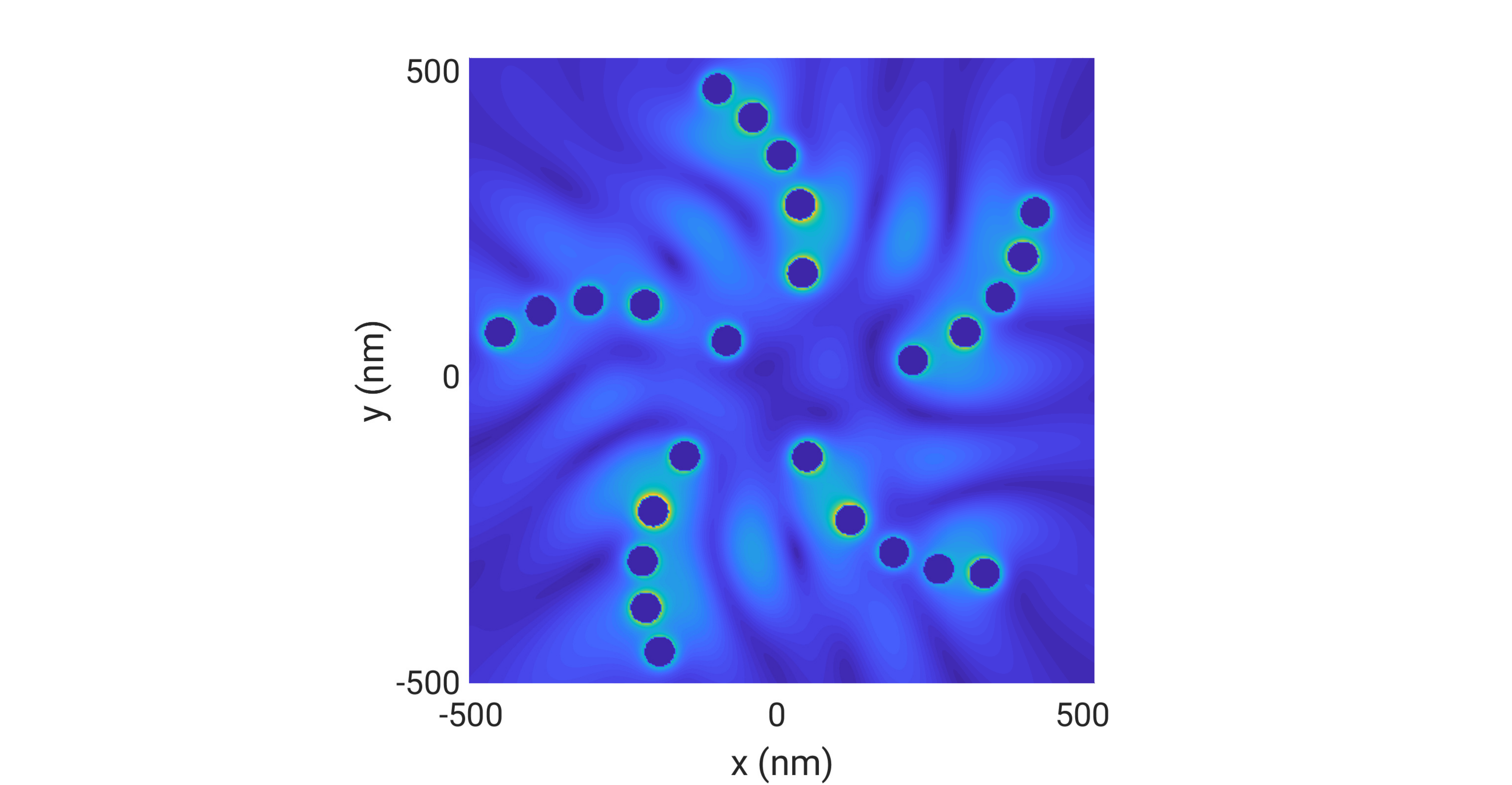}
	\includegraphics[width=0.49\textwidth,trim={8cm 0 8cm 0},clip]{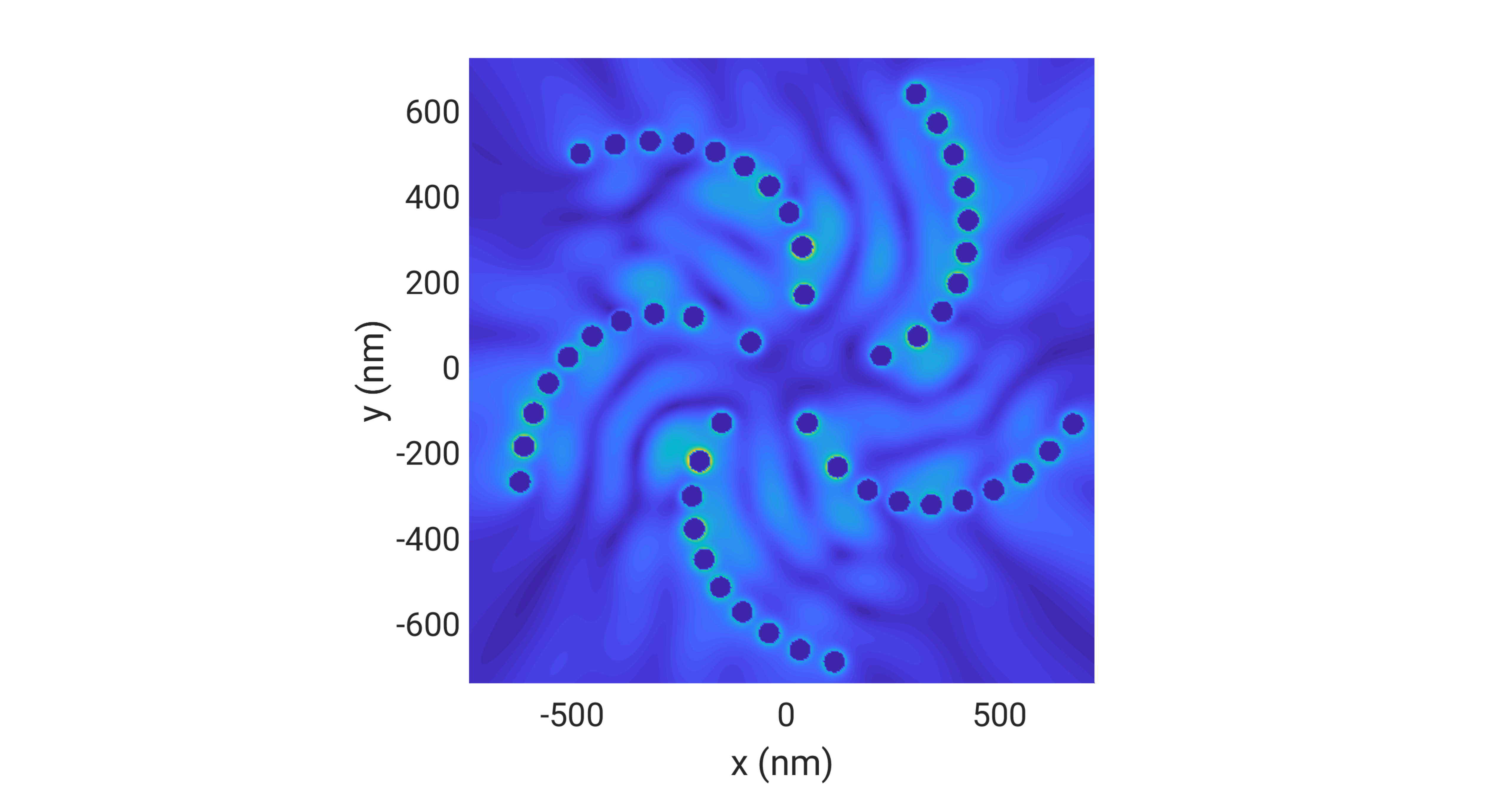}
	\includegraphics[width=0.49\textwidth,trim={8cm 0 8cm 0},clip]{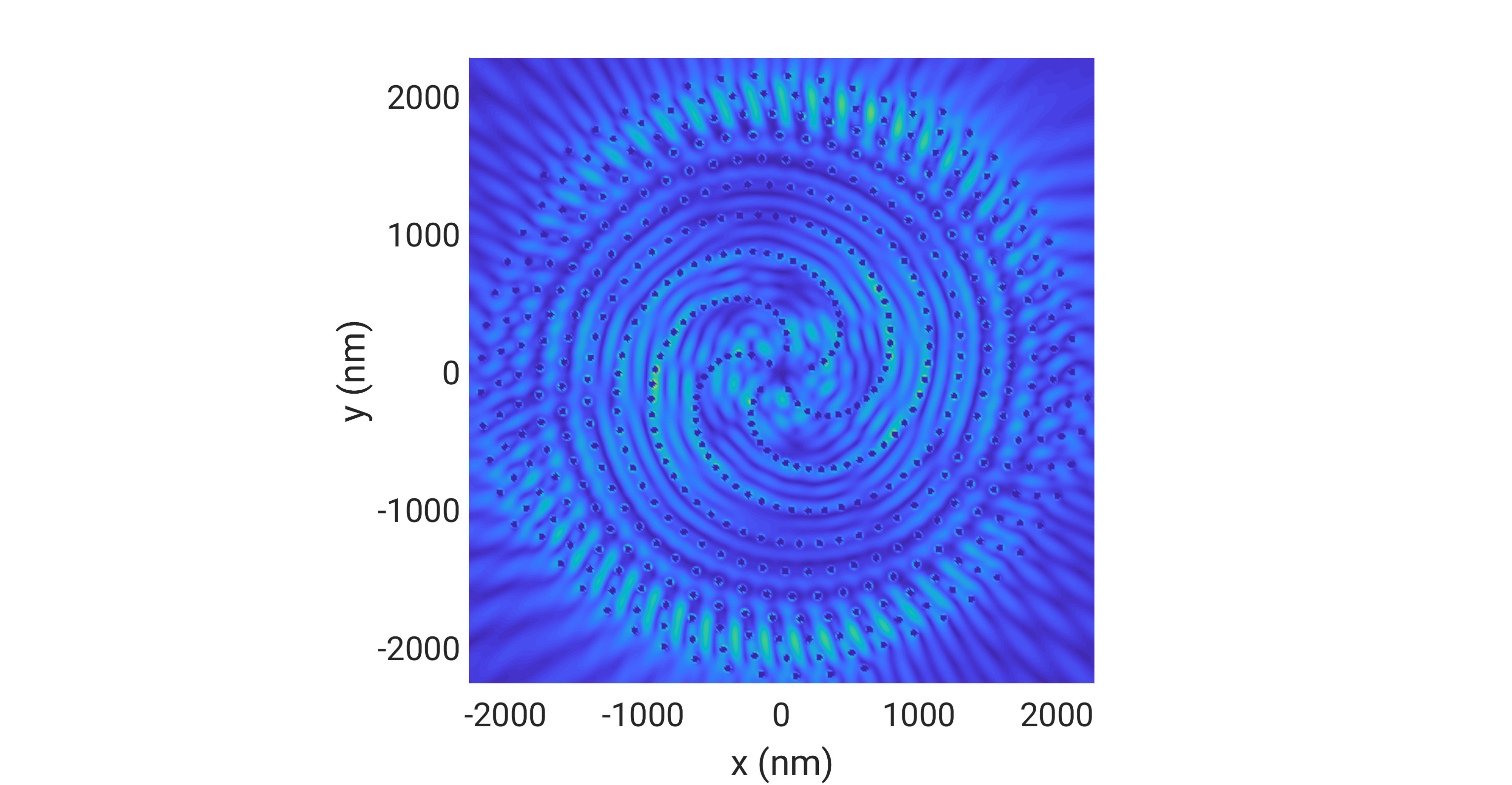}
	\caption{Sample configurations of Vogel spirals with $\{2,25,50,500\}$ particles.}
\label{vogel}
\end{figure} 
\subsection{Plasmonic Nanoparticle Arrays Simulations}
\label{appe: palas}

In the final example, we calculate the extinction and scattering efficiencies $Q_{ext}$ and $Q_{sc}$ for plasmonic systems with varying numbers of scatterers using the Coupled Dipole Approximation (CDA) approach. CDA is a method for mimicking the optical response of an array of similar, non-magnetic metallic nanoparticles with dimensions far smaller than the wavelength of light (here 25 nm).
$Q_{ext}$ and $Q_{sc}$ are defined as the QoIs in this document. We construct surrogate models for efficiency with up to three fidelities using our proposed method. We examine particle arrays resulting from Vogel spirals. Since the number of interactions of incident waves from particles influences the magnetic field, the number of nanoparticles in a plasmonic array has a substantial effect on the local extinction field caused by plasmonic arrays. The configurations of Vogel spirals with particle numbers in the set $\{2,25,50\}$ that define fidelities F1 through F3 are depicted in Fig. \ref{vogel}.
$\lambda\in[200,800]$ nm, $\alpha_{vs}\in[0,2\pi]$ rad, and $a_{vs}\in(1,1500)$ are determined to be the parameter space. These are, respectively, the incidence wavelength, the divergence angle, and the scaling factor. A Sobol sequence is utilized to choose inputs. The computing time requires to execute CDA increases exponentially as the number of nanoparticles increases. Consequently, the proposed sampling approach results in significant reductions in computational costs.

The response of a plasmonic array to electromagnetic radiation is calculable using the solution of the local electric fields, $\E_{loc}(\r_j)$, for each nano-sphere. Considering $N$ metallic particles defined by the same volumetric polarizability $\alpha(\omega)$ and situated at vector coordinates $\r_i$, it is possible to calculate the local field $\E_{loc}(\r_j)$ by solving~\citep{guerin2006effective} the corresponding linear equation.

\begin{equation}\label{eqFoldyLax}
\E_{loc}(\mathbf r_i)=\mathbf E_0({\mathbf r_i})-\frac{\alpha k^2}{\epsilon_0} \sum_{j=1,j\neq i}^{N} \mathbf{\tilde{G}}_{ij} \mathbf E_{loc}(\mathbf r_j)
\end{equation}
in which $\mathbf E_0(\r_i)$ is the incident field, $k$ is the wave number in the background medium, $\epsilon_0$ denotes the dielectric permittivity of vacuum ($\epsilon_0 = 1$ in the CGS unit system), and $\mathbf{\tilde{G}}_{ij}$ is constructed from $3\times3$ blocks of the overall $3N\times3N$ Green's matrices for the $i$th and $j$th particles. $\mathbf{\tilde{G}}_{ij}$ is a zero matrix when $j = i$, and
otherwise calculated as
\begin{equation}\label{Gij}\tilde{\bf G}_{ij}=\frac{\exp(ikr_{ij})}{r_{ij}}\left\{{\bf I}-\widehat{\r}_{ij}\widehat{\r}_{ij}^T-\left[\frac{1}{ikr_{ij}}+\frac{1}{(kr_{ij})^2}({\bf I}-3\widehat{\r}_{ij}\widehat{\r}_{ij}^T)\right]\right\}
\end{equation}
where $\widehat{\r}_{ij}$ denotes the unit position vector from particles $j$ to $i$ and $r_{ij}=|\r_{ij}|$. 
By solving Eqs.~\ref{eqFoldyLax} and \ref{Gij}, the total local fields $\mathbf E_{loc}(\r_i)$, and as a result the scattering and extinction cross-sections, are computed. Details of the numerical solution can be found in \citep{razi2019optimization}.

$Q_{ext}$ and $Q_{sc}$ are derived by normalizing the scattering and extinction cross-sections relative to the array's entire projected area. We considered the Vogel spiral class of particle arrays, which is described by~\citep{christofi2016probing}
\begin{equation}\label{eqVSa}
\rho_n = \sqrt{n}a_{vs} \quad \mbox{and}\quad 
\theta_n = n\alpha_{vs},
\end{equation}
where $\rho_n$ and $\theta_n$ represent the radial distance and polar angle of the $n$-th particle in a Vogel spiral array, respectively. Therefore, the Vogel spiral configuration may be uniquely defined by the incidence wavelength $\lambda$, the divergence angle $\alpha_{vs}$, the scaling factor $a_{vs}$, and the number of particles $n$.

\begin{figure}[]
	\centering
	\begin{subfigure}[b]{0.32\linewidth}
		\includegraphics[width=1\textwidth]{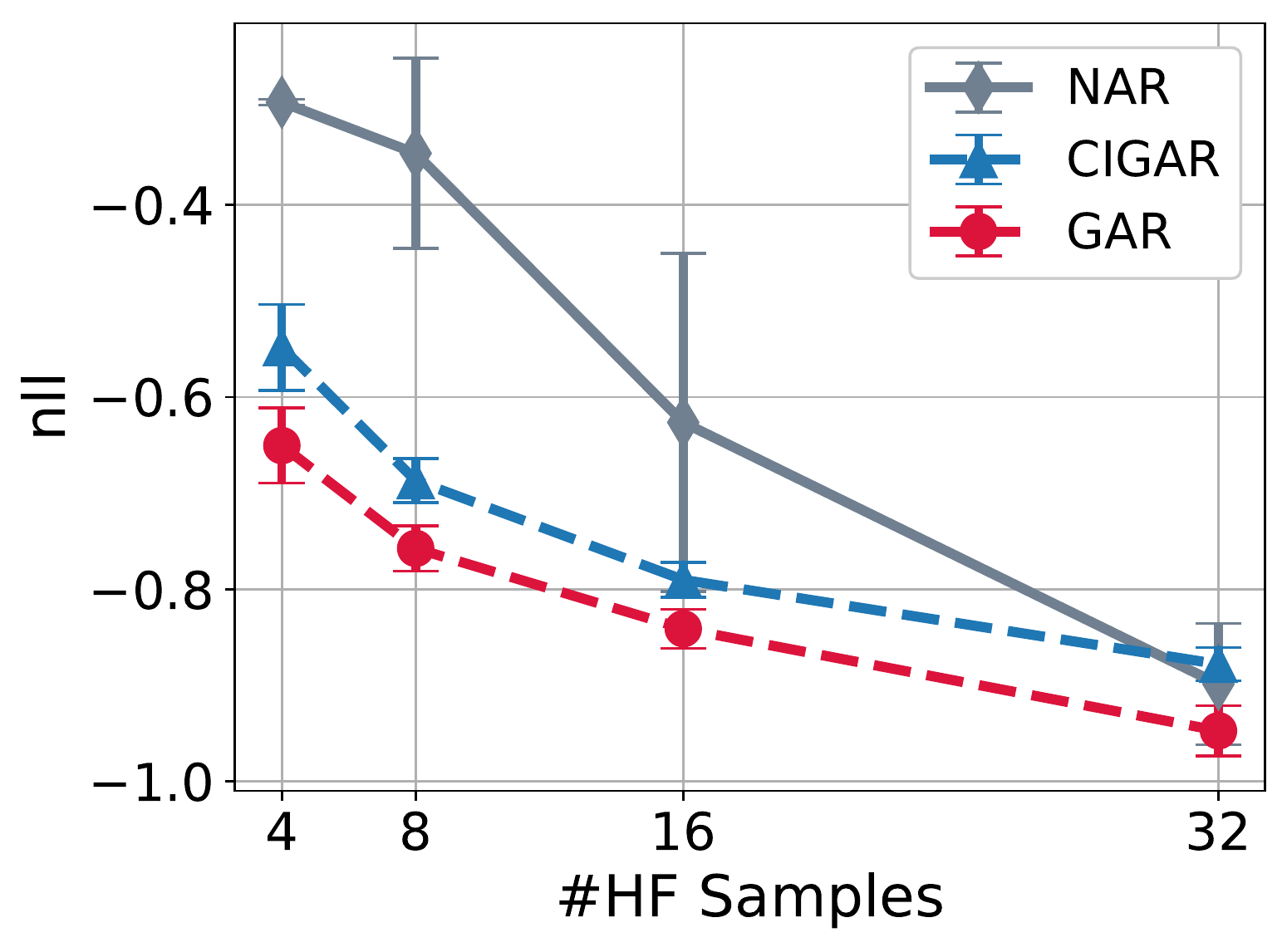}
	\end{subfigure}
	\begin{subfigure}[b]{0.32\linewidth}
		\includegraphics[width=1\textwidth]{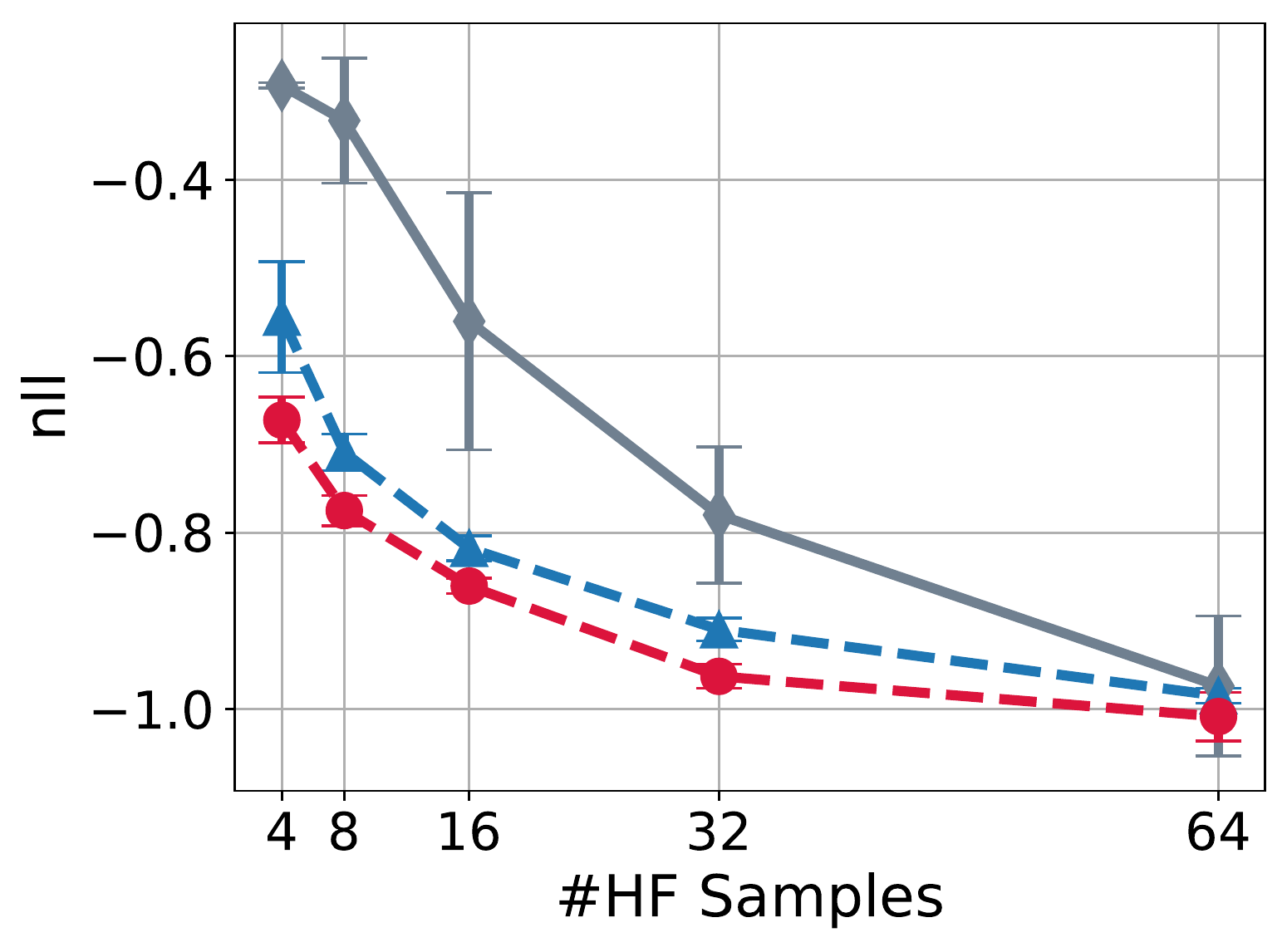}
	\end{subfigure}
	\begin{subfigure}[b]{0.32\linewidth}
		\includegraphics[width=1\textwidth]{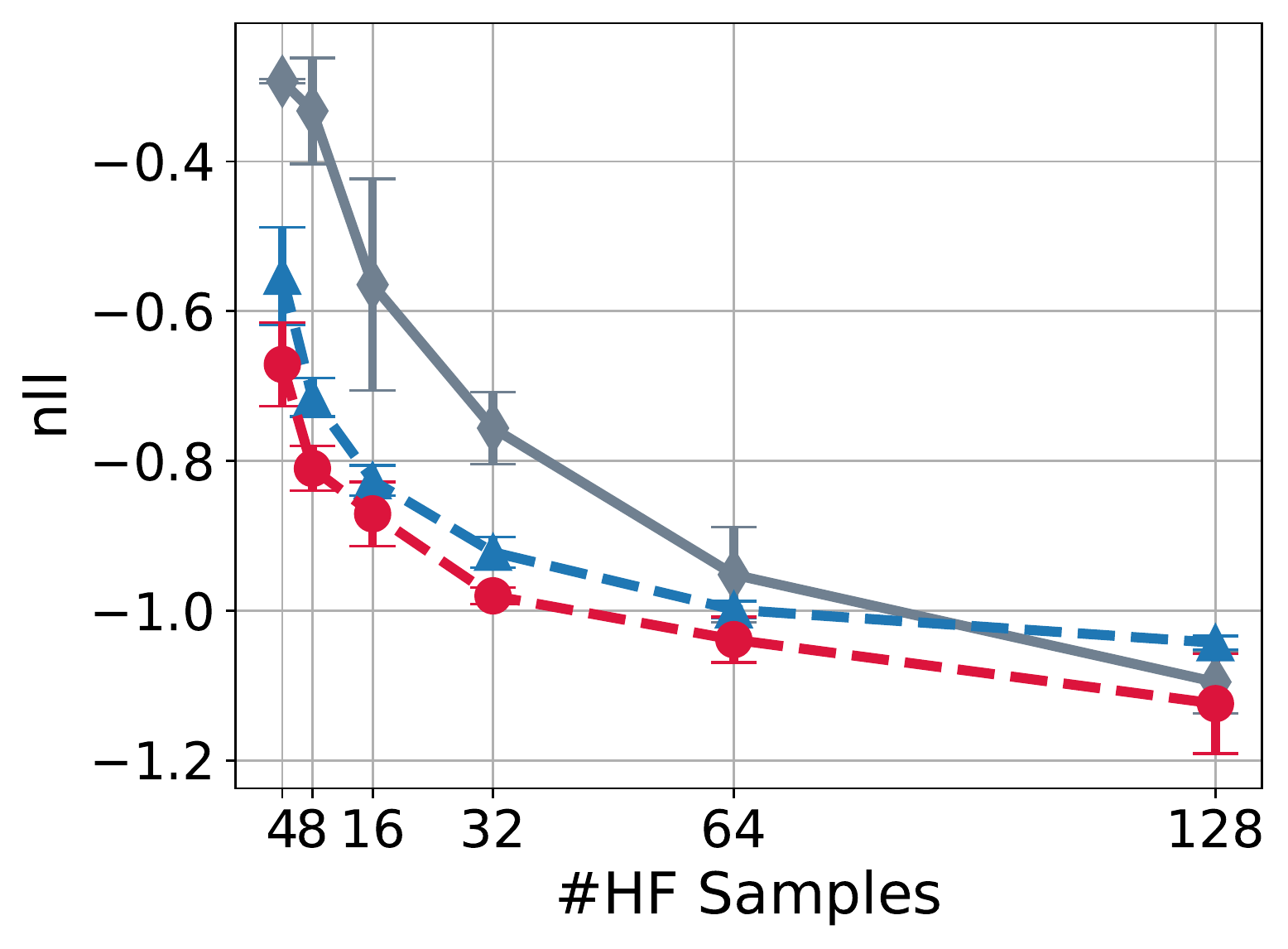}
	\end{subfigure}
	\caption{{NLL  with low-fidelity training sample number fixed to \{32,64,128\} for topology structure predictions.}}
	\label{fig: top nll}
\end{figure}

\begin{figure}[]
	\centering
	\begin{subfigure}[b]{0.32\linewidth}
		\includegraphics[width=1\textwidth]{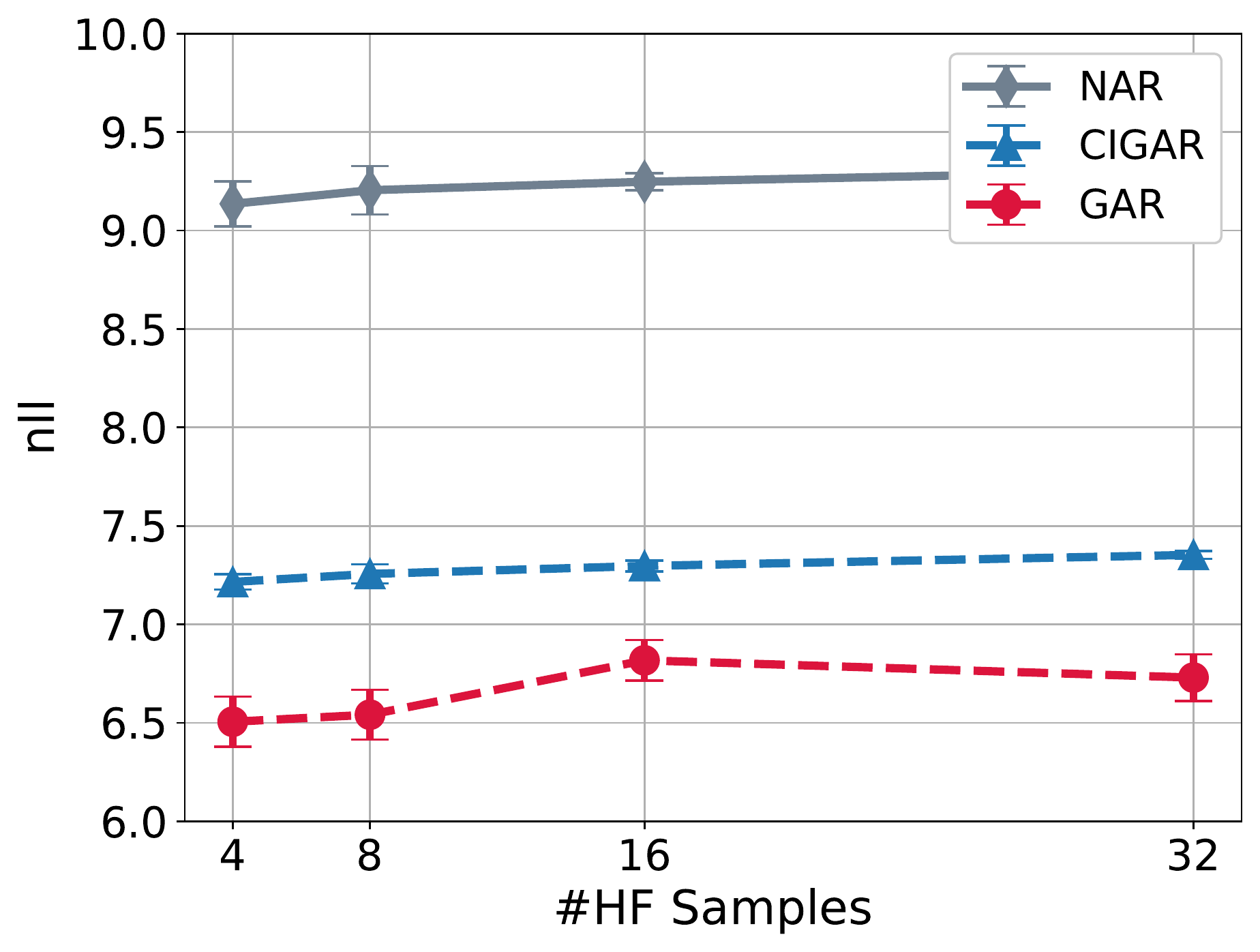}
		\caption{ECD+IP}
	\end{subfigure}
	\begin{subfigure}[b]{0.32\linewidth}
		\includegraphics[width=1\textwidth]{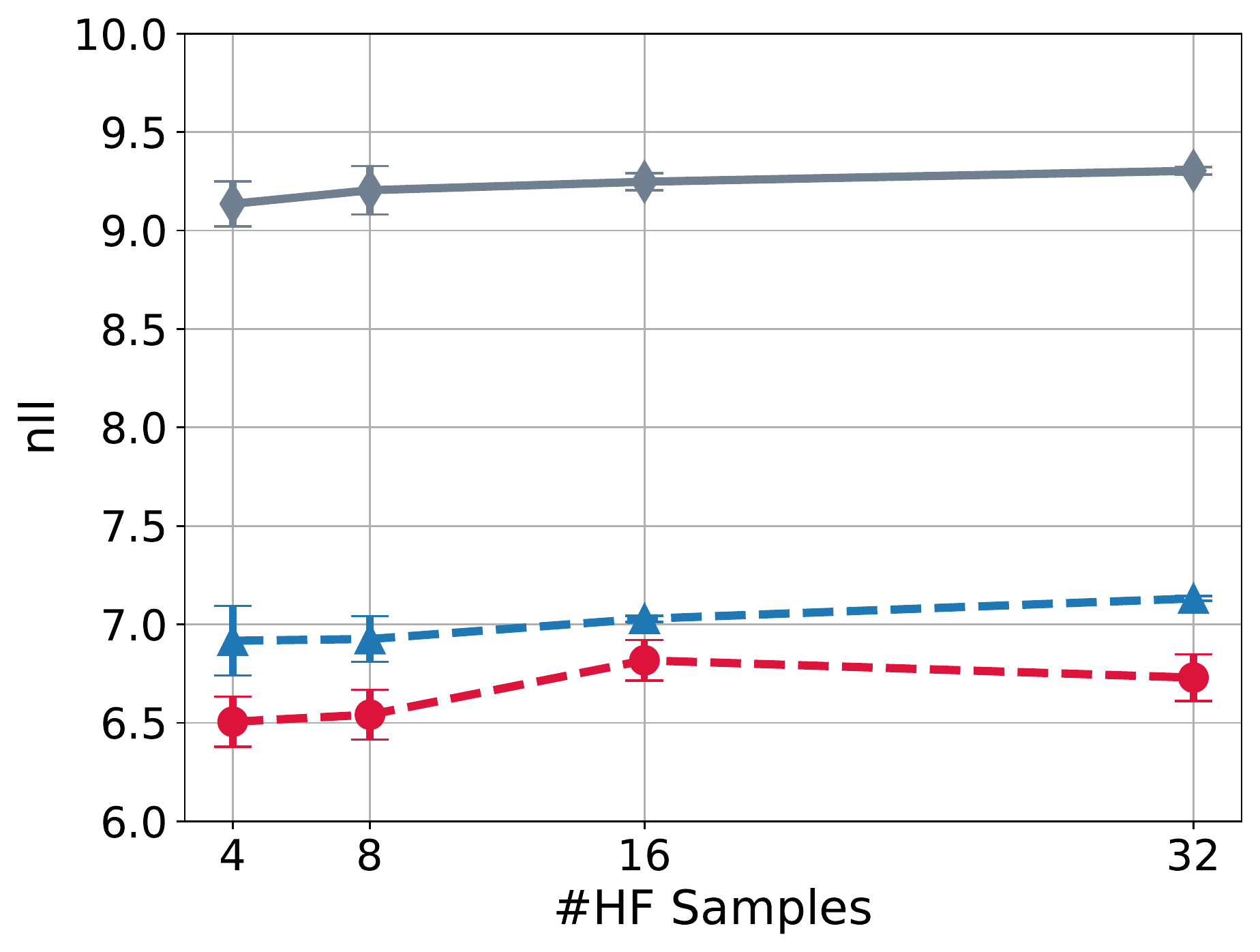}
		\caption{ECD}
	\end{subfigure}
	\begin{subfigure}[b]{0.32\linewidth}
		\includegraphics[width=1\textwidth]{./figure_rebuttal/nll_pic/false//SOFC/SOFC_MF_y1y2}
		\caption{IP}
	\end{subfigure}
	\caption{{NLL for SOFC with low-fidelity training sample number fixed to 32.}}
	\label{fig: SOFC nll}
\end{figure}

\subsection{Metircs for the Predictive Uncertainty}
{
Despite that RMSE has been used as a standard metric for evaluating the performance of a multi-fidelity fusion algorithm \citep{xing2021residual,xing2021deep,li2020deep,perdikaris2017nonlinear}, a metric that considers the predictive uncertainty is also important~\citep{wu2022multi}, particularly when the downstream applications rely heavily on the quality of the predictive confidence, \eg in MFBO~\citep{li2020multi}.
To assess the proposed method more comprehensively, we evaluate the quality of the predictive posterior using the most commonly used metric, negative-log-likelihood (nll).} 

{
We repreduce Figs. \ref{fig: topop rmse} and \ref{fig: sofc rmse} using exactly the same experimental setups but with the nll metric, and the results are shown in Figs. \ref{fig: top nll} and \ref{fig: SOFC nll}.
Note that the nll of DC and MF-BNN is every poor, probably due to our implementations, and cannot be fitted into the figures. Thus they are not shown in the figures.
Also note that some figures show negative nll. This is because our computation of the nll omits the constant term. Nevertheless, this modification does not affect the comparison.
We can see that for the topology structure posterior in  Fig.~\ref{fig: top nll}, the results are consistent with the conclusion drawn on the RMSE results. 
Since the CIGAR ignores the inter-output correlations, it will overestimate the covariance determinant, leading to higher nll than GAR. The NAR starts with poor performance with a small number of training data. 
It consistently improves with increasing number of training data and end up with similar perform as \ours and \ourss.
Similarly, the SOFC results are consistent with the conclusion for the RMSE results.
However, all methods demonstrated do not improve their performance significantly with more training data. 
This is caused by the calculations of the nll and the data itself. More specifically, in the ECD and IP fields, there are a few spatial locations where the recorded values are almost constant (caused by the Dirichlet boundary conditions). In this case, the nll will be dominated by the logarithm of variance and becomes less informative for the quality of the predictive variance. We thus see that the nll in \Figref{fig: SOFC nll} fluctuates around the same values no matter how many training points are used. 
We leave investigating the uncertainty metric using more advance metric (\eg \citep{wu2021quantifying}) more in depth in the future considering the scope of this work.
}


\begin{thebibliography}{74}
  \providecommand{\natexlab}[1]{#1}
  \providecommand{\url}[1]{\texttt{#1}}
  \expandafter\ifx\csname urlstyle\endcsname\relax
    \providecommand{\doi}[1]{doi: #1}\else
    \providecommand{\doi}{doi: \begingroup \urlstyle{rm}\Url}\fi
  
  \bibitem[Shahriari et~al.()Shahriari, Swersky, Wang, Adams, and
    Freitas]{shahriari2016taking}
  Bobak Shahriari, Kevin Swersky, Ziyu Wang, Ryan~P. Adams, and Nando Freitas.
  \newblock Taking the {{Human Out}} of the {{Loop}}: {{A Review}} of {{Bayesian
    Optimization}}.
  \newblock 104\penalty0 (1):\penalty0 148--175.
  \newblock ISSN 0018-9219, 1558-2256.
  \newblock \doi{10.1109/JPROC.2015.2494218}.
  \newblock URL \url{https://ieeexplore.ieee.org/document/7352306/}.
  
  \bibitem[Psaros et~al.()Psaros, Meng, Zou, Guo, and
    Karniadakis]{psaros2022uncertainty}
  Apostolos~F. Psaros, Xuhui Meng, Zongren Zou, Ling Guo, and George~Em
    Karniadakis.
  \newblock Uncertainty {{Quantification}} in {{Scientific Machine Learning}}:
    {{Methods}}, {{Metrics}}, and {{Comparisons}}.
  
  \bibitem[Kennedy()]{kennedy2000predicting}
  M.~Kennedy.
  \newblock Predicting the output from a complex computer code when fast
    approximations are available.
  \newblock 87\penalty0 (1):\penalty0 1--13.
  \newblock ISSN 0006-3444, 1464-3510.
  \newblock \doi{10.1093/biomet/87.1.1}.
  
  \bibitem[Poloczek et~al.(2017)Poloczek, Wang, and Frazier]{poloczek2017multi}
  Matthias Poloczek, Jialei Wang, and Peter Frazier.
  \newblock Multi-information source optimization.
  \newblock \emph{Advances in neural information processing systems}, 30, 2017.
  
  \bibitem[Song et~al.()Song, Chen, and Yue]{song2019general}
  Jialin Song, Yuxin Chen, and Yisong Yue.
  \newblock A {{General Framework}} for {{Multi-fidelity Bayesian Optimization}}
    with {{Gaussian Processes}}.
  \newblock In \emph{The 22nd {{International Conference}} on {{Artificial
    Intelligence}} and {{Statistics}}}, pages 3158--3167.
  \newblock URL \url{http://proceedings.mlr.press/v89/song19b.html}.
  
  \bibitem[Parussini et~al.()Parussini, Venturi, Perdikaris, and
    Karniadakis]{parussini2017multi}
  L.~Parussini, D.~Venturi, P.~Perdikaris, and G.E. Karniadakis.
  \newblock Multi-fidelity {{Gaussian}} process regression for prediction of
    random fields.
  \newblock 336\penalty0 (C):\penalty0 36--50.
  \newblock ISSN 0021-9991.
  \newblock \doi{10.1016/j.jcp.2017.01.047}.
  
  \bibitem[Xing et~al.({\natexlab{a}})Xing, Razi, Kirby, Sun, and
    Shah]{xing2020greedy}
  W.~Xing, M.~Razi, R.~M. Kirby, K.~Sun, and A.~A. Shah.
  \newblock Greedy nonlinear autoregression for multifidelity computer models at
    different scales.
  \newblock 1:\penalty0 100012, {\natexlab{a}}.
  \newblock ISSN 2666-5468.
  \newblock \doi{10.1016/j.egyai.2020.100012}.
  \newblock URL
    \url{https://www.sciencedirect.com/science/article/pii/S2666546820300124}.
  
  \bibitem[Wang et~al.()Wang, Xing, Kirby, and Zhe]{wang2021multi}
  Zheng Wang, Wei Xing, Robert Kirby, and Shandian Zhe.
  \newblock Multi-{{Fidelity High-Order Gaussian Processes}} for {{Physical
    Simulation}}.
  \newblock In \emph{International {{Conference}} on {{Artificial Intelligence}}
    and {{Statistics}}}, pages 847--855. {PMLR}.
  \newblock URL \url{http://proceedings.mlr.press/v130/wang21c.html}.
  
  \bibitem[Xing et~al.({\natexlab{b}})Xing, Shah, Wang, Zhe, Fu, and
    Kirby]{xing2021residual}
  W.~W. Xing, A.~A. Shah, P.~Wang, S.~Zhe, Q.~Fu, and R.~M. Kirby.
  \newblock Residual gaussian process: A tractable nonparametric bayesian
    emulator for multi-fidelity simulations.
  \newblock 97:\penalty0 36--56, {\natexlab{b}}.
  \newblock ISSN 0307-904X.
  \newblock \doi{10.1016/j.apm.2021.03.041}.
  \newblock URL
    \url{https://www.sciencedirect.com/science/article/pii/S0307904X21001724}.
  
  \bibitem[Fern\'andez-Godino et~al.()Fern\'andez-Godino, Park, Kim, and
    Haftka]{fernandezgodino2017review}
  M.~Giselle Fern\'andez-Godino, Chanyoung Park, Nam-Ho Kim, and Raphael~T.
    Haftka.
  \newblock Review of multi-fidelity models.
  
  \bibitem[Peherstorfer et~al.()Peherstorfer, Willcox, and
    Gunzburger]{peherstorfer2018survey}
  B.~Peherstorfer, K.~Willcox, and M.~Gunzburger.
  \newblock Survey of {{Multifidelity Methods}} in {{Uncertainty Propagation}},
    {{Inference}}, and {{Optimization}}.
  \newblock 60\penalty0 (3):\penalty0 550--591.
  \newblock ISSN 0036-1445.
  \newblock \doi{10.1137/16M1082469}.
  
  \bibitem[Xing et~al.({\natexlab{c}})Xing, Kirby, and Zhe]{xing2021deep}
  Wei~W. Xing, Robert~M. Kirby, and Shandian Zhe.
  \newblock Deep coregionalization for the emulation of simulation-based
    spatial-temporal fields.
  \newblock 428:\penalty0 109984, {\natexlab{c}}.
  \newblock ISSN 0021-9991.
  \newblock \doi{10.1016/j.jcp.2020.109984}.
  \newblock URL
    \url{https://linkinghub.elsevier.com/retrieve/pii/S0021999120307580}.
  
  \bibitem[Li et~al.(2020)Li, Kirby, and Zhe]{li2020deep}
  Shibo Li, Robert~M Kirby, and Shandian Zhe.
  \newblock Deep multi-fidelity active learning of high-dimensional outputs.
  \newblock \emph{arXiv preprint arXiv:2012.00901}, 2020.
  
  \bibitem[Alvarez et~al.()Alvarez, Rosasco, and Lawrence]{alvarez2011kernels}
  Mauricio~A. Alvarez, Lorenzo Rosasco, and Neil~D. Lawrence.
  \newblock Kernels for {{Vector-Valued Functions}}: A {{Review}}.
  
  \bibitem[Le~Gratiet()]{legratiet2013bayesian}
  Loic Le~Gratiet.
  \newblock Bayesian {{Analysis}} of {{Hierarchical Multifidelity Codes}}.
  \newblock 1\penalty0 (1):\penalty0 244--269.
  \newblock ISSN 2166-2525.
  \newblock \doi{10.1137/120884122}.
  
  \bibitem[Perdikaris et~al.()Perdikaris, Raissi, Damianou, D.~Lawrence, and
    Karniadakis]{perdikaris2017nonlinear}
  Paris Perdikaris, M~Raissi, Andreas Damianou, N~D.~Lawrence, and George
    Karniadakis.
  \newblock \emph{Nonlinear Information Fusion Algorithms for Data-Efficient
    Multi-Fidelity Modelling}, volume 473.
  \newblock {Royal Society}.
  \newblock \doi{10.1098/rspa.2016.0751}.
  
  \bibitem[Rasmussen and Williams()]{rasmussen2006gaussian}
  Carl~Edward Rasmussen and Christopher K~I Williams.
  \newblock Gaussian {{Processes}} for {{Machine Learning}}.
  \newblock page 266.
  
  \bibitem[Zhe et~al.()Zhe, Xing, and Kirby]{zhe2019scalable}
  Shandian Zhe, Wei Xing, and Robert~M. Kirby.
  \newblock Scalable {{High-Order Gaussian Process Regression}}.
  \newblock In \emph{The 22nd {{International Conference}} on {{Artificial
    Intelligence}} and {{Statistics}}}, pages 2611--2620. {PMLR}.
  \newblock URL \url{http://proceedings.mlr.press/v89/zhe19a.html}.
  
  \bibitem[Kolda()]{kolda2006multilinear}
  Tamara~Gibson Kolda.
  \newblock Multilinear operators for higher-order decompositions.
  \newblock URL \url{http://www.osti.gov/servlets/purl/923081-u0xXJa/}.
  
  \bibitem[Xu et~al.()Xu, Yan, Yuan, and Qi]{xu2011infinite}
  Zenglin Xu, Feng Yan, Yuan, and Qi.
  \newblock Infinite {{Tucker Decomposition}}: {{Nonparametric Bayesian Models}}
    for {{Multiway Data Analysis}}.
  
  \bibitem[Wilson and Nickisch()]{wilson2015kernel}
  Andrew~Gordon Wilson and Hannes Nickisch.
  \newblock Kernel interpolation for scalable structured {{Gaussian}} processes
    ({{KISS-GP}}).
  \newblock In \emph{Proceedings of the 32nd {{International Conference}} on
    {{International Conference}} on {{Machine Learning}} - {{Volume}} 37},
    {{ICML}}'15, pages 1775--1784. {JMLR.org}.
  
  \bibitem[Perrone et~al.(2018)Perrone, Jenatton, Seeger, and
    Archambeau]{perrone2018scalable}
  Valerio Perrone, Rodolphe Jenatton, Matthias~W Seeger, and C{\'e}dric
    Archambeau.
  \newblock Scalable hyperparameter transfer learning.
  \newblock \emph{Advances in neural information processing systems}, 31, 2018.
  
  \bibitem[Li et~al.({\natexlab{a}})Li, Xing, Kirby, and Zhe]{li2020multi}
  Shibo Li, Wei Xing, Robert Kirby, and Shandian Zhe.
  \newblock Multi-{{Fidelity Bayesian Optimization}} via {{Deep Neural
    Networks}}.
  \newblock 33:\penalty0 8521--8531, {\natexlab{a}}.
  \newblock URL
    \url{https://papers.nips.cc/paper/2020/hash/60e1deb043af37db5ea4ce9ae8d2c9ea-Abstract.html}.
  
  \bibitem[Narayan et~al.()Narayan, Gittelson, and Xiu]{narayan2014stochastic}
  Akil Narayan, Claude Gittelson, and Dongbin Xiu.
  \newblock A {{Stochastic Collocation Algorithm}} with {{Multifidelity Models}}.
  \newblock 36\penalty0 (2):\penalty0 A495--A521.
  \newblock ISSN 1064-8275.
  \newblock \doi{10.1137/130929461}.
  
  \bibitem[Cutajar et~al.()Cutajar, Pullin, Damianou, Lawrence, and
    Gonz\'alez]{cutajar2019deep}
  Kurt Cutajar, Mark Pullin, Andreas Damianou, Neil Lawrence, and Javier
    Gonz\'alez.
  \newblock Deep {{Gaussian Processes}} for {{Multi-fidelity Modeling}}.
  
  \bibitem[Wilson and Adams(2013)]{wilson2013gaussian}
  Andrew Wilson and Ryan Adams.
  \newblock Gaussian process kernels for pattern discovery and extrapolation.
  \newblock In \emph{International conference on machine learning}, pages
    1067--1075. PMLR, 2013.
  
  \bibitem[\'Alvarez et~al.()\'Alvarez, Rosasco, and
    Lawrence]{alvarez2012kernels}
  Mauricio~A. \'Alvarez, Lorenzo Rosasco, and Neil~D. Lawrence.
  \newblock Kernels for {{Vector-Valued Functions}}: {{A Review}}.
  \newblock 4\penalty0 (3):\penalty0 195--266.
  \newblock ISSN 1935-8237, 1935-8245.
  \newblock \doi{10.1561/2200000036}.
  \newblock URL \url{http://www.nowpublishers.com/article/Details/MAL-036}.
  
  \bibitem[Goulard and Voltz(1992)]{goulard1992linear}
  Michel Goulard and Marc Voltz.
  \newblock Linear coregionalization model: tools for estimation and choice of
    cross-variogram matrix.
  \newblock \emph{Mathematical Geology}, 24\penalty0 (3):\penalty0 269--286,
    1992.
  
  \bibitem[Goovaerts et~al.(1997)]{goovaerts1997geostatistics}
  Pierre Goovaerts et~al.
  \newblock \emph{Geostatistics for natural resources evaluation}.
  \newblock Oxford University Press on Demand, 1997.
  
  \bibitem[Teh et~al.(2005)Teh, Seeger, and Jordan]{teh2005semiparametric}
  Yee~Whye Teh, Matthias Seeger, and Michael~I Jordan.
  \newblock Semiparametric latent factor models.
  \newblock In \emph{International Workshop on Artificial Intelligence and
    Statistics}, pages 333--340. PMLR, 2005.
  
  \bibitem[Higdon et~al.()Higdon, Gattiker, Williams, and
    Rightley]{higdon2008computer}
  Dave Higdon, James Gattiker, Brian Williams, and Maria Rightley.
  \newblock Computer {{Model Calibration Using High-Dimensional Output}}.
  \newblock 103\penalty0 (482):\penalty0 570--583.
  \newblock ISSN 0162-1459, 1537-274X.
  \newblock \doi{10.1198/016214507000000888}.
  
  \bibitem[Xing et~al.({\natexlab{d}})Xing, Triantafyllidis, Shah, Nair, and
    Zabaras]{xing2016manifold}
  W.W. Xing, V.~Triantafyllidis, A.A. Shah, P.B. Nair, and N.~Zabaras.
  \newblock Manifold learning for the emulation of spatial fields from
    computational models.
  \newblock 326:\penalty0 666--690, {\natexlab{d}}.
  \newblock ISSN 0021-9991.
  \newblock \doi{10.1016/j.jcp.2016.07.040}.
  \newblock URL
    \url{https://linkinghub.elsevier.com/retrieve/pii/S0021999116303722}.
  
  \bibitem[Xing et~al.({\natexlab{e}})Xing, Shah, and Nair]{xing2015reduced}
  Wei Xing, Akeel~A. Shah, and Prasanth~B. Nair.
  \newblock Reduced dimensional {{Gaussian}} process emulators of parametrized
    partial differential equations based on {{Isomap}}.
  \newblock 471\penalty0 (2174):\penalty0 20140697, {\natexlab{e}}.
  \newblock ISSN 1364-5021, 1471-2946.
  \newblock \doi{10.1098/rspa.2014.0697}.
  
  \bibitem[{\'A}lvarez et~al.(2019){\'A}lvarez, Ward, and
    Guarnizo]{alvarez2019non}
  Mauricio~A {\'A}lvarez, Wil Ward, and Cristian Guarnizo.
  \newblock Non-linear process convolutions for multi-output gaussian processes.
  \newblock In \emph{The 22nd International Conference on Artificial Intelligence
    and Statistics}, pages 1969--1977. PMLR, 2019.
  
  \bibitem[Boyle and Frean(2004)]{boyle2004dependent}
  Phillip Boyle and Marcus Frean.
  \newblock Dependent gaussian processes.
  \newblock \emph{Advances in neural information processing systems}, 17, 2004.
  
  \bibitem[Higdon()]{higdon2002space}
  Dave Higdon.
  \newblock Space and {{Space-Time Modeling}} using {{Process Convolutions}}.
  \newblock In Clive~W. Anderson, Vic Barnett, Philip~C. Chatwin, and Abdel~H.
    El-Shaarawi, editors, \emph{Quantitative {{Methods}} for {{Current
    Environmental Issues}}}, pages 37--56. {Springer London}.
  \newblock ISBN 978-1-4471-1171-9 978-1-4471-0657-9.
  \newblock \doi{10.1007/978-1-4471-0657-9_2}.
  
  \bibitem[Bonilla et~al.(2007)Bonilla, Agakov, and Williams]{bonilla2007kernel}
  Edwin~V Bonilla, Felix~V Agakov, and Christopher~KI Williams.
  \newblock Kernel multi-task learning using task-specific features.
  \newblock In \emph{Artificial Intelligence and Statistics}, pages 43--50. PMLR,
    2007.
  
  \bibitem[Rakitsch et~al.(2013)Rakitsch, Lippert, Borgwardt, and
    Stegle]{rakitsch2013it}
  Barbara Rakitsch, Christoph Lippert, Karsten Borgwardt, and Oliver Stegle.
  \newblock It is all in the noise: Efficient multi-task gaussian process
    inference with structured residuals.
  \newblock \emph{Advances in neural information processing systems}, 26, 2013.
  
  \bibitem[Li and Chen()]{li2018hierarchical}
  Ping Li and Songcan Chen.
  \newblock Hierarchical {{Gaussian Processes}} model for multi-task learning.
  \newblock 74:\penalty0 134--144.
  \newblock ISSN 0031-3203.
  \newblock \doi{10.1016/j.patcog.2017.09.021}.
  \newblock URL
    \url{https://linkinghub.elsevier.com/retrieve/pii/S0031320317303746}.
  
  \bibitem[Wilson et~al.(2012)Wilson, Knowles, and
    Ghahramani]{wilson2011gaussian}
  Andrew~Gordon Wilson, David~A. Knowles, and Zoubin Ghahramani.
  \newblock Gaussian process regression networks.
  \newblock In \emph{Proceedings of the 29th International Coference on
    International Conference on Machine Learning}, ICML'12, page 1139–1146,
    Madison, WI, USA, 2012. Omnipress.
  \newblock ISBN 9781450312851.
  
  \bibitem[Nguyen and Bonilla(2013)]{nguyen2013efficient}
  Trung Nguyen and Edwin Bonilla.
  \newblock Efficient variational inference for gaussian process regression
    networks.
  \newblock In \emph{Artificial Intelligence and Statistics}, pages 472--480.
    PMLR, 2013.
  
  \bibitem[Kolda and Bader()]{kolda2009tensor}
  Tamara~G. Kolda and Brett~W. Bader.
  \newblock Tensor {{Decompositions}} and {{Applications}}.
  \newblock 51\penalty0 (3):\penalty0 455--500.
  \newblock ISSN 0036-1445, 1095-7200.
  \newblock \doi{10.1137/07070111X}.
  
  \bibitem[Li et~al.({\natexlab{b}})Li, Xing, Kirby, and Zhe]{li2020scalable}
  Shibo Li, Wei Xing, Robert~M. Kirby, and Shandian Zhe.
  \newblock Scalable {{Gaussian Process Regression Networks}}.
  \newblock volume~3, pages 2456--2462, {\natexlab{b}}.
  \newblock \doi{10.24963/ijcai.2020/340}.
  \newblock URL \url{https://www.ijcai.org/proceedings/2020/340}.
  
  \bibitem[Damianou()]{damianou2015deep}
  Andreas Damianou.
  \newblock Deep {{Gaussian}} processes and variational propagation of
    uncertainty.
  
  \bibitem[Kandasamy et~al.(2016)Kandasamy, Dasarathy, Oliva, Schneider, and
    P{\'o}czos]{kandasamy2016gaussian}
  Kirthevasan Kandasamy, Gautam Dasarathy, Junier~B Oliva, Jeff Schneider, and
    Barnab{\'a}s P{\'o}czos.
  \newblock Gaussian process bandit optimisation with multi-fidelity evaluations.
  \newblock \emph{Advances in neural information processing systems}, 29, 2016.
  
  \bibitem[Zhang et~al.(2017)Zhang, Hoang, Low, and
    Kankanhalli]{zhang2017information}
  Yehong Zhang, Trong~Nghia Hoang, Bryan Kian~Hsiang Low, and Mohan Kankanhalli.
  \newblock Information-based multi-fidelity bayesian optimization.
  \newblock In \emph{NIPS Workshop on Bayesian Optimization}, 2017.
  
  \bibitem[Wu et~al.(2022)Wu, Chinazzi, Vespignani, Ma, and Yu]{wu2022multi}
  Dongxia Wu, Matteo Chinazzi, Alessandro Vespignani, Yi-An Ma, and Rose Yu.
  \newblock Multi-fidelity hierarchical neural processes.
  \newblock \emph{arXiv preprint arXiv:2206.04872}, 2022.
  
  \bibitem[Meng and Karniadakis()]{meng2020composite}
  Xuhui Meng and George~Em Karniadakis.
  \newblock A composite neural network that learns from multi-fidelity data:
    {{Application}} to function approximation and inverse {{PDE}} problems.
  \newblock 401:\penalty0 109020.
  \newblock ISSN 0021-9991.
  \newblock \doi{10.1016/j.jcp.2019.109020}.
  \newblock URL
    \url{http://www.sciencedirect.com/science/article/pii/S0021999119307260}.
  
  \bibitem[Requeima et~al.(2019)Requeima, Tebbutt, Bruinsma, and
    Turner]{requeima2019gaussian}
  James Requeima, William Tebbutt, Wessel Bruinsma, and Richard~E Turner.
  \newblock The gaussian process autoregressive regression model (gpar).
  \newblock In \emph{The 22nd International Conference on Artificial Intelligence
    and Statistics}, pages 1860--1869. PMLR, 2019.
  
  \bibitem[Xia et~al.(2020)Xia, Bruinsma, Tebbutt, and Turner]{xia2020gaussian}
  Rui Xia, Wessel Bruinsma, William Tebbutt, and Richard~E Turner.
  \newblock The gaussian process latent autoregressive model.
  \newblock In \emph{Third Symposium on Advances in Approximate Bayesian
    Inference}, 2020.
  
  \bibitem[Tuo et~al.()Tuo, Wu, and Yu]{tuo2014surrogate}
  Rui Tuo, C.~F.~Jeff Wu, and Dan Yu.
  \newblock Surrogate {{Modeling}} of {{Computer Experiments With Different Mesh
    Densities}}.
  \newblock 56\penalty0 (3):\penalty0 372--380.
  \newblock ISSN 0040-1706, 1537-2723.
  \newblock \doi{10.1080/00401706.2013.842935}.
  
  \bibitem[Efe and Ozbay()]{efe2003proper}
  Mehmet~Onder Efe and Hitay Ozbay.
  \newblock Proper orthogonal decomposition for reduced order modeling: 2d heat
    flow.
  \newblock In \emph{Proceedings of 2003 IEEE Conference on Control Applications,
    2003. CCA 2003.}, volume~2, pages 1273--1277. IEEE.
  
  \bibitem[Raissi and Karniadakis()]{raissi2017machine}
  Maziar Raissi and George~Em Karniadakis.
  \newblock Machine {{Learning}} of {{Linear Differential Equations}} using
    {{Gaussian Processes}}.
  \newblock 348:\penalty0 683--693.
  \newblock ISSN 0021-9991.
  \newblock \doi{10/gbzfgr}.
  
  \bibitem[Sobol'()]{sobol1967distribution}
  I.M Sobol'.
  \newblock On the distribution of points in a cube and the approximate
    evaluation of integrals.
  \newblock 7\penalty0 (4):\penalty0 86--112.
  \newblock ISSN 0041-5553.
  \newblock \doi{10.1016/0041-5553(67)90144-9}.
  \newblock URL
    \url{https://linkinghub.elsevier.com/retrieve/pii/0041555367901449}.
  
  \bibitem[Bruns and Tortorelli(2001)]{BRUNS20013443}
  Tyler~E. Bruns and Daniel~A. Tortorelli.
  \newblock Topology optimization of non-linear elastic structures and compliant
    mechanisms.
  \newblock \emph{Computer Methods in Applied Mechanics and Engineering},
    190\penalty0 (26):\penalty0 3443 -- 3459, 2001.
  \newblock ISSN 0045-7825.
  \newblock \doi{https://doi.org/10.1016/S0045-7825(00)00278-4}.
  \newblock URL
    \url{http://www.sciencedirect.com/science/article/pii/S0045782500002784}.
  
  \bibitem[Chung()]{chung2010computational}
  TJ~Chung.
  \newblock \emph{Computational fluid dynamics}.
  \newblock Cambridge university press.
  
  \bibitem[Sugimoto()]{sugimoto1991burgers}
  N~Sugimoto.
  \newblock Burgers equation with a fractional derivative; hereditary effects on
    nonlinear acoustic waves.
  \newblock 225:\penalty0 631--653.
  
  \bibitem[Nagel()]{nagel1996particle}
  Kai Nagel.
  \newblock Particle hopping models and traffic flow theory.
  \newblock 53\penalty0 (5):\penalty0 4655.
  
  \bibitem[Kutluay et~al.()Kutluay, Bahadir, and
    {\"O}zde{\c{s}}]{kutluay1999numerical}
  S~Kutluay, AR~Bahadir, and A~{\"O}zde{\c{s}}.
  \newblock Numerical solution of one-dimensional burgers equation: explicit and
    exact-explicit finite difference methods.
  \newblock 103\penalty0 (2):\penalty0 251--261.
  
  \bibitem[Shah et~al.()Shah, Xing, and Triantafyllidis]{shah2017reduced}
  A.~A. Shah, W.~W. Xing, and V.~Triantafyllidis.
  \newblock Reduced-order modelling of parameter-dependent, linear and nonlinear
    dynamic partial differential equation models.
  \newblock 473\penalty0 (2200):\penalty0 20160809.
  \newblock ISSN 1364-5021, 1471-2946.
  \newblock \doi{10.1098/rspa.2016.0809}.
  
  \bibitem[Raissi et~al.()Raissi, Perdikaris, and Karniadakis]{raissi2017physics}
  Maziar Raissi, Paris Perdikaris, and George~Em Karniadakis.
  \newblock Physics informed deep learning (part i): Data-driven solutions of
    nonlinear partial differential equations.
  
  \bibitem[Chapra et~al.()Chapra, Canale, et~al.]{chapra2010numerical}
  Steven~C Chapra, Raymond~P Canale, et~al.
  \newblock \emph{Numerical methods for engineers}.
  \newblock Boston: McGraw-Hill Higher Education,.
  
  \bibitem[Persides()]{persides1973laplace}
  S~Persides.
  \newblock The laplace and poisson equations in schwarzschild's space-time.
  \newblock 43\penalty0 (3):\penalty0 571--578.
  
  \bibitem[Lagaris et~al.(Sept./1998)Lagaris, Likas, and
    Fotiadis]{lagarisSept./1998artificial}
  I.~E. Lagaris, A.~Likas, and D.~I. Fotiadis.
  \newblock Artificial {{Neural Networks}} for {{Solving Ordinary}} and {{Partial
    Differential Equations}}.
  \newblock 9\penalty0 (5):\penalty0 987--1000, Sept./1998.
  \newblock ISSN 1045-9227.
  \newblock \doi{10.1109/72.712178}.
  
  \bibitem[Spitzer()]{spitzer1964electrostatic}
  Frank Spitzer.
  \newblock Electrostatic capacity, heat flow, and brownian motion.
  \newblock 3\penalty0 (2):\penalty0 110--121.
  
  \bibitem[Burdzy et~al.()Burdzy, Chen, Sylvester, et~al.]{burdzy2004heat}
  Krzysztof Burdzy, Zhen-Qing Chen, John Sylvester, et~al.
  \newblock The heat equation and reflected brownian motion in time-dependent
    domains.
  \newblock 32\penalty0 (1B):\penalty0 775--804.
  
  \bibitem[Black and Scholes()]{black1973pricing}
  Fischer Black and Myron Scholes.
  \newblock The pricing of options and corporate liabilities.
  \newblock 81\penalty0 (3):\penalty0 637--654.
  
  \bibitem[Xing et~al.({\natexlab{f}})Xing, Elhabian, Keshavarzzadeh, and
    Kirby]{xing2020shared}
  Wei Xing, Shireen~Y. Elhabian, Vahid Keshavarzzadeh, and Robert~M. Kirby.
  \newblock Shared-{{Gaussian Process}}: {{Learning Interpretable Shared Hidden
    Structure Across Data Spaces}} for {{Design Space Analysis}} and
    {{Exploration}}.
  \newblock 142\penalty0 (8), {\natexlab{f}}.
  \newblock ISSN 1050-0472, 1528-9001.
  \newblock \doi{10.1115/1.4046074}.
  
  \bibitem[Andreassen et~al.(2011)Andreassen, Clausen, Schevenels, Lazarov, and
    Sigmund]{Andreassen2011}
  Erik Andreassen, Anders Clausen, Mattias Schevenels, Boyan~S. Lazarov, and Ole
    Sigmund.
  \newblock Efficient topology optimization in matlab using 88 lines of code.
  \newblock \emph{Structural and Multidisciplinary Optimization}, 43\penalty0
    (1):\penalty0 1--16, Jan 2011.
  \newblock ISSN 1615-1488.
  \newblock \doi{10.1007/s00158-010-0594-7}.
  \newblock URL \url{https://doi.org/10.1007/s00158-010-0594-7}.
  
  \bibitem[Bendsoe and Sigmund(2004)]{bendsoe_topology_2004}
  Martin~Philip Bendsoe and Ole Sigmund.
  \newblock Topology optimization: Theory, methods and applications.
  \newblock \emph{Springer}, 2004.
  
  \bibitem[Gu{\'e}rin et~al.(2006)Gu{\'e}rin, Mallet, and
    Sentenac]{guerin2006effective}
  Charles-Antoine Gu{\'e}rin, Pierre Mallet, and Anne Sentenac.
  \newblock Effective-medium theory for finite-size aggregates.
  \newblock \emph{JOSA A}, 23\penalty0 (2):\penalty0 349--358, 2006.
  
  \bibitem[Razi et~al.()Razi, Wang, He, Kirby, and
    Dal~Negro]{razi2019optimization}
  Mani Razi, Ren Wang, Yanyan He, Robert~M. Kirby, and Luca Dal~Negro.
  \newblock Optimization of {{Large-Scale Vogel Spiral Arrays}} of {{Plasmonic
    Nanoparticles}}.
  \newblock 14\penalty0 (1):\penalty0 253--261.
  \newblock ISSN 1557-1955, 1557-1963.
  \newblock \doi{10.1007/s11468-018-0799-y}.
  
  \bibitem[Christofi et~al.(2016)Christofi, Pinheiro, and
    Dal~Negro]{christofi2016probing}
  Aristi Christofi, Felipe~A Pinheiro, and Luca Dal~Negro.
  \newblock Probing scattering resonances of vogel’s spirals with the green’s
    matrix spectral method.
  \newblock \emph{Optics letters}, 41\penalty0 (9):\penalty0 1933--1936, 2016.
  
  \bibitem[Wu et~al.(2021)Wu, Gao, Xiong, Chinazzi, Vespignani, Ma, and
    Yu]{wu2021quantifying}
  Dongxia Wu, Liyao Gao, Xinyue Xiong, Matteo Chinazzi, Alessandro Vespignani,
    Yi-An Ma, and Rose Yu.
  \newblock Quantifying uncertainty in deep spatiotemporal forecasting.
  \newblock \emph{arXiv preprint arXiv:2105.11982}, 2021.
  
  \end{thebibliography}
\end{document}